\documentclass{article}
    \usepackage[final]{neurips_2024}

\usepackage[utf8]{inputenc} % allow utf-8 input
\usepackage[T1]{fontenc}    % use 8-bit T1 fonts

\usepackage{multirow}
\usepackage{array}
\usepackage{appendix}

\makeatletter
\newcommand{\thickhline}{%
    \noalign {\ifnum 0=`}\fi \hrule height 1pt
    \futurelet \reserved@a \@xhline
}
\newcolumntype{"}{@{\hskip\tabcolsep\vrule width 1pt\hskip\tabcolsep}}
\makeatother

\usepackage{makecell}
\usepackage{enumitem}

\usepackage{amssymb}
\usepackage{mathtools}
\usepackage{amsthm}
\usepackage{amsmath} %maths
\allowdisplaybreaks
\usepackage[utf8]{inputenc} 
\usepackage{bm}
\usepackage{bbm}
\usepackage{algorithm}
\usepackage{algpseudocode}
\usepackage{algorithmicx}
\usepackage{MnSymbol}
\usepackage{comment}
\usepackage{setspace}
\usepackage{url}            % simple URL typesetting
\usepackage{booktabs}       % professional-quality tables
\usepackage{nicefrac}       % compact symbols for 1/2, etc.
\usepackage{microtype}      % microtypography
\usepackage{parskip}

\usepackage{hyperref}
\usepackage[nameinlink,capitalize]{cleveref}
\usepackage{crossreftools}
\makeatletter
\newcommand{\neutralize}[1]{\expandafter\let\csname c@#1\endcsname\count@}
\makeatother

\makeatletter
\renewenvironment{proof}[1][Proof]%
{%
	\par\noindent{\bfseries\upshape {#1.}\ }%
}%
{\qed\newline}
\makeatother
\creflabelformat{line}{#2\textup{#1}#3} 

\usepackage{thmtools}

\declaretheorem[name=Theorem,parent=section]{theorem}
\declaretheorem[name=Lemma,parent=section]{lemma}
\declaretheorem[name=Definition,parent=section]{definition}
\declaretheorem[name=Corollary,parent=section]{corollary}

\declaretheorem[name=Assumption, parent=section]{assumption}
\declaretheorem[name=Condition, parent=section]{condition}
\declaretheorem[name=Proposition, parent=section]{proposition}

\usepackage{transparent}
\usepackage{xcolor}

\newcommand{\algcommentlight}[1]{\textcolor{blue!70!black}{\transparent{0.5}\footnotesize{\texttt{\textbf{//\hspace{2pt}#1}}}}}

\newcommand{\Input}{\item[\textbf{Input:}]}

\usepackage{color-edits}
\addauthor{zm}{red}

\newcommand{\E}{\mathbb{E}}
\renewcommand{\P}{\mathbb{P}}
\newcommand{\cA}{\mathcal{A}}
\newcommand{\veps}{\varepsilon}
\newcommand{\pistar}{\pi^\star}
\newcommand{\reg}{\mathrm{reg}}

\newcommand{\bell}{\bm{\ell}}
\newcommand{\replearn}{\texttt{replearn}}
\newcommand{\rep}{\mathrm{rep}}

\newcommand{\Qbar}{\overline{Q}}

\newcommand{\pihat}{\widehat{\pi}}
\newcommand{\x}{\bm{x}}
\renewcommand{\a}{\bm{a}}
\newcommand{\phistar}{\phi^\star}

\newcommand{\Qhat}{\widehat{Q}}
\newcommand{\ind}[1]{^{(#1)}}

\newcommand{\mustar}{\mu^\star}
\newcommand{\nn}{\nonumber}
\newcommand{\phihat}{\hat{\phi}}
\newcommand{\cX}{\mathcal{X}}
\newcommand{\cG}{\mathcal{G}}

\newcommand{\reals}{\mathbb{R}}
\newcommand{\nreg}{N_\mathrm{reg}}

\newcommand{\calD}{\mathcal{D}}

\newcommand{\order}{\mathcal{O}}
\newcommand{\hphi}{\hat{\phi}}
\newcommand{\hpsi}{\hat{\mu}}
\newcommand{\hd}{\hat{d}}
\newcommand{\hV}{\widehat{V}}
\newcommand{\hell}{\hat{\ell}}
\newcommand{\tr}{\text{\rm Tr}}

\newcommand{\cM}{\mathcal{M}}
\newcommand{\otil}{\widetilde{\mathcal{O}}}
\newcommand{\Vhat}{\widehat{V}}

\DeclareMathOperator*{\argmin}{argmin}

\newcommand{\unif}{\texttt{unif}}
\newcommand{\ldotst}{%
	\mathinner{{\ldotp}{\ldotp}}%
}
\newcommand{\wtilde}{\widetilde}
\newcommand{\cF}{\mathcal{F}}
\newcommand{\cC}{\mathcal{C}}
\newcommand{\cQ}{\mathcal{Q}}
\newcommand{\cH}{\mathcal{H}}
\newcommand{\fhat}{\hat{f}}

\newcommand{\cov}{\mathrm{cov}}
\newcommand{\indd}[1]{^{#1}}

\newcommand{\cI}{\mathcal{I}}
\newcommand{\bpi}{\bm{\pi}}

\newcommand{\z}{\bm{z}}
\newcommand{\bcH}{\bm{\cH}}
\newcommand{\thetahat}{\hat\theta}
\newcommand{\spanner}{\mathrm{span}}
\newcommand{\bzeta}{\bm{\zeta}}
\newcommand{\bh}{\bm{h}}
\newcommand{\est}{\texttt{LinEst}}
\newcommand{\psdp}{\texttt{PSDP}}

\newcommand{\veceval}{\texttt{EstVec}}
\newcommand{\rspanner}{\texttt{RobustSpanner}}
\newcommand{\apx}{\texttt{LinOpt}}

\newcommand{\Trepval}{\frac{A H \log(|\Phi|/\delta)}{\alpha \veps^2}}
\newcommand{\Tspanval}{\frac{A \log(d H|\Phi|\veps^{-1} \delta^{-1} )}{\alpha \veps^2}}
\newcommand{\Trepvalf}{\alpha^{-1}\veps^{-2} A H \log(|\Phi|/\delta)}
\newcommand{\Tspanvalf}{\alpha^{-2} \veps^{-2} A \log(d H|\Phi|\veps^{-1} \delta^{-1} )}
\newcommand{\Tcovval}{\veps^{-2}A d^{13} H^6 \log(\Phi/\delta)}
\newcommand{\loss}{\mathrm{loss}}
\newcommand{\freed}{\mathrm{freed}}
\newcommand{\phib}{\bar{\phi}}
\newcommand{\phibar}{\bar{\phi}}
\newcommand{\thetab}{\bar{\theta}}
\newcommand{\ww}{w}
\newcommand{\g}{g}
\newcommand{\varthetab}{\bm{\vartheta}}
\newcommand{\bvthetab}{\bar{\varthetab}}
\newcommand{\thetabar}{\bar{\theta}}

\newcommand{\hatell}{\hat{\ell}}
\newcommand{\hatphi}{\hat{\phi}}

\newcommand{\tilp}{\rho}
\newcommand{\hatp}{\widehat{P}}

\newcommand{\bbB}{\mathbb{B}}

\newcommand{\poly}{\mathrm{poly}}
\newcommand{\cE}{\mathcal{E}}
\newcommand{\john}{\texttt{John}}
\newcommand{\bhpi}{\widehat{\bpi}}
\newcommand{\Lhat}{\widehat{L}}
\newcommand{\w}{\bm{w}}

\usepackage{titletoc}
\usepackage{appendix}

\title{Beating Adversarial Low-Rank MDPs with \\Unknown Transition and Bandit Feedback}

\author{%
\And
  Haolin Liu\thanks{The authors are listed in alphabetical order. } \\
  University of Virginia\\
  \texttt{srs8rh@virginia.edu} \\
  \qquad \qquad \\
  \And
  \hspace{20pt}Zakaria Mhammedi$^*$\\
 \hspace{20pt} Google Research \\
 \hspace{20pt} \texttt{mhammedi@google.com} \\
  \And
  \And 
   \hspace{27pt} Chen-Yu Wei$^*$ \\
   \hspace{27pt}  University of Virginia\\
  \hspace{27pt} \texttt{chenyu.wei@virginia.edu} \\
  \And
  \hspace{8pt} Julian Zimmert$^*$ \\
   \hspace{8pt} Google Research \\
   \hspace{8pt} \texttt{zimmert@google.com} \\
  \And
}

\begin{document}

\maketitle

\begin{abstract}
 % We consider regret minimization in low-rank MDPs with fixed transition and adversarial losses. Previous work has investigated this problem under either full-information loss feedback with unknown transitions~\citep{zhao2024learning}, or bandit loss feedback with known transition \citep{foster2022complexity}. We initiate the study on the setting with bandit loss feedback and unknown transitions. Assuming that the loss has a linear structure, we propose a computationally inefficient model-based algorithm with $\poly(d,A, H)T^{\nicefrac{2}{3}}$ regret, and oracle-efficient model-free algorithms with $\poly(d,A, H)T^{\nicefrac{5}{6}}$ regret, where $d$ is the rank of the transitions, $A$ is the number of actions, $H$ is the horizon length, and $T$ is the number of episodes.  We show that the linear structure is necessary for the case of bandit feedback--without structure on the reward function, the regret has to scale polynomially with the number of states. This is contrary to the full-information case~\citep{zhao2024learning}, where the regret can be independent of the number of states even for unstructured reward function. 
 We consider regret minimization in low-rank MDPs with fixed transition and adversarial losses. Previous work has investigated this problem under either full-information loss feedback with unknown transitions~\citep{zhao2024learning}, or bandit loss feedback with known transition \citep{foster2022complexity}. First, we improve the $\poly(d,A, H)T^{\nicefrac{5}{6}}$ regret bound 
 of~\cite{zhao2024learning} to $\poly(d,A,H)T^{\nicefrac{2}{3}}$ for the full-information unknown transition setting, where $d$ is the rank of the transitions, $A$ is the number of actions, $H$ is the horizon length, and $T$ is the number of episodes. Next, we initiate the study on the setting with bandit loss feedback and unknown transitions. Assuming that the loss has a linear structure, we propose both model-based and model-free algorithms achieving $\poly(d,A, H)T^{\nicefrac{2}{3}}$ regret, though they are computationally inefficient. We also propose oracle-efficient model-free algorithms with $\poly(d,A, H)T^{\nicefrac{4}{5}}$ regret. We show that the linear structure is necessary for the bandit case---without structure on the reward function, the regret has to scale polynomially with the number of states. This is contrary to the full-information case~\citep{zhao2024learning}, where the regret can be independent of the number of states even for unstructured reward function. 
\end{abstract}

\section{Introduction}
%Over the past few years, reinforcement learning (RL) has sparked revolutions in various application areas, including game-playing~\citep{mnih2015human, silver2016mastering}, robotics~\citep{kober2013reinforcement, kalashnikov2018scalable} and fine-tuning large language models~\citep{ouyang2022training}. The empirical success of reinforcement learning has motivated researchers to explore its theoretical foundations. A central problem in reinforcement learning theory is identifying provably learnable settings with minimal assumptions. In stochastic environments, existing works have proposed various general classes that ensure learnability~\citep{jiang2017contextual,  jin2021bellman, du2021bilinear, foster2021statistical, zhong2022gec, chen2022statistical, chen2022unified, xie2022role, foster2023tight, foster2024model}. Remarkably, recent works~\citep{foster2021statistical, foster2023tight, foster2024model} characterize the statistical complexity of interactive decision-making, providing both necessary and sufficient condition for sample-efficient reinforcenemt learning in stochastic settings. 

% \HL{1. Change log-barrier to exp.  2. Add description for $T^{\frac{2}{3}}$ model-free.  3. Change to $T^{\frac{4}{5}}$.}

We study online reinforcement learning (RL) in low-rank Markov Decision Processes (MDPs). Low-rank MDPs is a class of MDPs where the transition probability can be decomposed as an inner product between two low-dimensional features, i.e., $P(x'\mid x,a)=\phi^\star(x,a)^\top \mu^\star(x')$, where $P(x'\mid x,a)$ is the probability of transitioning to state $x'$ when the learner takes action $a$ on state $x$, and $\phi^\star$, $\mu^\star$ are two feature mappings. The ground truth features $\phi^\star$ and $\mu^\star$ are unknown to the learner. This setting has recently caught theoretical attention due to its simplicity and expressiveness \citep{ agarwal2020flambe, uehara2021representation, zhang2022making, cheng2023improved, modi2024model, zhang2022efficient, mhammedi2024efficient, huang2023reinforcement}.  In particular, since the learner does not know the features, it is necessary for the learner to perform \emph{feature learning} (or \emph{representation learning}) to approximate them. This allows low-rank MDPs to model the additional difficulty not present in traditional linear function approximation schemes where the features are given, such as in linear MDPs \citep{jin2020provably} and in linear mixture MDPs \citep{ayoub2020model}. Since feature learning is an indispensable part of modern deep RL
pipelines, low-rank MDP is a model that is closer to practice than traditional linear function approximation.

Most prior theoretical work on low-rank MDPs focuses on reward-free learning; this is a setting where instead of focusing on a particular reward function, the goal is to learn a model for the transitions (or, in the model-free setting, a small set of policies with good state cover), that enables policy optimization for \emph{any} downstream reward functions.  While this is a reasonable setup in some cases, in other applications, the learner can only obtain loss information from interactions with the environment, and only observes the loss on the state-actions that have been visited (i.e., bandit feedback). This introduces additional challenges to the learner. 

Furthermore, in many online learning scenarios, the loss function may change over time, reflecting the non-stationary nature of the environment or task switches \citep{padakandla2020reinforcement}. This could be modeled by the \emph{adversarial} MDP setting, where the loss function changes arbitrarily from one episode to the next, and the changes might even depend on the behavior of the learner. This setting is also extensively studied, but mostly restricted to tabular MDPs \citep{rosenberg2019online, jin2020learning, shani2020optimistic, luo2021policy} or traditional linear function approximation schemes \citep{cai2020provably, luo2021policy, he2022near, zhao2022mixture, sherman2023improved, dai2023refined, liu2023towards}. The work by \cite{zhao2024learning} initiated the study on adversarial MDPs in low-rank MDPs, but their work is restricted to full-information loss feedback. 

When feature learning, bandit feedback, and adversarial losses are combined, the problem becomes highly challenging, and to the best of our knowledge their are no provably efficient algorithms to tackle this setting. In this work, we provide the first result for this combination. We hope that our result would bring new ideas to RL in practice, where all three elements are usually present simultaneously. 
We give several main results, targeting at either tighter regret (i.e., the performance gap between the optimal policy and the learner) or computational efficiency, as summarized in \cref{tab:comparison}. Below we give a brief introduction for each of them. A more thorough related work review is in \cref{app: related work}.

\begin{table}[t]
\caption{Comparison of adversarial low-rank MDP algorithms. $\order$ here hides factors of order $\poly(d, |\cA|, \log T, \log|\Phi||\Upsilon|)$. 
$^\dagger$\cref{alg:model-free policy} assumes access to  $\phi(x,a)$ for any $(\phi,x,a)\in \Phi\times \cX\times \cA$, while other algorithms only require access to $\phi(x,a)$ for any 
$(\phi, a)\in \Phi\times \cA$ on \emph{visited} $x$. }
\label{tab:comparison}
\renewcommand{\arraystretch}{1.5}
\setcellgapes{2pt} % Set minimum gap between cell text and border
\makegapedcells % Apply gap settings
\vspace*{2pt}
\centering
\hspace*{-48pt}
\scalebox{0.99}{
    \begin{tabular}{|c|c|c|c|c|c|}
\hline
\textbf{Feedback}  &  \textbf{Algorithm} & \textbf{Algorithm type} & \textbf{Regret} & \textbf{Efficiency} & \textbf{Loss}\\
\hline
% General CB & \cite{syrgkanis2016improved} & $(\log|\Pi|)^{\nicefrac{1}{3}}(|\calA|T)^{\nicefrac{2}{3}}$ & \checkmark & $\poly(|\calA|, \log|\Pi|, T)$ & ERM Oracle \\
% \hline
\multirow{3}{*}{Full-info} & \cite{zhao2024learning} & Model-based & $\order(T^{\nicefrac{5}{6}})$ & Oracle-efficient &Arbitrary  \\
\cline{2-6}
 &\cref{alg:full-info} & Model-based & $\order(T^{\nicefrac{2}{3}})$ &  Oracle-efficient &Arbitrary \\
\cline{2-6}
 &\textbf{Lower Bound} & & $\Omega(\sqrt{|\cA| T})$ &  & Arbitrary\\
\hline 
\hline
\multirow{5}{*}{\makecell[c]{\vspace{-10mm} Bandit}} 
& \cref{alg:low-rank bandit} & Model-based & $\order(T^{\nicefrac{2}{3}})$ & Inefficient & \makecell{Linear loss\\[-1pt] Unknown loss feature}  \\
\cline{2-6}
 &\cref{alg:model-free policy}$^\dagger$ & Model-free & $\order(T^{\nicefrac{2}{3}})$  & Inefficient & \makecell{Linear loss \\[-1pt] Unknown loss feature } \\
 \cline{2-6}
 & \makecell{\cref{alg:oraceleff} \\ (for oblivious adversary)} & Model-free & $\order(T^{\nicefrac{4}{5}})$ & Oracle-efficient & \makecell{Linear loss\\[-1pt] Unknown loss feature}   \\
\cline{2-6}
 & \makecell{\cref{alg:algorithm_name} \\ (for adaptive adversary)} & Model-free & $\order(T^{\nicefrac{4}{5}})$ & Oracle-efficient & \makecell{Linear loss \\[-1pt] Known loss feature }\\ 
 \cline{2-6}
 & \textbf{Lower Bound} & & $\Omega(\sqrt{|\cX| |\cA| T})$ &  & Arbitrary  \\
 \hline
\end{tabular}
    }
\end{table}

\begin{itemize}[leftmargin=10pt]

\item \textbf{$T^{\nicefrac{2}{3}}$-regret algorithm under full-information feedback (\cref{alg:full-info}). } This setting is studied by the only prior work in adversarial low-rank MDPs \citep{zhao2024learning}, and we greatly improve their $T^{\nicefrac{5}{6}}$ regret bound to $T^{\nicefrac{2}{3}}$.  Our algorithm begins with a model-based initial exploration phase to estimate the transition. It then performs policy optimization where the critic is the $Q$ value induced by the estimated transition and the full information loss.

\item \textbf{$T^{\nicefrac{2}{3}}$-regret model-based/model-free inefficient algorithm  under bandit feedback (\cref{alg:low-rank bandit}, \cref{alg:model-free policy}). } %This algorithm begins with a model-based initial exploration phase to learn an estimated transition and the corresponding features. In the subsequent episodes, we utilize 
\cref{alg:low-rank bandit} starts with a model-based initial exploration phase to learn an estimated transition, and then runs exponential weights over policy space for regret minimization in the second phase. To tackle bandit feedback, we construct a novel loss estimator that leverages the structure of low-rank MDP to perform accurate off-policy evaluation.  \cref{alg:model-free policy} starts with a different exploration phase, where it calls \texttt{VoX} \citep{mhammedi2023efficient} to learn a policy cover; \texttt{VoX} is a model-free, reward-free exploration algorithm. After this initial exploratory phase, the algorithm also applies exponential weights and utilizes the same loss estimator as in \cref{alg:low-rank bandit}. However, due to its model-free nature, certain components of the estimator cannot be directly accessed and must be derived through specific optimizations.

% The loss estimators also carefully leverage the low-rank structure to ensure its bias is only related to the estimated error for features, which can be effectively compensated by incorporating an exploration bonus into the exponential weights.

\item \textbf{$T^{\nicefrac{4}{5}}$-regret model-free oracle-efficient algorithm  under bandit feedback (\cref{alg:oraceleff}, \cref{alg:algorithm_name}). } \cref{alg:oraceleff} also starts with the model-free exploration algorithm \texttt{VoX} \citep{mhammedi2023efficient} to learn a policy cover. After that, the algorithm operates in epochs; during epoch $k$, the algorithm commits to a fixed mixture of policies. This mixture consists of certain exploratory policies (based on the policy cover from the initial phase) and a policy computed using an online learning algorithm based on estimated $Q$-functions from previous epochs (these serve as loss functions). \cref{alg:algorithm_name} deals with the much more challenging setting of an \emph{adaptive} adversary with bandit feedback. Here, we make the additional assumption that the loss feature, which may be different from the feature of the low-rank decomposition, is given. The algorithm is similar to \cref{alg:oraceleff} with key differences outlined in \cref{sec:extension}.

%Mixing with eploratory policies is needed here to estimate the $Q$-functions that the online algorithm uses. 

% \item \textbf{$T^{\nicefrac{5}{6}}$-regret model-free oracle-efficient algorithm against adaptive adversary  under bandit feedback  (\cref{alg:algorithm_name}). } This algorithm deals with the much more challenging setting of an \emph{adaptive} adversary with bandit feedback. Here, we make the additional assumption that the loss feature, which may be different from the feature of the low-rank decomposition, is given. The algorithm is similar to the previous model-free algorithm with key differences outlined in \cref{sec:extension}.

% after computing a policy cover, the algorithm calls $\texttt{RepLearn}$ to compute a feature map $\phi^\rep$. Then, for every $h\in[H]$, the algorithm computes a \emph{spanner}; a set of policies $\Psi^\spanner_h =\{\pi_{h,1},\dots,\pi_{h,d}\}$ that act as an approximate spanner for the set $\{\E^{\pi}[\phi^\rep_h(\x_h,\a_h)]:\pi \in \Pi\}\subseteq \reals^d$. These spanner policies are then used as the exploratory policies after the initial phase; we require these spanner policies instead of policies in the poly cover because the estimation of the $Q$-functions become much more challenging with an adaptive adversary. 
\end{itemize}

%Our algorithms provide new insights on representation learning for RL and also broaden the scope for provably efficient reinforcement learning in adversarial settings. 
%We provide a more thorough related work review in \cref{app: related work}. 

\section{Preliminaries}
\label{sec:prelim}
We study the episodic online reinforcement learning setting with horizon $H$. We consider an MDP $\mathcal M=(\cX, \cA, P_{1:H}^\star)$, where
$\cX$ represents a countable (possibly infinite) state space\footnote{We assume that $\cX$ is
countable only to simplify the presentation. Our results can easily be extended to a continuous state space with an appropriate measure-theoretic
treatment (see e.g.~\cite{mhammedi2023efficient}).}, 
$\cA$ is a finite action space, and $P^\star_h: \cX \times \cA \rightarrow \Delta(\cX)$ denotes the transition kernel from layer $h$ to $h+1$. We assume that the initial state $x_1\in \cX$ is fixed for simplicity without loss of generality.
For any policy $\pi:\cX \mapsto \Delta(\cA)$ and arbitrary set of transition kernels $\{P_h\}_{h\in[H]}$, we let $\P^{P,\pi}$ denote the law over $(\x_{1},\a_1,\dots,\x_H,\a_H)$ induced by the process of setting $\x_1 =x_1$, sampling $\a_1\sim \pi_1(\cdot \mid \x_1)$, then for $h=2,\dots, H$, $\x_{h}\sim P_{h-1}(\cdot \mid \x_{h-1},\a_{h-1})$ and $\a_h \sim \pi_h(\cdot\mid \x_h)$. We let $\E^{P,\pi}$ denote the corresponding expectations. Further, we let $d^{P,\pi}_h(x)\coloneqq \P^{P,\pi}[\x_h= x]$ denote the \emph{occupancy} of $x\in \cX$. We also let $d^{P,\pi}_h(x,a)\coloneqq \P^{P,\pi}[\x_h= x,\a_h=a]$. Further, we let $\E^{\pi}=\E^{P^\star,\pi}$, $\P^{\pi}=\P^{P^\star,\pi}$, and $d^{\pi}_h=d_h^{P^\star,\pi}$. We use $\pi\circ_h \pi'$ to denote a policy that follows $\pi_k(\cdot\mid\cdot)$ for $k<h$ and $\pi_k'(\cdot\mid\cdot)$ for $k\geq h$. Similarly, $\pi\circ_h \pi'\circ_{h'}\pi''$ denotes a policy that follows $\pi_k$ for $k<h$, $\pi_k'$ for $h\leq k<h'$ and $\pi_k''$ for $k>h'$.

%\zmdelete{We assume the state space $\cX$ is organized to have distinct, non-overlapping \textit{layers}, i.e., $\cX=\cX_1 \cup \cX_2 \cup \cdots \cup \cX_H$ where $\cX_{h}\cap \cX_{h'}=\varnothing$ for any $1\le h<h'\le H$ \zm{Do we need this?}, and transition is only possible from one layer to the next. The transition at layer $h$ is determined by $P_h^\star \colon \cX_h \times \cA\to \Delta(\cX_{h+1})$.  For simplicity, and without loss of generality, we assume $\cX_1=\{x_1\}$ as the fixed initial state for every episode. \zm{Do we need this?}}

% $P_h^\star \colon \cX\times \cA\to \Delta(\cX)$ is the transition kernal of step $h$. 

% We assume the state space $\cX$ is organized to have distinct, non-overlapping \textit{layers}, i.e., $\cX=\cX_1 \cup \cX_2 \cup \cdots \cup \cX_H$ where $\cX_{h}\cap \cX_{h'}=\varnothing$ for any $1\le h<h'\le H$, and transition is only possible from one layer to the next, that is, $P(x' \mid x,a) \neq 0$ only when $s \in \cX_h$ and $x' \in \cX_{h+1}$. For simplicity, and without loss of generality, we assume $\cX_1=\{x_1\}$ as the fixed initial state for every episode.

We consider a learner interacting with the MDP $\mathcal M$ for $T$ episodes with adversarial loss functions. Before the game starts, an oblivious adversary chooses the loss functions for all episodes $(\ell^t_{1:H}: \cX \times \cA \rightarrow  [0,1])_{t=1}^T$. For each episode $t \in [T]$, the learner starts at state $\x_{1}^t = x_1$, then for each step $h\in [H]$ within episode $t$, the learner observes state 
$\x_{h}^t\in\cX_h$, chooses an action $\a_h^t \in \cA$, then suffers loss $\ell_{h}^t(\x_{h}^t,\a_{h}^t)$. The state $\x_{h+1}^t$  at the next step is drawn from transition $P^\star_{h}(\cdot \mid \x_{h}^t,\a_{h}^t)$. We consider \textit{bandit feedback} setting where the learner could only observe the losses $\ell^t_1(\x^t_1,\a^t_1), \dots, \ell^t_H(\x^t_H,\a^t_H)$ at the visited state-action pairs.

We let $\Pi\coloneqq \{\pi: \cX \rightarrow \Delta(\cA)\}$ denote the set of Markovian policies. For policy $\pi\in \Pi$, loss $\ell$ and transition kernels $P_{1:H}$, we denote by $Q^{P,\pi}_h(\cdot,\cdot; \ell)$ the \emph{state-action} value function (a.k.a.~$Q$-function) at step $h\in[H]$ with respect to the transitions $P_{1:H}$ and loss $\ell$; that is
\begin{align}
\label{eq: Q-function}
Q^{P,\pi}_h(x,a; \ell) \coloneqq  \E^{P,\pi}\left[\sum_{s=h}^H \ell_s(\x^t_{s},\a^t_{s})  \mid \x_h=x, \a_h =a \right],
\end{align}
for all $(x,a)\in \cX\times \cA$. We let $V^{P,\pi}_h(x; \ell) \coloneqq \max_{a\in \cA} Q^{P,\pi}_h(x,a;\ell)$ be the corresponding \textit{state} value function at layer $h$. Further, we write $Q^{\pi}_h(\cdot,\cdot; \ell)\coloneqq Q^{P^\star,\pi}_h(\cdot,\cdot; \ell)$ and $V^{\pi}_h( \cdot;\ell)\coloneqq V^{P^\star,\pi}( \cdot;\ell)$.

For all of our algorithms except for \cref{alg:algorithm_name}, we aim to construct (possibly randomized) policies $\{\bpi^t\}_{t\in[T]}$ that ensure a sublinear \emph{pseudo-regret} with respect to the best-fixed policy; that is,
\begin{align}
\mathrm{Reg}_T \coloneqq  \min_{\pi\in \Pi}\mathrm{Reg}_T(\pi) \quad \text{where} \quad \mathrm{Reg}_T(\pi) \coloneqq \E\left [\sum_{t=1}^T V_1^{\bpi^t}(x_1; \ell^t) - \sum_{t=1}^T V_1^{\pi}(x_1; \ell^t)\right]. 
\end{align}

For\cref{alg:algorithm_name}, we bound the standard \emph{regret} 
\begin{align}
    \overline{\mathrm{Reg}}_T := \sum_{t=1}^T V_1^{\pi^t}(x_1; \ell^t) - \min_{\pi\in\Pi}\sum_{t=1}^T V_1^{\pi}(x_1; \ell^t) \label{eq:standardreg} 
\end{align}
with high probability. This allows it to handle adaptive adversary.

Throughout, we will assume that the MDP $\cM$ is low-rank with \emph{unknown} feature maps $\phistar_h$ and $\mu^\star_h$.
\begin{assumption}[Low-Rank MDP]
\label{assm:normalizing}
    There exist (unknown) features maps $\phistar_{1:H}:\cX\times \cA\rightarrow \reals^d$ and $\mu^\star_{1:H}:\cX \rightarrow \reals^d$, such that for all $h\in[H-1]$ and $(x,a,x')\in \cX\times \cA\times \cX$:
\begin{align}
\P[\x_{h+1}=x'\mid \x_h=x,\a_{h}=a] =  \phistar_h(x,a)^\top \mu^\star_{h+1}(x'). \label{eq:lr}
\end{align}
   Furthermore, for all $h\in[H]$, the feature maps $\mu^\star_h$ and $\phistar_h$ are such that $\sup_{(x,a)\in \cX\times \cA}\|\phistar_h(x,a)\|\leq 1$ and  $\|\sum_{x\in \cX}g(x)\cdot \mu^\star_h(x)\|\leq \sqrt{d}$, for all $g:\cX\rightarrow [0,1]$.
\end{assumption}

\begin{comment}
\begin{definition}[Low-rank MDP]\label{def: low-rank MDP} In \textit{low-rank MDP}, for any layer $h \in [H]$, each state $x \in \cX_h$ and action $a\in \cA$ is associated with an (unknown) feature $\phi_h^\star(x,a)\in \mathbb R^d$ with $\| \phi_h^\star(s,a) \|_2\le 1$.
There exists an (unknown) mapping $\psi_{h} \colon \cX_{h+1} \to \mathbb{R}^d$ such that the transition can be expressed as
\begin{align*}
P_h^\star(x'\mid x,a) &= \langle \phi_h^\star(x,a),\psi_{h}(x')\rangle,\quad \forall (x,a,x')\in \cX_{h}\times \cA\times \cX_{h+1}.    \label{eq: P assumption}
\end{align*}
For any measurable function $g: \cX \to \mathbb{R}^d$ such that $\|g\|_{\infty} \le 1$, we have $\left\|\int \psi_h(x)g(x) dx\right\|_2 \le \sqrt{d}$ for all $h \in [H]$.
\end{definition}
\end{comment}

\paragraph{Loss function under bandit feedback.} 
For bandit feedback setting, we make the following additional linear assumption on the losses; in the sequel, we will argue that this is necessary to avoid a sample complexity scaling with the number of states. This linear loss assumption also appears in \cite{ren2022free,  zhang2022making} for stochastic low-rank MDPs. Note that for the full-information feedback setting, such an assumption is not required.

\begin{assumption}[Loss Representation]
    \label{assm:linearlossweak}
   For any $t\in[T]$ and layer $h$, there is a vector $g_{h}\indd{t}\in \mathbb{B}_d(1)$ such that the loss $\ell^t_{h}(x,a)$ at round $t$ satisfies:
    \begin{align}
     \forall (x,a)\in \cX\times \cA,\quad   \ell^t_{h}(x,a) = \phi^\star_h(x,a)^\top g_{h}\indd{t}. \label{eq:linearlossweak}
    \end{align} 
    \end{assumption}
We note that there is no loss of generality in assuming that the losses are expressed using the same features $\phistar_{1:H}$ as the low-rank structure in \eqref{eq:lr}. This is because if the losses have different features, we can simply combine these features with the low-rank features, and redefine $\phistar$ accordingly. For the bulk of our results (and as stated in the prequel), we will assume that the losses $\{\ell^t_{h}(\cdot,\cdot)\}_{h\in[H],t\in[T]}$ (or equivalently $\{g_h\indd{t}\}_{h\in[H],t\in[T]}$ under \cref{assm:linearlossweak}) are chosen by an aversary before the start of the game (i.e.~oblvious adversary). In \cref{sec:adaptive}, we will present a model-free, oracle-efficient algorithm for an adaptive adversary. 

\paragraph{Function approximation.}
So far, \cref{assm:normalizing} and \cref{assm:linearlossweak} are in line with assumptions made in the linear MDP setting \citep{jin2020provably}. However, unlike in linear MDPs, we do not assume that the feature maps $\phistar_{1:H}$ are known. To facilitate representation learning and ultimately a sublinear regret, we need to make \emph{realizability} assumptions. In particular, in the model-free setting, we assume we have a function class $\Phi$ that contains the true features $\phistar_{1:H}$. In the model-based setting, we additionally assume access to a function class $\Upsilon$
that contains the feature maps $\mu^\star_{1:H}$\footnote{The setting where we assume access to function classes that realize both $\phistar_{1:H}$ and $\mustar_{1:H}$ is called \emph{model-based} because it allows one to model the transition probabilities, thanks to the low-rank structure in \eqref{eq:lr}.}. We will formalize these assumptions in their corresponding sections in the sequel.

\paragraph{Other notation.}  For $\psi:\Pi \rightarrow \reals^d$, we define $\john(\psi, \Pi)$ as a distribution $\mu\in\Delta(\Pi)$ such that $\|\psi(\pi)\|_{G^{-1}}^2\leq d$ for all $\pi\in\Pi$, where $G=\sum_{\pi\in\Pi} \mu(\pi)\cdot \psi(\pi)\psi(\pi)^\top$. This is the standard John's exploration or $G$-optimal design, which always exists.

\section{Model-based Algorithms for Adversarial Low-rank MDPs}
\label{sec:model-base}
In this section, we discuss adversarial low-rank MDPs under model-based assumption. The model-based assumption is formalized in the \cref{assm: model-base} below. This assumption is standard which also appears in prior works on model-based learning in low-rank MDPs ~\citep{agarwal2020flambe, uehara2021representation, zhang2022making, cheng2023improved, zhao2024learning}.
% \begin{assumption}[Model-based Assumption]\label{assm: model-base}
% For any $h \in [H]$, the learner gets access to a model space $\cM_h = \{(\phi_h, \mu_h): \phi_h \in \Phi_h, \mu_h \in \Psi_h\}$ such that for any $h \in [H]$, $\phi^\star_h \in \Phi_h$ and $\mu_h \in \Psi_h$. Moreover, for any $h \in [H]$, any $(x,a,x') \in \cX_h \times \cA \times \cX_{h+1}$, any $\phi_h \in \Phi_h$, and any $\mu_h \in \Psi_h$, we have $\|\phi_h(s,a)\|_2 \le 1$,  $\int \phi_h(x,a)^\top \mu_h(x') dx' = 1$  and for any measurable function $g: \cX \to \mathbb{R}^d$ such that $\|g\|_{\infty} \le 1$, we have $\left\|\int \mu_h(x)g(x) dx\right\|_2 \le \sqrt{d}$ for all $h \in [H]$. Denote $\Phi = \bigcup_{h=1}^H \Phi_h$ and $\Psi = \bigcup_{h=1}^H \Psi_h$.
% \end{assumption}

\begin{assumption}[Model-based assumption]\label{assm: model-base}
The learner has access to two model spaces $\Phi$ and $\Upsilon$ such that $\phi^\star \in \Phi$ and $\mu^\star \in \Upsilon$. Moreover, for any $\phi \in \Phi$, $\mu \in \Upsilon$, and $h\in[2\ldotst H]$, we have $\sup_{(x,a)\in \cX\times \cA}\|\phi_{h-1}(x,a)\|\leq 1$, $\sum_{x' \in \cX} \phi_{h-1}(x,a)^\top \mu_{h}(x') = 1$ and  $\|\sum_{x\in \cX}g(x)\cdot \mu_{h}(x)\|\leq \sqrt{d}$, for all $g:\cX\rightarrow [0,1]$.
\end{assumption}

% Our setting aligns with~\cite{zhao2024learning} who get an oracle efficient algorithm with $T^{\nicefrac{5}{6}}$ regret. Our \cref{alg:full-info} successfully achieves $\sqrt{T}$ regret with the price of more computational complexity that scale with the size of transition space $|\Phi||\Psi|$.

\subsection{Adversarial Low-rank MDPs with Full Information}
\label{sec:full-info}

We first discuss learning adversarial low-rank MDPs with full information and model-based assumption. This setting aligns with~\cite{zhao2024learning}, and our \cref{alg:full-info} successfully improves their regret from  $T^{\nicefrac{5}{6}}$ to  $T^{\nicefrac{2}{3}}$.

As argued in \cite{zhao2024learning}, the challenge of learning adversarial low-rank MDPs lies in the need for balancing exploration and exploitation both in representation learning and policy optimization over adversarial losses. To tackle this doubled exploration and exploitation challenge, the algorithm of \cite{zhao2024learning} performs \emph{simultaneous} representation learning and policy optimization.  With a closer look at their analysis, we find that there is a drawback of this approach: because their algorithm handles the two tasks at the same time, it spends less exploration for representation learning in the early phase of the algorithm. This results in larger error in the estimated Q-values (i.e., critic) fed to policy optimization, and worsens the overall regret. %To address this issue, we resort to a two-phase algorithm, where pure representation learning is performed in the first phase, and pure policy optimization is performed in the second phase. %   combines the policy optimization framework~\citep{shani2020optimistic} with the exploration bonus developed by~\cite{uehara2021representation} for reward-free learning. 

To address this issue, we design \cref{alg:full-info} as a simple two-phase algorithm that separates representation learning and policy optimization. In the first phase, following \cite{cheng2023improved}, we perform optimal reward-free exploration for low-rank MDPs to estimate the transition. The resulted estimator, $\widehat{P}$, is able to accurately approximate the true transition and give accurate Q-value estimators for any policy. The more accurate Q-estimator allows for more effective policy optimization in the second phase. \cref{thm:full-info} shows the guarantee of \cref{alg:full-info}  where $\otil$ hides logarithmic factors of $d,H,T, |\Phi||\Upsilon|$.

\begin{algorithm}[t]
\caption{Model-Based Algorithm for Full-Information Feedback}\label{alg:full-info}
\begin{algorithmic}[1]
    \State Let $\eta = \frac{1}{H\sqrt{T}}$, $\epsilon = (Hd^2|\cA|(d^2+|\cA|))^{\frac{1}{3}}T^{-\frac{1}{3}}$, and $T_0 = \otil( \epsilon^{-2}H^3d^2|\cA|(d^2 + |\cA|))$. Let $\pi^1$ be a uniform policy.
     \State Run \cite[Algorithm 1]{cheng2023improved} for $T_0$ episodes and get outputs $\hphi \in \Phi, \hpsi \in \Upsilon$.
      \State Define transitions $\hatp_{1:H-1}$ as 
    $$\hatp_h(x'\mid  x,a) = \hphi_h(x,a)^\top \hpsi_{h+1}(x'), \quad \forall (x,a,x') \in \cX \times \cA \times \cX.$$
    \For{$t = T_0+ 1, T_0 + 2,\ldots, T$}
     \State Execute policy $\pi^t$ and observe trajectory $(\x_{1:H},\a_{1:H})$ and full information loss $\ell^t$.
     \State Update policy for all $h\in[H]$: 
     \begin{align*}
             \pi^{t+1}_h(a\mid x) \propto \exp\left( - \eta \sum_{i=1}^{t} \Qhat_h^i(x,a) \right) \ \text{where}\ \Qhat^t_h(x,a) = Q_h^{\hatp,\pi^t}(x,a; \ell^t). 
         \end{align*}  \label{line: exponential weight every state}
     
     %\For{$h=1, \ldots, H$}
     %     \State Perform policy evaluation and define $\Qhat^t_h(x,a) = Q_h^{\hatp^t,\pi^t}(x,a; \ell^t)$. \label{line:Qfun}
     %    \State Update policy through exponential weights 
     %    \begin{align*}
      %       \pi^{t+1}_h(a\mid x) \propto \exp\left( - \eta \sum_{i=1}^{t} \left(\Qhat_h^i(x,a) - \Bhat_h^i(x,a)\right) \right). 
      %   \end{align*}  \label{line: exponential weight every state}
      %  \EndFor
    \EndFor
    \end{algorithmic}
\end{algorithm}

\begin{theorem} \label{thm:full-info}
  \cref{alg:full-info} ensures $
     \mathrm{Reg}_T  \le  \otil\left(H^3 \left(d^2 + |\cA|\right)T^{\frac{2}{3}}\right)$.
\end{theorem}

The proof for \cref{thm:full-info} is given in \cref{app: model based full info}.  \cite{zhao2024learning} also constructs a lower bound $\Omega\left(H\sqrt{d|\cA|T}\right)$ for this settings. Thus, the $\mathrm{poly}(|\cA|)$-dependence is unavoidable.

% One technical difference in our analysis compared with \cite{zhao2024learning} is the way to bound the overhead caused by the bonus. The overhead of bonus is typically a generalization error for online data (i.e. $\sum_{t=1}^T\E^{\pi^t}[b_h^t(\x_h,\a_h)]$) while it is usually guaranteed that the bonus $b^t$ has small in-sample error over the data collected till $t-1$ episodes (i.e. $\sum_{i=1}^{t-1}\E^{\pi^i}[b_h^t(\x_h,\a_h)]$ is small). Thus, a change-of-measure happens here and it is usually handled by the elliptical potential lemma, which is also the case in \cite{zhao2024learning}. Recently, \cite{xie2022role} showed that for the MDPs such that the occupancy induced by any policy could be covered by a fixed distribution, the change-of-measure could happen for any function. We show that low-rank MDPs actually have such ``coverability'' property, and small in-sample error guaranteed by MLE directly leads to well-bounded overhead of the bonus.

% Note that in \cite{zhao2024learning}, a lower bound $\Omega(H\sqrt{d|\cA|T})$ is derived for the same setting. Thus, the regret guarantee of \cref{alg:full-info} is for all $d,|\cA|$ and $T$ if $T \ge d|\cA|$.
\subsection{Model-Based, Computationally Inefficient  Algorithm for Bandit Feedback} \label{sec:model-based bandit}
In this section, following \cref{assm: model-base},  we introduce the first (model-based) algorithm (\cref{alg:low-rank bandit}) for adversarial low-rank MDPs with bandit feedback and sublinear regret. Compared with linear MDPs, the key challenge for more general low-rank MDPs is to construct a proper loss estimator. For linear MDP, since the feature is known, the loss estimator closely resembles that of linear bandits. However, low-rank MDPs lack such structural simplicity, making standard loss estimators invalid. To overcome this challenge, we propose a new loss estimator that works for any loss function based on off-policy evaluation and the low-rank structure of transition. In this section, $\otil$ hides logarithmic factors of $d,H,T, |\Phi||\Upsilon|$.

% Our model-based assumption  is formalized in the \cref{assm: model-base} below. This assumption is standard which also appears in prior works on model-based learning in low-rank MDPs ~\citep{agarwal2020flambe, uehara2021representation, zhang2022making, cheng2023improved, zhao2024learning}.

% \begin{assumption}[Model-based assumption]\label{assm: model-base}
% The learner has access to two model spaces $\Phi$ and $\Upsilon$ such that $\phi^\star \in \Phi$ and $\mu^\star \in \Upsilon$. Moreover, for any $\phi \in \Phi$, $\mu \in \Upsilon$, and $h\in[2\ldotst H]$, we have $\sup_{(x,a)\in \cX\times \cA}\|\phi_{h-1}(x,a)\|\leq 1$, $\sum_{x' \in \cX} \phi_{h-1}(x,a)^\top \mu_{h}(x') = 1$ and  $\|\sum_{x\in \cX}g(x)\cdot \mu_{h}(x)\|\leq \sqrt{d}$, for all $g:\cX\rightarrow [0,1]$.
% \end{assumption}

In \cref{alg:low-rank bandit}, we first conduct an initial representation learning phase to establish accurate transition estimator $\hatp$  and its corresponding features $\hphi$ and $\hpsi$ based on reward-free exploration algorithms in \cite{cheng2023improved}. Then, in the second phase, we use exponential weights to maintain a distribution over the policy space $\Pi'$ where we mix a uniform policy with  $\Pi$ to enhance exploration. At every round $t$, a behavior policy $\pi^t$ is chosen from the current policy distribution, and we use the data collected by $\pi^t$ to estimate the value for every $\pi \in \Pi'$. The success of such off-policy evaluation is based on the following observations of low-rank MDP. Using the low-rank transition structure, for $h \ge 2$, we have 
\begin{align} \label{eqn:low-rank}
  \forall \pi\in \Pi,\quad   d_h^{\pi}(x) =  \phi_{h-1}^\star(\pi)^\top \mu_{h}^\star(x),\quad \text{where}\quad \phi_{h-1}^\star(\pi) \coloneqq \E^{\pi}[\phistar_{h-1}(\x_{h-1},\a_{h-1})].
\end{align}
Thus, using the definition of the $V$-function from \cref{sec:prelim}, we have for any loss function $\ell$ and $\pi$:
\begin{align} \label{eqn:V dec}
    V^{\pi}_1(x_1; \ell)- \E[\ell_1(\x_1,\a_1)]= \sum_{h=2}^H \sum_{(x,a)\in \cX\times \cA}  \phi_{h-1}^{\star} (\pi)^\top \mu_{h}^\star(x) \pi_h(a\mid x)\cdot \ell_h(x,a).
\end{align}
Letting $\Lambda_h^t \coloneqq (\E_{\bpi^t \sim p^t}\left[\phi_h^\star(\bpi^t)\phi_h^{\star}(\bpi^t)^\top\right])^{-1}$ and ingnoring the loss term $\E[\ell_1(\x_1,\a_1)]$ from the first step (this term can easily be treated as in a bandit setting with $H=1$), we have for all $\pi\in \Pi'$:
\begin{align*}
     V_1^{\pi}(x_1; \ell)  &=  \E_{\bpi^t \sim p^t}\left[\sum_{h=2}^H \sum_{(x,a)\in \cX\times \cA}  \phi_{h-1}^{\star} (\pi)^\top\Lambda_{h-1}^t \phi_{h-1}^\star(\bpi^t)\phi_{h-1}^{\star}(\bpi^t)^\top \mu_{h}^\star(x) \pi_h(a\mid x)\cdot \ell_h(x,a)\right],
    \\&= \E_{\bpi^t \sim p^t}\left[\sum_{h=2}^H \sum_{x,a}  \phi_{h-1}^{\star} (\pi)^\top \Lambda_{h-1}^t \phi_{h-1}^\star(\bpi^t)\cdot d_h^{\bpi^t}(x)\bpi_h^t(a\mid x)\frac{\pi_h(a\mid x)}{\bpi_h^t(a\mid x)}\cdot \ell_h(x,a)\right],
    \\&= \E_{\bpi^t \sim p^t}\E^{\bpi^t}\left[ \sum_{h=2}^H \phi_{h-1}^{\star}(\pi)^\top\Lambda_{h-1}^t \phi_{h-1}^\star(\bpi^t)\cdot \frac{\pi_h(\a_h\mid \x_h)}{\bpi_h^t(\a_h\mid \x_h)}\cdot \ell_h(\x_h,\a_h)\right].
\end{align*}
Thus, for all $\ell$ and $\pi$, $\sum_{h=2}^H \phi_{h-1}^{\star}(\pi)^\top\Lambda_{h-1}^t \phi_{h-1}^\star(\bpi^t)\cdot {\pi_h(\a_h\mid \x_h)}\bpi_h^t(\a_h\mid \x_h)^{-1}\cdot \ell_h(\x_h,\a_h)$ for $\bpi^t\sim p^t$ and $(\x_h,\a_h)\sim d^{\bpi^t}_h$ is an unbiased estimator of $V^\pi_1(x_1, \ell)$. However, $\phi_{h-1}^\star(\pi)$ is not accessible because both the true feature $\phi^\star$ and occupancy $d^\pi$ for the true transition are unknown. Thus, our estimator incorporates the learned feature $\hphi$ and the occupancy of $\hatp$ instead as shown in \cref{line:est feature} and \cref{line:bandit estimator}. Utilizing estimated features and transition could introduce additional bias but the initial representation learning already ensures such bias is small enough to tackle. We compensate for the bias by incorporating an exploration bonus $b^t(\pi)$ in exponential weights. To further encourage exploration, we additionally perform John's exploration together with exponential weights when selecting behavior policies. The main guarantee of \cref{alg:low-rank bandit} is given in \cref{thm:model-based bound}.

% This technique is also widely used in learning adversarial linear MDPs ~\citep{luo2021policy, dai2023refined, saha2020improved, kong2023improved, liu2023towards}. 

% and technically speaking, introduce boundness property to our loss estimator, we additionally mix a uniform distribution in our policy space and perform John's exploration together with exponential weights when selecting behaviour policies. 

% Let $\Pi$ be policy set in which every policy is a mixture of a deterministic policy and $\beta$ uniform distribution.

\begin{algorithm}[t]
\caption{Model-Based Algorithm for Bandit Feedback} \label{alg:low-rank bandit}
\begin{algorithmic}[1]
    \Input A policy class $\Pi$. 
    \State Set $\epsilon = T^{-\frac{1}{3}}$, $\gamma = T^{-\frac{1}{3}}$, $\beta = T^{-\frac{1}{3}}$, $\eta = ({4Hd|\cA|})^{-1} T^{-\frac{2}{3}}$, and $T_0 = \otil( \epsilon^{-2}H^3d^2|\cA|(d^2 + |\cA|))$.
    \State Run \cite[Algorithm 1]{cheng2023improved} for $T_0$ episodes and get outputs $\hphi \in \Phi, \hpsi \in \Upsilon$. 
    \State Define transitions $\hatp_{1:H-1}$ as 
    $$\hatp_h(x'\mid  x,a) = \hphi_h(x,a)^\top \hpsi_{h+1}(x'), \quad \forall (x,a,x') \in \cX \times \cA \times \cX.$$
    \State For all $h\in[H-1]$, define $\hphi_h(\pi) = \sum_{(x,a)\in \cX\times \cA} \hd_{h}^\pi(x,a)\cdot \hphi_{h}(x,a)$, where $\hd_h^\pi \coloneqq d_h^{\hatp,\pi}$. \label{line:est feature}
    \State Define the policy space $\Pi' = \{\pi':~ \exists \pi\in\Pi, \ \ \pi_h'(\cdot\mid x)= (1-\beta)\pi_h(\cdot\mid x) +  \beta/|\cA|, \ \  \forall x, h\}$. %\zm{Define $\pi_\unif$ in prelim} 
    \label{line:mix pi}
    \For{$t = T_0 + 1,\, T_0 + 2,\ldots, T$}
         \State Define $ p^t(\pi) \propto \exp\left(-\eta \sum_{i=1}^{t-1}  \left(\hatell^i(\pi) - b^i(\pi)\right)\right),$ for all $\pi \in \Pi'$. \label{line:bandit-exp}
         \State Let $\tilp^t(\pi) = (1-\gamma)p^t(\pi) +  \frac{\gamma}{H-1}\sum_{h=1}^{H-1}  J_h$, where $J_h = \john(\hatphi_h(\cdot), \Pi')$. \algcommentlight{$\john$ as in \S\ref{sec:prelim}}  %\zm{Do we define \text{John} somewhere? Also use macro john}  \cwcomment{I made some edit here. John's exp defined in prelim}
         \State Execute policy $\bpi^t\sim \tilp^t$ and observe trajectory $(\x^t_{1:H}, \a^t_{1:H})$ and losses $\bell^t_h =\ell^t_h(\x_h^t,\a_h^t)$.
         \State Define $\Sigma^t_{h} = \sum_{\pi \in \Pi'}\tilp^t(\pi)\cdot \hatphi_h(\pi)\hatphi_h(\pi)^\top$, $b^t(\pi) = \sqrt{d}H\epsilon \cdot \sum_{h=1}^{H-1} \|\hatphi_h(\pi)\|_{(\Sigma^t_{h})^{-1}}$, and
             \begin{align*}  
            \hell^t(\pi) = \frac{\pi_1(\a_1^t\mid \x_1^t)}{\bpi_1^t(\a_1^t\mid \x_1^t)}\bell_1^t + \sum_{h=2}^H \hphi_{h-1}(\pi)^\top\left(\Sigma_{h-1}^t\right)^{-1} \hphi_{h-1}(\bpi^t) \frac{\pi_h(\a_h^t\mid \x_h^t)}{\bpi_h^t(\a_h^t\mid \x_h^t)}\bell_h^t.
             \end{align*} \label{line:bandit estimator}
         \EndFor
    \end{algorithmic}
\end{algorithm}

\begin{theorem}\label{thm:model-based bound}
\cref{alg:low-rank bandit} achieves $\mathrm{Reg}_T(\pi) \le  \otil\left(d^2 H^3 |\cA|(d^2 + |\cA|)T^{\nicefrac{2}{3}}\log|\Pi|\right)$ for any $\pi \in \Pi$.
\end{theorem}

Note that the guarantee in \cref{thm:model-based bound} only holds for policy $\pi \in \Pi$. To ensure our regret bound is meaningful, at least a near-optimal policy should be contained in the given policy set $\Pi$. In general, the size of such a policy set would grow exponentially with the number of states (e.g.~covering of all Markovian policies), making the regret have polynomial dependence on the number of states. In \cref{thm:lowerboundmain}, we show that even for low-rank MDPs, if the loss function lacks structure, the regret cannot avoid polynomial dependence on the number of states. The detailed construction for this lower-bound is given in \cref{sec:lower bound}.

\begin{theorem}\label{thm:lowerboundmain}
There exists a low-rank MDP with $|\cX|$ states, $|\cA|$ actions and sufficiently large $T$ with unstructured losses such that any agent suffers at least regret of $\Omega(\sqrt{|\cX| |\cA| T})$.
\end{theorem}

\cref{thm:lowerboundmain} shows that under bandit feedback, in general, we could not gain too much from low-rank transition structure compared with tabular MDPs. This contrasts with the $\Omega(\sqrt{|\cA| T})$ lower bound in the full information settings (\cite{zhao2024learning}). To get rid of any dependence on the number of states, we additionally introduce \cref{assm:linearlossweak} to impose linear structure on the loss function. Unlike linear MDPs that require the loss feature to be known, our algorithm can even handle linear loss with unknown feature, since our \cref{alg:low-rank bandit} never explicitly uses the loss feature. The linear structure is only used to control the size of the candidate policy class in the analysis (i.e., making $\log|\Pi|$ irrelevant to the number of states). Specifically, when both loss and transition are linear, the Q-function is also linear, making it sufficient to consider the following linear policy space: 
\begin{align*}
    \Pi_{\text{\rm lin}} = \left\{\pi :\cX\rightarrow \cA \ \Big\mid \   \pi_h(a\mid x ) = \mathbb{I}\big\{a=\argmin_{a \in \cA} \phi_h(x, a)^\top \theta_h\big\}, \ h\in[H],\  \|\theta_h\|_2 \le \sqrt{d}H T, \ \phi \in \Phi\right\}.
\end{align*}
The $\frac{1}{T}$-cover of $ \Pi_{\text{\rm lin}}$ only have size $|\Phi|\cdot T^{\mathcal{O}(d)} $ following standard arguments (e.g, Exercise 27.6 of \cite{lattimore2020bandit}) and if we feed it into \cref{alg:low-rank bandit}, our regret could avoid dependece on the size of state space as shown in \cref{cor:linear policy reg}.

\begin{corollary}\label{cor:linear policy reg}
If the loss function satisfies \cref{assm:linearlossweak}, applying \cref{alg:low-rank bandit} with $\Pi$ as the $\frac{1}{T}$-cover of $ \Pi_{\text{\rm lin}}$ ensures $\mathrm{Reg}_T   \le  \otil\left(d^3 H^3 |\cA|(d^2 + |\cA|)T^{\nicefrac{2}{3}}\right)$.
\end{corollary}

% % \section{Model Free, Oracle Efficient Algorithm (Oblivious Adversary)}
% \label{sec:model-free}
% \section{Model Free, Oracle Efficient Algorithm (Oblivious Adversary)}
\section{Model-free Algorithms for Adversarial Low-rank MDPs}
\label{sec:model-free}

In this section, we consider the model-free setting, where we only assume access to a feature class $\Phi$ that contains the true feature map $\phistar$. 
\begin{assumption}[Model-free realizability]
\label{assm:real}
    The learner has access to a function class $\Phi$ such that
\begin{align}
    \phistar \in \Phi \qquad \text{and} \qquad \sup_{\phi\in \Phi} \sup_{(x,a)\in \cX\times \cA} \|\phi(x,a)\| \leq 1. \label{eq:phibound}
\end{align}
\end{assumption}
This is a standard assumption in the context of model-free RL~\citep{modi2024model, zhang2022efficient, mhammedi2024efficient}. We note that having access to the function class $\Phi$ alone (instead of both $\Phi$ and $\Upsilon$ as in \cref{assm: model-base}) is not sufficient to model the transition probabilities (unlike in the model-based case). This makes the model-free setting much more challenging; in fact, until the recent work by \cite{mhammedi2023efficient} there were no model-free, oracle-efficient algorithms for this setting that do not require any additional structural assumptions on the MDP.

\subsection{Model-free, Inefficient Algorithm for Bandit Feedback}
\label{sec:model-free inefficient}

Our first algorithm follows the same structure as \cref{alg:low-rank bandit} but incorporates a model-free initial exploration phase introduced by \cite{mhammedi2023efficient}. Unlike the model-based exploration phase, which directly provides an estimated transition, the model-free exploration phase outputs a policy cover. This policy cover can be combined with the optimization in Algorithm 1 of \cite{liu2023towards} to solve the expected feature $\hat{\phi}$, which is then used in the loss estimator in \cref{line:bandit estimator} of \cref{alg:low-rank bandit}. The algorithm, summarized in \cref{alg:model-free policy}, is inefficient but achieves $T^{\nicefrac{2}{3}}$ regret. More details and proofs can be found in \cref{app:model-free ineff}.

\subsection{Model-free, Oracle Efficient Algorithm for Bandit Feedback (Oblivious Adversary)}
\begin{algorithm}[t]
\caption{Oracle Efficient Algorithm for Adversarial Low-Rank MDPs (Oblivious Adversary).}
\label{alg:oraceleff}
\begin{algorithmic}[1]
    \setstretch{1.2}
\Input Number of rounds $T$, feature class $\Phi$, confidence parameter $\delta\in (0,1)$.
\State Set $\veps \gets T^{-1/3}$, $N_\reg \gets T^{2/3}$, $\nu \gets N^{-1/2}_\reg$, and $T_0\gets \Tcovval$.\label{line:T00}
\State Get $\Psi^\cov_{1:H}\gets \texttt{VoX}(\Phi, \veps, \delta)$. \label{line:explore}
\hfill  \algcommentlight{Compute policy cover with $\texttt{VoX}$ \citep{mhammedi2023efficient}.}
\For{$k=1,\dots,{(T-T_0)}/{\nreg}$}
% \State Define $\pihat_h\ind{k}(a | x)\in \argmin_{\pi \in \Delta(\cA)}\sum_{s < k} \sum_{a\in \cA}  \left(\eta \pi(a\mid \cdot) \Qhat\ind{s}_h(\cdot,a) +\log\frac{1}{\pi(a\mid \cdot)}\right)$ for $h\in[H]$.  \label{line:pihat0}   
\State Define $\pihat_h\ind{k}(a~|~x) \propto \exp{\left(-\eta \sum_{s < k}  \Qhat\ind{s}_h(x,a)\right)}$ for $h\in[H]$.  \label{line:pihat0}   
\For{$t=T_0+(k-1)\cdot N_\reg +1,\ \dots,\  T_0+k \cdot N_\reg $} 
%\State Define the random variables $\bzeta\indd{t} \sim \mathrm{Ber}(\nu)$, $\bh\indd{t}\sim \unif([H])$, and $\bpi\indd{t} \sim \unif(\Psi^{\texttt{cov}}_{\bh\indd{t}})$.
\State Sample variables $\bzeta\indd{t} \sim \mathrm{Ber}(\nu)$, $\bh\indd{t}\sim \unif([H])$, and $\bpi\indd{t} \sim \unif(\Psi^{\texttt{cov}}_{\bh\indd{t}})$.
\State Set $\bhpi^t = \mathbb{I}\{\bzeta\indd{t} = 0\} \cdot  \pihat\ind{k}+ \mathbb{I}\{\bzeta\indd{t} = 1\} \cdot \bpi\indd{t}\circ_{\bh\indd{t}}   \pi_\unif \circ_{\bh\indd{t}+1} \pihat\ind{k}.$
%\State Let $\rho\ind{k}\in \Delta(\Pi)$ be the distribution of the random policy \[\mathbb{I}\{\bzeta\indd{t} = 0\} \cdot  \pihat\ind{k}+ \mathbb{I}\{\bzeta\indd{t} = 1\} \cdot \bpi\indd{t}\circ_{\bh\indd{t}}   \pi_\unif \circ_{\bh\indd{t}+1} \pihat\ind{k}.\] 
\State Execute $\bhpi^t$, and observe trajectory $(\x\indd{t}_1,\a\indd{t}_1,\dots,\x\indd{t}_H,\a\indd{t}_H )$.
\State For $h\in[H]$, observe loss $\bell_h\indd{t}\coloneqq \ell^t_h(\x_h\indd{t},\a_h\indd{t})$. 
\EndFor
\State For $h\in[H]$ and $\cI\ind{k}= \{ T_0+(k-1)\cdot N_\reg+1,\ \dots,\ T_0+ k \cdot N_\reg\}$, compute $(\phihat\ind{k}_h,\hat{\theta}\ind{k}_h)$: 
 \begin{align}
	(\phihat\ind{k}_h,\hat{\theta}\ind{k}_h) &\gets \argmin_{(\phi,\theta)\in \Phi\times \bbB_{d}(H\sqrt{d})}  \sum_{t\in \cI\ind{k}} \left(\phi_h(\x\indd{t}_h,\a\indd{t}_h)^\top \theta -\sum_{s=h}^H \bell\indd{t}_s\right)^2 \cdot \mathbb{I}\{\bm{\zeta}^t=0 \text{ or } \bm{h}^t \leq h\}. \label{eq:regproblem0} %\\
 \end{align}
 \State Set $\Qhat\ind{k}_h(x,a) =  \phihat\ind{k}_h(x,a)^\top \hat{\theta}\ind{k}_h$, for all $(x,a)\in \cX\times \cA$. 
 \EndFor
\end{algorithmic}
\end{algorithm}
We now descibe the key component of our efficient model-free algorithm (\cref{alg:oraceleff}).
\paragraph{Exploration phase and policy cover.}
Similar to algorithms in the previous section, \cref{alg:oraceleff} begins with a reward-free exploratorion phase; \cref{line:explore} of \cref{alg:oraceleff}. However, unlike in the previous sections where the role of this exploration phase was to learn a model for the transition probabilities, here the goal is to compute a, so called, \emph{policy cover} which is a small set of policies that can be used to effectively explore the state space. 
\begin{definition}[Approximate policy cover] 
    \label{def:cover}For $\alpha,\veps\in(0,1]$ and $h\in[H]$, a subset $\Psi\subseteq \Pi$ is an $(\alpha, \veps)$-policy cover for layer $h$ if 
\begin{align}
    \max_{\pi\in \Psi} d_h^{\pi}(x)\geq \alpha \cdot \max_{\pi'\in \Pi} d_h^{\pi'}(x), \quad  \text{for all $x\in \cX$ such that } \max_{\pi'\in \Pi} d^{\pi'}_h(x)\geq \veps \cdot \|\mu^\star_h(x)\|.
\end{align}
\end{definition}
In \cref{line:explore}, \cref{alg:oraceleff} calls $\texttt{VoX}$ \citep{mhammedi2023efficient}, a reward-free and model-free exploration algorithm to compute $(1/(8Ad),\veps)$-policy covers $\Psi_{1}^\cov, \dots ,\Psi_H^\cov$ for layers $1,\dots, H$, respectively, with $|\Psi_h|=d$ for all $h\in[H]$. This call to $\texttt{VoX}$ requires $O(1/\veps^2)$ episodes; see the guarantee of \texttt{VoX}{} in \cref{lem:vox}. 
After this initial phase, the algorithm operates in epochs, each consisting of $N_\reg\in\mathbb{N}$ episodes, where in each epoch $k\in[K]$, the algorithm commits to executing policies sampled from a fixed policy distribution $\rho\ind{k}\in \Delta(\Pi)$ with support on the policy covers $\Psi^\cov_{1:H}$ and a policy $\pihat\ind{k}$ specified by an online learning algorithm. Next, we describe in more detail how $\rho\ind{k}$ is constructed and motivate the elements of its construction starting with the online learning policies $\{\pihat\ind{k}\}_{k\in[K]}$.

\paragraph{Online learning policies.}
Given estimates $\{\Qhat_{1:H}\ind{s}\}_{s<k}$ of the average $Q$-functions \begin{align}\left\{\frac{1}{N_\reg} \sum_{t\text{ in epoch }s} Q^{\pihat\ind{s}}_{1:H}(\cdot,\cdot;\ell^t)\right\}_{s<k}\label{eq:avgQ} \end{align} from the previous epoch (we will describe how these estimates are computed in the sequel), \cref{alg:oraceleff} computes policy $\pihat\ind{k}$ for epoch $k$ according to 
\begin{align}
 \pihat_h\ind{k}(a~|~x) \propto \exp{\left(-\eta \sum_{s < k}  \Qhat\ind{s}_h(x,a)\right)}, \label{eq:onlinepolicy}
\end{align}
for all $h\in[H]$. Given a state $x\in \cX$, such exponential weight update ensures a sublinear regret with respect to the sequence of loss functions given by $\{\pi(\cdot~|~x) \mapsto \sum_{h\in[H]}\Qhat_{h}\ind{k}(x, \pi_h(\cdot~|~x))\}_{k\in[K]}$. Thanks to the performance difference lemma, and as shown in \cite{luo2021policy}, a sublinear regret with respect to these "surrogate" loss functions translates into a sublinear regret in the low-rank MDP game we are interested in, granted that $\{\Qhat_{1:H}\ind{k}\}_{k\in[K]}$ are good estimates of the average $Q$-functions \citep{luo2021policy}. In line with previous analyses, we require the $Q$-function estimates to ensure that the following bias term 
\begin{align}
\E^{\pi}\left[\max_{a\in\cA} \left(\frac{1}{N_\reg}\sum_{t\text{ in epoch } k}Q_h^{\pihat\ind{k}}(\x_h,a;\ell^t)-\Qhat\ind{k}_h(\x_h,a)\right)^2\right] \label{eq:biases}
\end{align}
is small for all $h\in[H]$, $k\in[K]$, and $\pi\in \Pi$. 

\paragraph{$Q$-function estimates.} Thanks to the low-rank MDP structure and the linear loss representation assumption (\cref{assm:linearlossweak}), the avegare $Q$-functions in \eqref{eq:avgQ} are linear in the feature maps $\phistar$. Thus, using the function class $\Phi$ in \cref{assm:normalizing} we can estimate these average $Q$-functions by regressing the sum of losses $\sum_{s=h}^H\bell_s^t$ onto $(\x_h^t,\a_h^t)$ for $t$ in the $k$th epoch (as in \eqref{eq:regproblem0}). However, na\"ively doing this using only trajectories generated by $\pihat\ind{k}$ would only ensure that the bias term in \eqref{eq:biases} is small for $\pi = \pihat\ind{k}$. To ensure that it is small for all possible policies $\pi$'s, we need to estimate the $Q$-functions on the trajectories of policies that are guaranteed to have good state coverage; this is where we use the policy cover from the initial phase. 

\paragraph{Mixture of policies.}
At episode $t$ in each epoch $k\in[K]$, we execute policy $\bhpi^t$ sampled from $\rho\ind{k}$, where $\rho\ind{k}$ is the distribution of the random policy: 
\begin{align}
 \mathbb{I}\{\bzeta\indd{t} = 0\} \cdot  \pihat\ind{k}+ \mathbb{I}\{\bzeta\indd{t} = 1\} \cdot \bpi\indd{t}\circ_{\bh\indd{t}}   \pi_\unif \circ_{\bh\indd{t}+1} \pihat\ind{k}, \label{eq:this}
 \end{align}
with $\bzeta\indd{t} \sim \mathrm{Ber}(\nu)$, $\bh\indd{t}\sim \unif([H])$, and $\bpi\indd{t} \sim \unif(\Psi^{\texttt{cov}}_{\bh\indd{t}})$. In words, at the start of each episode of any epoch $k$, we execute $\pihat\ind{k}$ (see \eqref{eq:onlinepolicy}) with probability $1-\nu$; and with probability $\nu$, we execute a policy in $\Psi^\cov_{1:H}$ selected uniformly at random. As explained in the previous paragraph, this ensures a small bias for all choices of $\pi$ in \eqref{eq:biases} thanks the policy cover property of $\Psi^\cov_{1:H}$. We now state the guarantee of \cref{alg:oraceleff}. 

% \subsection{Algorithm Guarantee}
% We now state the guarantee of \cref{alg:oraceleff}. 
\begin{theorem}
\label{thm:oraclealg}
Let $\delta\in(0,1)$ be given and suppose \cref{assm:normalizing} and \cref{assm:linearlossweak} hold. This, for $T = \mathrm{poly}(A,d,H, \log(|\Phi|/\delta))$ sufficiently large, \cref{alg:oraceleff} guarantees $\mathrm{Reg}_T \leq \poly(A,d,H,\log(|\Phi|/\delta))\cdot T^{\nicefrac{4}{5}}$ regret against an oblivious adversary.
\end{theorem}
The proof is in \cref{sec:oracleeff}. Note that the $T$-dependence in this regret even outperforms that of the previous best bound by \cite{zhao2024learning} (see \cref{tab:comparison}). Compared to their algorithm, \cref{alg:oraceleff} is model-free and only requires bandit feedback. This makes the result in \cref{thm:oraclealg} rather surprising.

\subsection{Model-free, Oracle Efficient Algorithm (Adaptive Adversary)}
\label{sec:extension}
\label{sec:adaptive}
In this section, we present a variant of \cref{alg:oraceleff} (\cref{alg:algorithm_name}) that guarantees a sublinear regret against an adaptive adversary. Given the difficulty of this setting, we make the additional assumption that the algorithm has access to the loss feature $\phi^\loss$, which may be different than the low-rank MDP feature $\phi^\star$ (unlike in \cref{assm:linearlossweak}).
\begin{assumption}[Loss Representation]
    \label{assm:linearloss}
    There is a (known) feature map $\phi^\loss$ satisfying $\sup_{h\in[H],(x,a)\in \cX\times \cA}\|\phi^\loss_h(x,a)\|\leq 1$ and such that for any round $h\in[H]$, $t\in[T]$, and history $\cH\indd{t-1}=(x_{1:H}\indd{1:t-1},a_{1:H}\indd{1:t-1})$, the loss function at round $t$ satisfies
    \begin{align}
     \forall (x,a)\in \cX\times \cA,\quad   \ell_{h}(x,a;\cH\indd{t-1}) = \phi^\loss_h(x,a)^\top \g _{h}\indd{t}, \label{eq:linearloss}
    \end{align} 
    for some $\g _{h}\indd{t}\in \bbB_d(1)$.
    \end{assumption}
Note that \cref{assm:linearloss} asserts that the loss at round $t$ depends only on the history $\cH^{t-1}$ and the current state action pair.
Before moving forward, we introduce some additional notation we will use throughout this section.    
\paragraph{Additional notation.} For any two feature maps $\phi,\psi:\cX\times \cA \rightarrow \reals^d$, we denote by $[\phi,\psi]: \cX \times\cA \rightarrow \reals^{2d}$ the vertical concatenation of the two feature maps. For any $h\in[H]$, $t\in[T]$, policy $\pi\in \Pi$, and history $\cH\indd{t-1}=(x_{1:H}\indd{1:t-1},a_{1:H}\indd{1:t-1})$, we denote by $Q^\pi_h(\cdot, \cdot;\cH\indd{t-1})$ the $Q$-function at layer $h$ corresponding to rollout policy $\pi$; that is,
 \begin{align}
    Q^\pi_h(x,a;\cH\indd{t-1}) \coloneqq \E^{\pi}\left[ \sum_{s=h}^H \ell_s(\x_s,\a_s; \cH\indd{t-1}) \mid \x_h = x, \a_h =a\right].
 \end{align}
 Finally, we let $V_h^{\pi}(x;\cH\indd{t-1})\coloneqq \max_{a\in \cA}Q_h^{\pi}(x,a;\cH\indd{t-1})$ denote the corresponding $V$-function.
 % \subsection{Algorithm Description}
 
\cref{alg:algorithm_name} is similar to \cref{alg:oraceleff} with the following key differences; after computing a policy cover, the algorithm calls $\texttt{RepLearn}$ (a representation learning algorithm initially introduced by \cite{modi2024model} and subsequently refined by \cite{mhammedi2023efficient}) to compute a feature map $\phi^\rep$. Then, for every $h\in[H]$, the algorithm computes a \emph{spanner}; a set of policies $\Psi^\spanner_h =\{\pi_{h,1},\dots,\pi_{h,2d}\}$ that act as an approximate spanner for the set $\{\E^{\pi}[\phi^\rep_h(\x_h,\a_h), \phi^\loss_h(\x_h,\a_h)]:\pi \in \Pi\}\subseteq \reals^{2d}$, where we use $[\cdot,\cdot]$ to denote the vertical concatenation of vectors. These spanner policies are then used as the exploratory policies after the initial phase; that is, at episode $t$ in each epoch $k\in[K]$, we execute policy $\bpi\ind{k}$ sampled from $\rho\ind{k}$, where $\rho\ind{k}$ is set to be the distribution of the random policy: $\mathbb{I}\{\bzeta\indd{t} = 0\} \cdot  \pihat\ind{k}+ \mathbb{I}\{\bzeta\indd{t} = 1\} \cdot \bpi\indd{t}  \circ_{\bh\indd{t}+1} \pihat\ind{k}$,
with $\bzeta\indd{t} \sim \mathrm{Ber}(\nu)$, $\bh\indd{t}\sim \unif([H])$, and $\bpi\indd{t} \sim \unif(\Psi^{\texttt{span}}_{\bh\indd{t}})$. Here, the main difference to \cref{alg:oraceleff} (see also \eqref{eq:this}) is that we use $\bpi^t\sim \unif(\Psi^\spanner_{\bh^t})$ instead of $\bpi^t\sim \unif(\Psi^\cov_{\bh^t})$. We require these spanner policies instead of policies in the policy cover, as an adaptive adversary's history-dependent losses prevent standard least squares regression due to the lack of permutation invariance of state-action pairs across episodes within an epoch. Estimating the Q-functions is thus more complex, and we approach it in expectation over roll-ins using policies in $\Psi^{\texttt{span}}$, the ``in-expectation'' estimation task is in a sense easier.

% \subsection{Algorithm Guarantee}
We now state the guarantee of \cref{alg:algorithm_name}. 
\begin{theorem}
\label{thm:adaptive}
Let $\delta\in(0,1)$ be given and suppose that \cref{assm:normalizing} and \cref{assm:linearloss} hold. Then, for $T = \mathrm{poly}(A,d,H, \log(|\Phi|/\delta))$ sufficiently large, \cref{alg:algorithm_name} guarantees with probability at least $1-\delta$, \begin{align} \sum_{t\in [T]} V_h^{\bpi^t}(x_1;\bcH\indd{t-1}) - \min_{\pi\in \Pi} \sum_{t\in [T]} V_h^{\pi}(x_1;\bcH\indd{t-1}) 
 \leq \poly(A,d,H,\log(|\Phi|/\delta))\cdot T^{\nicefrac{4}{5}},
\end{align}
where $\bpi^t$ is the policy that \cref{alg:algorithm_name} executes at episode $t\in[T]$.
\end{theorem}

\begin{algorithm}[H]
\caption{Oracle Efficient Algorithm for Adversarial Low-Rank MDPs (Adaptive Adversary).}
\label{alg:algorithm_name}
\begin{algorithmic}[1]
    \setstretch{1.2}
\Input Number of rounds $T$, feature class $\Phi$, loss feature $\phi^\loss$, confidence parameter $\delta\in (0,1)$.
\State Set $\veps \gets T^{-1/3}$,  $N_\reg \gets T^{2/3}$, $\nu \gets N^{-1/4}_\reg$, and $\alpha \gets (8 Ad)^{-1}$, $T_\cov\gets \Tcovval$. 
\State Set $T_\rep \gets \Trepvalf$, $T_\spanner\gets \Tspanvalf$, 
\State Define $\cF_h = \left\{(x,a)\mapsto \max_{a\in \cA} \phibar_{h}(x,a)^\top \thetabar  \mid \phibar_{h}=[\phi^\loss_{h},\phi_{h}], \phi \in \Phi, \thetabar \in \bbB_{2d}(1)\right\}$, $\forall h\in[H]$.
\State Get $\Psi^\cov_{1:H}\gets \texttt{VoX}(\Phi, \veps, \delta/4)$.
%\hfill  \algcommentlight{Call $\texttt{VoX}$ as in \cite{mhammedi2023efficient} to compute a policy cover.}
\State Get $\phi^\rep_h \gets \texttt{RepLearn}(h,\cF_{h+1},\Phi,\unif(\Psi^\cov_{h}),T_\rep)$, for all $h\in[H-1]$. \hfill \algcommentlight{$\texttt{RepLearn}$ as in \cite{mhammedi2023efficient}}
\State For all $h\in[H]$, set $\phib^\rep_h \gets [\phi^\loss_h,\phi^\rep_h]\in \reals^{2d}$.
\State For $h\in[H]$, set $\Psi^{\texttt{span}}_{h}\gets \texttt{Spanner}(h,\Phi,\Psi^\cov_{1:h},  \phib^\rep_h, T_{\spanner})$. \hfill \algcommentlight{\cref{alg:spanner}}
\State Set $T_0 \gets T_\cov+ T_\rep + T_\spanner$. \label{line:T0}
\For{$k=1,\dots,{(T-T_0)}/{\nreg}$}
\State Define $\pihat_h\ind{k}(a~|~x) \propto \exp{\left(-\eta \sum_{s < k}  \Qhat\ind{s}_h(x,a)\right)}$ for $h\in[H]$.    \label{line:pihat}
%\State Define mixture policy $\pi\ind{k}_h(\cdot \mid x) \gets  (1-\nu)\cdot \pihat\ind{k}_h(\cdot \mid x) + \nu \cdot \pi_\unif$, for all $h\in[H]$ and $x\in \cX$
\For{$t=T_0+(k-1)\cdot N_\reg +1,\ \dots,\  T_0+k \cdot N_\reg $} 
\State Define the random variables $\bzeta\indd{t} \sim \mathrm{Ber}(\nu)$, $\bh\indd{t}\sim \unif([H])$, and $\bpi\indd{t} \sim \unif(\Psi^{\texttt{span}}_{\bh\indd{t}})$.
% \State Let $\rho\ind{k}\in \Delta(\Pi)$ be the distribution of $\mathbb{I}\{\bzeta\indd{t} = 0\} \cdot  \pihat\ind{k}+ \mathbb{I}\{\bzeta\indd{t} = 1\} \cdot \bpi\indd{t}\circ_{\bh\indd{t}+1}    \pihat\ind{k}$.  
\State Set $\bhpi^t = \mathbb{I}\{\bzeta\indd{t} = 0\} \cdot  \pihat\ind{k}+ \mathbb{I}\{\bzeta\indd{t} = 1\} \cdot \bpi\indd{t}\circ_{\bh\indd{t}+1}    \pihat\ind{k}$.
\State Execute $\bhpi^t$, and observe trajectory $(\x\indd{t}_1,\a\indd{t}_1,\dots,\x\indd{t}_H,\a\indd{t}_H )$.
\State For $h\in[H]$, observe loss $\bell_h\indd{t}\coloneqq \ell_h(\x_h\indd{t},\a_h\indd{t};\bcH\indd{t-1})$, where $\bcH\indd{t-1} \coloneqq (\x_{1:H}\indd{1:t-1},\a_{1:H}\indd{1:t-1})$. 
\EndFor
%\State Set $\Sigma\ind{k}_h \gets \gamma I + \frac{1}{N_\reg} \sum_{(x_{1:h},a_{1:H},\ell_{1:H})\in \cD\ind{k}} \phi^\rep_s(x_h,a_h)\phi^\rep_h(x_h,a_h)^\top$, for all $h\in[H]$.
%\State Define $B_{h+1}(x,a) \gets \beta \cdot  \sum_{s=h+1}^H (1+1/H)^{s-h}\cdot\|\phi^\rep_s(\x_s,\a_s)\|_{\bm{\Sigma}^{-1}_{s}}$, for all $h\in[H-1]$.
\State For $h\in[H]$ and $\cI\ind{k}= \{ T_0+(k-1)\cdot N_\reg+1,\ \dots,\ T_0+ k \cdot N_\reg\}$, compute $\hat{\theta}\ind{k}_h$ such that 
 \begin{align}
	\hat{\theta}\ind{k}_h &\gets \argmin_{\theta\in \bbB_{2d}(4Hd^{2})}\sum_{\pi\in \Psi^{\texttt{span}}_{h}}  \left|\sum_{t\in \cI\ind{k}} \mathbb{I}\{\bh\indd{t}=h,\bpi\indd{t}=\pi, \bzeta\indd{t}=1\}\cdot \left(\phibar^\rep_h(\x\indd{t}_h,\a\indd{t}_h)^\top \theta -\sum_{s=h}^H \bell\indd{t}_s\right)\right| \label{eq:regproblem} 
 \end{align}
 \State Set $\Qhat\ind{k}_h(x,a) =  \phibar^\rep_h(x,a)^\top \hat{\theta}\ind{k}_h$, for all $(x,a)\in \cX\times \cA$.
 \EndFor
\end{algorithmic}
\end{algorithm}
\vspace{-3.5mm}

\section{Conclusion}
In this paper, we focus on learning low-rank MDPs with unknown transitions and adversarial losses. For the full-information setting, we improve upon previous regret bounds. More importantly, we initiate the study of the challenging bandit feedback setting, developing various algorithms that achieve sublinear regret under different assumptions. However, the optimal $\sqrt{T}$ regret remains out of reach due to the limitations of our two-phase design. An interesting direction for future work is to perform on-the-fly representation learning to adapt to adversarial losses and achieve optimal regret.

% \newpage
% \appendix

\clearpage
\bibliographystyle{apalike}
\bibliography{reference}

\newpage
\appendix
% \appendixpage
% \addappheadtotoc

% {
% \startcontents[section]
% \printcontents[section]{l}{1}{\setcounter{tocdepth}{2}}
% }

\appendixpage

{
\startcontents[section]
\printcontents[section]{l}{1}{\setcounter{tocdepth}{2}}
}

\newpage

\section{Related Work} \label{app: related work}
\paragraph{Learning low-rank MDPs in the stochastic setting. }
In the absent of adversarial losses, several general learning frameworks have been developed for super classes of low-rank MDPs which offer tight sample complexity for either reward-based \citep{jiang2017contextual,  jin2021bellman, du2021bilinear, foster2021statistical, zhong2022gec} or reward-free \citep{chen2022statistical, chen2022unified, xie2022role} settings. However, these algorithms require solving non-convex optimization problems on non-convex version spaces, making them computationally inefficiency. Oracle-efficient algorithms for low-rank MDPs are first obtained by \cite{agarwal2020flambe} using a model-based approach, and the sample complexity bound has been largely improved in subsequent works \citep{ uehara2021representation, zhang2022making, cheng2023improved}. The model-based approach, however, necessitates the function class to accurately model the transition, which is a strong requirement. To relax it, \cite{modi2024model, zhang2022efficient} developed oracle-efficient model-free algorithms, but both of them require additional assumptions on the MDP structure. Recently, \cite{mhammedi2024efficient} proposed a satisfactory model-free algorithm that removes all these assumptions. Our work leverages their techniques to tackle the more challenging adversarial setting.
~\\

% \cite{ren2022free,  zhang2022making} also use linear loss assumption. It seems \cite{ren2022free} is the only work that is efficient and have efficient $\sqrt{T}$-regret (although Bayesian guarantee)?

\paragraph{Learning adversarial MDPs.} Learning adversarial tabular MDPs under bandit feedback and unknown transition has been extensively studied~\citep{rosenberg2019online, jin2020learning, lee2020bias, jin2021best, shani2020optimistic, chen2021finding, luo2021policy, dai2022follow, dann2023best}. This line of work has demonstrated not only $\sqrt{T}$ regret bounds but also several data-dependent bounds.

For adversarial MDPs with a large state space which necessitates the use of function approximation, if the transition is known, \cite{foster2022complexity} shows that adversarial setting is as easy as the stochastic setting even under general function approximation. For full-information loss feedback with unknown transition, $\sqrt{T}$ bound is derived for both linear mixture MDPs \citep{cai2020provably, he2022near} and linear MDPs \citep{sherman2023rate}. For more challenging low-rank MDPs with unknown features, the best result only achieves  $T^{\nicefrac{5}{6}}$ regret~\citep{zhao2024learning}.

For function approximation with bandit feedback and unknown transition, \cite{zhao2022mixture} provides $\sqrt{T}$ bound for linear mixture MDPs, but their regret 
has polynomial dependence on the size of the state space due to the lack of structure on the loss function. For linear MDPs, a series of recent work has made significant progress in improving the regret bound~\citep{luo2021policy, dai2023refined, sherman2023improved, kong2023improved, liu2023towards}. The state-of-the-art result by \cite{liu2023towards} gives an  inefficient algorithm with $\sqrt{T}$ regret and an efficient algorithm with $T^{\nicefrac{3}{4}}$ regret. These regret bounds for linear MDPs do not depend on the state space size because of the linear loss assumption. We show  in \cref{sec:lower bound} that cross-state structure on the losses is necessary for low-rank MDPs with bandit feedback to achieve regret bound that do not scale with the number of states.
\clearpage
\section{Proof of \cref{thm:full-info} (Model-Based, Full Information)}\label{app: model based full info}
\begin{theorem}[Theorem 3 in \cite{cheng2023improved}] 
\label{thm:phat}With probability $1-\delta$, for any policy $\pi$ and layer $h$,  Algorithm 1 in \cite{cheng2023improved} outputs transition $\hatp_{1:H}$ and features $\hphi_h, \,\hpsi_h$ such that $\hatp_h(x'\mid x,a) =\hphi_h(x,a)^\top \hpsi_{h+1}(x')$ and 
\begin{align*}
\E^{\pi}\left[\|\hatp_h\left(\cdot\mid \x_h, \a_h\right) - P^\star_h\left(\cdot\mid \x_h,\a_h\right)\|_1\right] \le \epsilon,
\end{align*}
if the number of collected trajectories is at least $\order\left(\frac{H^3d^2|\cA|(d^2 + |\cA|)}{\epsilon^2}\log^2\left(TdH|\Phi||\Upsilon|\right)\right)$. 
\label{thm:reward-free}
\end{theorem}
Define $\hV^\pi_1(x_1; \ell)$ as the value function of policy $\pi$ under transition $\{\hatp_h\}_{h=1}^H$ and loss $\ell$. We have
\begin{align*}
    \text{\rm Reg}_T(\pi^\star) &= \sum_{t=1}^T V_1^{\pi^t}(x_1; \ell^t) - \sum_{t=1}^T V_1^{\pi^\star}(x_1; \ell^t)
    \\&=  \underbrace{\sum_{t=1}^T \left(V_1^{\pi^t}(x_1; \ell^t) - \hV_1^{\pi^t}(x_1; \ell^t)\right)}_{\textbf{Bias1}} + \underbrace{\sum_{t=1}^T \left(\hV_1^{\pi^\star}(x_1; \ell^t) - V_1^{\pi^\star}(x_1; \ell^t)\right)}_{\textbf{Bias2}}
    \\& \quad +  \underbrace{\sum_{t=1}^T\left(\hV_1^{\pi^t}(x_1; \ell^t) - \hV_1^{\pi^\star}(x_1; \ell^t)\right)}_{\textbf{FTRL}} \, + \,\, \order\left(\frac{H^3d^2|\cA|(d^2 + |\cA|)}{\epsilon^2}\log^2\left(TdH|\Phi||\Upsilon|\right)\right)
\end{align*}

\paragraph{Bounding the bias term.} By \cref{thm:phat} and \cref{lem:simulation}, we have 
\begin{align*}
    \textbf{Bias1} + \textbf{Bias2} \le 2H^2 \epsilon T.
\end{align*}

\paragraph{Bounding the FTRL term.} Since $\eta = \frac{1}{H \sqrt{T}}$, we have
\begin{align*}
    \textbf{FTRL} &\le \sum_{t=1}^T\sum_{h=1}^H \E^{\hatp, \pi^\star}\left[\left\langle \Qhat_h^t(\x_h, \cdot), \pi_h^t(\cdot \mid \x_h) - \pi_h^\star(\cdot \mid \x_h)\right\rangle\right] \tag{\cref{lem:PDL}}
    \\&\le \frac{H \log|\cA|}{\eta} + \eta \sum_{h=1}^H \sum_{t=1}^T  \E^{\hatp, \pistar \circ_h \pi^t} \left[\left(\Qhat_{h}^t(\x_h,\a_h)\right)^2 \right],\nn  \tag{\cref{lem:EXP bound}}
    \\&\le \frac{H\log|\cA|}{\eta} + 2H^3\eta T \tag{$\Qhat_{h}^t(\x_h,\a_h) \le H$}
    \\&= \order\left(H^2\sqrt{T}\log|\cA|\right).
\end{align*}
Thus, by setting $\epsilon = (Hd^2|\cA|(d^2+|\cA|))^{\frac{1}{3}}T^{-\frac{1}{3}}$, we have
\begin{align*}
     \text{\rm Reg}_T 
     &\le  \order\left(\frac{H^3d^2|\cA|(d^2 + |\cA|)}{\epsilon^2}\log^2\left(TdH|\Phi||\Upsilon|\right) + 2H^2\epsilon T + H^2\sqrt{T}\log|\mathcal{A}|\right) \\
     &\le  \order\left(H^3 \left(d^2 + |\cA|\right)T^{\frac{2}{3}} \log\left(|\cA| + dH|\Phi||\Upsilon| T\right)\right).
\end{align*}
\clearpage
\section{Proof of \cref{thm:model-based bound} (Model-Based, Bandit Feedback)}
%\begin{theorem}[Theorem 3 in \cite{cheng2023improved}] 
%\label{thm:phat}With probability $1-\delta$, for any policy $\pi$ and layer $h$,  Algorithm 1 in \cite{cheng2023improved} outputs transition $\hatp_{1:H}$ and features $\hphi_h, \,\hpsi_h$ such that $\hatp_h(x'\mid x,a) =\hphi_h(x,a)^\top \hpsi_{h+1}(x')$ and 
%\begin{align*}
%\E^{\pi}\left[\|\hatp_h\left(\cdot\mid \x_h, \a_h\right) - P^\star_h\left(\cdot\mid \x_h,\a_h\right)\|_1\right] \le \epsilon,
%\end{align*}
%if the number of collected trajectories is at least $\order\left(\frac{H^3d^2|\cA|(d^2 + |\cA|)}{\epsilon^2}\log^2\left(TdH|\Phi||\Upsilon|\right)\right)$. 
%\label{thm:reward-free}
%\end{theorem}

\begin{lemma}\label{lem: P to d} With $\hatp_{1:H}$ as in \cref{thm:phat}, we have for any $h \in [H]$ and any policy $\pi$,
\begin{align*}
    \sum_{x \in \cX}\left|\hd_{h}^\pi(x) - d_{h}^\pi(x) \right| \le \sum_{i=1}^{h-1} \E^{\pi}\left[\|\hatp_i(\cdot\mid \x_i,\a_i) - P_i(\cdot\mid \x_i,\a_i)\|_1\right] \le (h-1)\cdot \epsilon.
\end{align*}
where $\hd_h^\pi \coloneqq d^{\hatp,\pi}_h$.
\end{lemma}

\begin{proof}
We prove the claim by induction. When $h=1$, given that the $\|d_{1}^\pi - \hd_{1}^\pi\|_1 = 0$. Assume $$\sum_{x \in \cX}\left|\hd_{h}^\pi(x) - d_{h}^\pi(x) \right| \le \sum_{i=1}^{h-1} \E^{\pi}\left[\|\hatp_i(\cdot\mid \x_i,\a_i) - P_i(\cdot\mid \x_i,\a_i)\|_1\right].$$ 
We have
\begin{align*}
&\sum_{x \in \cX_{h+1}}\left|\hd_{h+1}^\pi(x) - d_{h+1}^\pi(x) \right| 
\\&= \sum_{x \in \cX} \sum_{a \in \cA} \sum_{ x' \in \cX_{h+1}}  \left|\hd_{h}^\pi(x) \pi_h(a\mid x)\cdot \hatp_h(x'|x,a) - d_{h}^\pi(x)\pi_h(a\mid x)\cdot P^\star_h(x'|x,a)\right|,
\\&\le  \sum_{x \in \cX} \sum_{a \in \cA} \sum_{ x' \in \cX_{h+1}}\left|\hd_{h}^\pi(x)\pi_h(a\mid x)\cdot \hatp_h(x'|x,a)  - d_{h}^\pi(x)\pi_h(a\mid x)\cdot \hatp_h(x'|x,a)\right|
\\&\quad +  \sum_{x \in \cX} \sum_{a \in \cA} \sum_{ x' \in \cX_{h+1}} \left|d_{h}^\pi(x)\pi_h(a\mid x) \cdot \hatp_h(x'|x,a)- d_{h}^\pi(x)\pi_h(a\mid x)\cdot P^\star_h(x'|x,a)\right|,
\\&\le \sum_{x \in \cX} \sum_{a \in \cA} \sum_{ x' \in \cX_{h+1}} \hatp_h(x'|x,a)\pi_h(a\mid x)\cdot |\hd_{h}^\pi(x) - d_{h}^\pi(x)|
\\&\quad + \sum_{x \in \cX} \sum_{a \in \cA} \sum_{ x' \in \cX_{h+1}} d_{h}^\pi(x) \pi_h(a\mid x) \cdot |\hatp_h(x'|x,a) - P_h(x'|x,a)|,
\\&\le \sum_{x \in \cX}|\hd_{h}^\pi(x) - d_{h}^\pi(x)| + \E^{\pi}\left[\|\hatp_h(\cdot\mid \x_h,\a_h) - P_h(\cdot\mid \x_h,\a_h)\|_1\right],
\\&\le \sum_{i=1}^{h} \E^{\pi}\left[\|\hatp_i(\cdot\mid \x_i,\a_i) - P_i(\cdot\mid \x_i,\a_i)\|_1\right],
\end{align*}
where the last step follows by the induction hypothesis.
The second inequality of \cref{lem: P to d} directly comes from \cref{thm:reward-free}.
\end{proof}

Our candidate policy space $\Pi'$ has a $\beta$ mixture of the random policy. For any deterministic policy $\pi^\star_0$, define policy $\pi^\star$ such that for any state $x \in \cX$, we have $\pi^\star(\cdot\mid x) = (1-\beta) \pi^\star_0(\cdot\mid x) + \frac{\beta}{|\cA|}$. We have $\pi^\star \in \Pi'$. Define $\hV^\pi_1(x_1; \ell)$ as the value function of policy $\pi$ under transition $\{\hatp_h\}_{h=1}^H$ and loss $\ell$.

For any policy $\pi^\star_0$, we have
\begin{align}
&\text{\rm Reg}_T(\pi^\star_0) 
\\&= \E\left [\sum_{t=1}^T\sum_{\pi} \tilp^t(\pi) V^{\pi}_1(\x_1; \ell^t)-  V^{\pi_0^\star}_1(\x_1; \ell^t) \right] + \order\left(\frac{H^3d^2|\cA|(d^2 + |\cA|)}{\epsilon^2}\log^2\left(TdH|\Phi||\Upsilon|\right)\right),\nn  
\\&= \E\left [\sum_{t=1}^T\sum_{\pi} p^t(\pi) V^{\pi}_1(\x_1; \ell^t)-  V^{\pi^\star}_1(\x_1; \ell^t) \right] + \order\left(\frac{H^3d^2|\cA|(d^2 + |\cA|)}{\epsilon^2}\log^2\left(TdH|\Phi||\Upsilon|\right)\right)
\\&\qquad + \underbrace{\E\left [\sum_{t=1}^T\sum_{\pi} \left(\tilp^t(\pi) - p^t(\pi)\right) V^{\pi}_1(\x_1; \ell^t) \right]}_{\textbf{Error1}} + \underbrace{\E\left [\sum_{t=1}^T V^{\pi^\star}_1(\x_1; \ell^t) -  V^{\pi_0^\star}_1(\x_1; \ell^t) \right]}_{\textbf{Error2}}, \nn 
\\&= \underbrace{\E\left [\sum_{t=1}^T \sum_{\pi} p^t(\pi) \left(V^{\pi}_1(\x_1; \ell^t)-  \hV^{\pi}_1(\x_1; \ell^t)\right) \right]}_{\textbf{Bias1}} + \underbrace{\E\left [\sum_{t=1}^T \hV^{\pi^\star}_1(\x_1; \ell^t)-  V^{\pi^\star}_1(\x_1; \ell^t) \right]}_{\textbf{Bias2}}  \nn 
\\&\qquad + \underbrace{\E\left [\sum_{t=1}^T \sum_{\pi} p^t(\pi) \hV^{\pi}_1(\x_1; \ell^t)-  \hV^{\pi^\star}_1(\x_1; \ell^t) \right]}_{\textbf{EXP}} \,+\,\, \textbf{Error1} + \textbf{Error2} \nn \\
& \qquad + \order\left(\frac{H^3d^2|\cA|(d^2 + |\cA|)}{\epsilon^2}\log^2\left(TdH|\Phi||\Upsilon|\right)\right). \label{eq:thisone}
\end{align}

Recall that $J = \text{John}\{\hatphi_h(\pi)\}_{\pi \in \Pi', h\in[H]} \in \Delta(\Pi' \times H)$, and we have $\tilp^t(\pi) = (1-\gamma)p^t(\pi) + \gamma \sum_{h=1}^H J(\pi,h)$ where $p^t(\pi)$ is defined in \cref{line:bandit-exp} of \cref{alg:low-rank bandit}.

\begin{lemma} \label{lem: EXP Error}
We have 
    \begin{align*}
        \textbf{\emph{Error1}} + \textbf{\emph{Error2}} \le H\gamma T +  2H^2\beta T.
    \end{align*}
\end{lemma}

\begin{proof}
    \begin{align*}
        \textbf{Error1} &= \gamma \E\left [\sum_{t=1}^T\sum_{\pi} \left(\sum_{h=1}^H J(\pi,h) - p^t(\pi)\right)\cdot  V^{\pi}_1(\x_1; \ell^t) \right]  \le H\gamma T.
        \\\textbf{Error2} &= \sum_{t=1}^T\sum_{h=1}^H \E^{\pistar} \left[\sum_{a \in \cA} \left|\pi_0^\star(a\mid \x_h) - \pi^\star(a\mid \x_h)\right| \cdot Q^{\pi_0^\star}_h(\x_h,a; \ell^t)\right] \le 2H^2\beta T,
    \end{align*}
    where the last step follows by \cref{lem:PDL}.
\end{proof}

\begin{lemma} \label{lem: EXP bias}
We have
    \begin{align*}
        \textbf{\emph{Bias1}} + \textbf{\emph{Bias2}} \le H^2T \epsilon.
    \end{align*}
\end{lemma}

\begin{proof}
    This is a direct result combing \cref{thm:reward-free} and \cref{lem:simulation}.
\end{proof}

For $(x_{1:H},a_{1:H}) \in \cX^H \times \cA^H$ and $\pi\in\Pi'$ where $\Pi'$ defined in \cref{line:mix pi} in \cref{alg:low-rank bandit} is the mix of
a given policy class $\Pi$ and a uniform policy, and is also the policy class we play with. Recall that the loss estimator

% \zm{What is $\pi'$? Make sure we remind the reader what it is throughout this section}
\begin{align}
\hell^t(\pi;\pi^t, x_{1:H},a_{1:H}) & \coloneqq  \frac{\pi_1(a_1\mid x_1)}{\pi^t_1(a_1\mid x_1)}\ell_1^t(x_1,a_1)\nn \\
& \quad + \sum_{h=2}^H \hphi_{h-1}(\pi)^\top\left(\Sigma_{h-1}^t\right)^{-1} \hphi_{h-1}(\pi^t) \frac{\pi_h(a_h\mid x_h)}{\pi^t_h(a_h\mid x_h)}\ell_h^t(x_h,a_h). \label{eq:lhat}
\end{align}
defined in \cref{line:bandit estimator} of \cref{alg:low-rank bandit}.
% \zm{link this to the definition in the algo} With this, we have the following result.
%We define $M_t^\star(\pi)$ as the distribution of trajectory $\tau^t = (s_1, a_1, \ell_1^t(s_1, a_1), \cdots, s_H, a_H, \ell_H^t(s_H, a_H))$ where for any $h \in [H]$, $s_h \sim d_h^{\pi}, a_h \sim \pi_h(\cdot|s_h)$. We use $\hell^t(\pi; \tau^t)$ to explicitly indicate the influence of $\tau^t$ to $\hell^t(\pi)$ defined in \cref{line:bandit estimator} of \cref{alg:low-rank bandit}.
\begin{lemma} \label{lem:value transfer}
For any episode $t \in [T]$, for any policy $\pi$ we have
    \begin{align*}
        \hV_1^\pi(x_1; \ell^t) &\le \E_{\bpi^t \sim \tilp^t} \E^{\bpi^t}\left[\hell^t(\pi; \bpi^t, \x_{1:H},\a_{1:H})\right] + \sqrt{d}H\epsilon \sum_{h=2}^{H} \left\|\hphi_{h-1}(\pi) \right\|_{\left(\Sigma_{h-1}^t\right)^{-1}},\\
         \hV_1^\pi(x_1; \ell^t) &\ge \E_{\bpi^t \sim \tilp^t} \E^{\bpi^t}\left[\hell^t(\pi; \bpi^t, \x_{1:H},\a_{1:H})\right]  - \sqrt{d}H\epsilon \sum_{h=2}^{H} \left\|\hphi_{h-1}(\pi) \right\|_{\left(\Sigma_{h-1}^t\right)^{-1}}.   
    \end{align*}
\end{lemma}

\begin{proof}
First, from the definition in \cref{line:est feature} of \cref{alg:low-rank bandit}, for any $x \in \cX$ with $h \ge 2$, we have
\begin{align*}
    \hd_h^{\pi}(x) = \sum_{x',a'} \hd_{h-1}^{\pi}(x',a') \cdot \hatp_{h-1}(x\mid x', a') =  \hphi_{h-1}(\pi)^\top \hpsi_{h}(x)
\end{align*}
where $\hphi_{h-1}(\pi) = \sum_{x',a'} \hd_{h-1}^{\pi}(x',a')\hphi_{h-1}(x',a')$.

We now prove the first inequality:
    \begin{align}
         &\hV_1^\pi(x_1; \ell^t) \nn
         \\&= \sum_{h=1}^H \sum_{x_h \in \cX} \sum_{a_h \in \cA} \hd_h^\pi(x_h) \pi_h(a_h\mid x_h) \cdot \ell_h^t(x_h,a_h), \nn
         \\&=  \underbrace{\sum_{a_1\in \cA} \pi_1(a_1\mid x_1)\cdot  \ell_1^t(x_1,a_1)}_{\textbf{First}}+ \underbrace{\sum_{h=2}^H \sum_{x_h \in \cX} \sum_{a_h \in \cA} \hphi_{h-1}(\pi)^\top \hpsi_{h}(x_h)\pi_h(a_h\mid x_h) \cdot \ell_h^t(x_h,a_h)}_{\textbf{Remain}} \label{eq:emain}
    \end{align}
    Through importance sampling, we have
    \begin{align*}
           \textbf{First}  =  \E_{\bpi^t \sim \tilp^t}\E^{\bpi^t}\left[\frac{\pi_1(\a_1\mid \x_1)}{\pi_1^t(\a_1 \mid \x_1)}\cdot  \ell_1^t(\x_1,\a_1)\right].
    \end{align*}
    % \begin{align*}
    %     \textbf{First} =   \E_{\bpi^t \sim \tilp^t}\E^{\bpi^t}\left[\frac{\pi_1(\a_1\mid \x_1)}{\sum_{\pi' \in \Pi} \tilp^t(\pi') \pi_1'(\a_1\mid \x_1)}\cdot  \ell_1^t(\x_1,\a_1)\right].
    % \end{align*}
    We now bound the remaining term in \eqref{eq:emain}
    Since $\Sigma_{h-1}^t = \E_{\bpi^t \sim \tilp^t}\left[\hphi_{h-1}(\bpi^t)\hphi_{h-1}(\bpi^t)^\top\right]$, we have
    \begin{align}
           &\textbf{Remain}  \nonumber
           \\&= \sum_{h=2}^H \sum_{x_h \in \cX} \sum_{a_h \in \cA} \hphi_{h-1}(\pi)^\top \left(\Sigma_{h-1}^t\right)^{-1} \E_{\bpi^t \sim \tilp^t}\left[\hphi_{h-1}(\bpi^t)\hphi_{h-1}(\bpi^t)^\top\right]\hpsi_{h}(x_h)\pi_h(a_h\mid x_h) \ell_h^t(x_h,a_h), \nonumber
           \\& =  \E_{\bpi^t \sim \tilp^t}\left[\sum_{h=2}^H \sum_{x_h \in \cX} \sum_{a_h \in \cA} \hphi_{h-1}(\pi)^\top \left(\Sigma_{h-1}^t\right)^{-1} \hphi_{h-1}(\bpi^t)\underbrace{\hphi_{h-1}(\bpi^t)^\top\hpsi_{h}(x_h)}_{\hd_{h}^{\bpi^t}(x_h)}\pi_h(a_h\mid x_h) \ell_h^t(x_h,a_h)\right], \nonumber
           \\&= \E_{\bpi^t \sim \tilp^t}\left[\sum_{h=2}^H \sum_{x_h \in \cX} \sum_{a_h \in \cA} \hphi_{h-1}(\pi)^\top \left(\Sigma_{h-1}^t\right)^{-1} \hphi_{h-1}(\bpi^t) \hd_{h}^{\bpi^t}(x_h)\pi_h(a_h\mid x_h) \ell_h^t(x_h,a_h)\right], \nonumber
           \\&=  \E_{\bpi^t \sim \tilp^t}\left[\sum_{h=2}^H \sum_{x_h \in \cX} \sum_{a_h \in \cA} \hphi_{h-1}(\pi)^\top \left(\Sigma_{h-1}^t\right)^{-1} \hphi_{h-1}(\bpi^t) d_{h}^{\bpi^t}(x_h)\pi_h(a_h\mid x_h) \ell_h^t(x_h,a_h)\right] \nonumber
           \\&\qquad +  \E_{\bpi^t \sim \tilp^t}\left[\sum_{h=2}^H \sum_{x_h \in \cX} \sum_{a_h \in \cA} \hphi_{h-1}(\pi)^\top \left(\Sigma_{h-1}^t\right)^{-1} \hphi_{h-1}(\bpi^t) \left(\hd_{h}^{\bpi^t}(x_h) - d_{h}^{\bpi^t}(x_h)\right)\pi_h(a_h\mid x_h) \ell_h^t(x_h,a_h)\right], \label{eqn:CS}
           \\&\le  \E_{\bpi^t \sim \tilp^t}\left[\sum_{h=2}^H \sum_{x_h \in \cX} \sum_{a_h \in \cA} \hphi_{h-1}(\pi)^\top \left(\Sigma_{h-1}^t\right)^{-1} \hphi_{h-1}(\bpi^t) d_{h}^{\bpi^t}(x_h,a_h)\frac{\pi_h(a_h\mid x_h)}{\bpi_h^t(a_h\mid x_h)}\ell_h^t(x_h,a_h)\right]  \nonumber
           \\&\qquad +  \sum_{h=2}^H \left\|\hphi_{h-1}(\pi) \right\|_{\left(\Sigma_{h-1}^t\right)^{-1}}  \E_{\bpi^t \sim \tilp^t}\left[\left\|\hphi_{h-1}(\bpi^t)\right\|_{\left(\Sigma_{h-1}^t\right)^{-1}}\right] \sum_{x_h \in \cX} \left|\hd_{h}^{\bpi^t}(x_h) - d_{h}^{\bpi^t}(x_h)\right|, \label{eq:penul}
           \\&\le  \E_{\bpi^t \sim \tilp^t}\E^{\bpi^t}\left[\sum_{h=2}^H \hphi_{h-1}(\pi)^\top \left(\Sigma_{h-1}^t\right)^{-1} \hphi_{h-1}(\bpi^t) \frac{\pi_h(\a_h\mid \x_h)}{\bpi_h^t(\a_h\mid \x_h)}\ell_h^t(\x_h,\a_h)\right]   \nonumber
           \\&\qquad +  \sqrt{d}H\epsilon \sum_{h=2}^H \left\|\hphi_{h-1}(\pi) \right\|_{\left(\Sigma_{h-1}^t\right)^{-1}},
    \end{align}
    where \eqref{eq:penul} follows by Cauchy-Schwarz, and the last inequality uses \cref{lem: P to d}.
    Adding up \textbf{First} and \textbf{Remain}, and using the defintion of $\hat\ell$ in \eqref{eq:lhat} implies the first inequality of the lemma.
 
    The second inequality follows the same procedure except for applying Cauchy-Schwarz inequality in the opposite direction in \cref{eqn:CS}.
\end{proof}

\begin{lemma}\label{lem:EXP EXP}
If $\eta \le \left(\frac{2|\cA| H d}{\beta \gamma} + d H^2\frac{\epsilon}{\sqrt{\gamma}}\right)^{-1}$, then
\begin{align*}
    \textbf{\emph{EXP}} \le  \frac{\log|\Pi|}{\eta} + \frac{6d\eta H^2|\cA|T}{\beta} +  4\eta d^2 H^4 \epsilon^2 T + 4dH^2\epsilon T,
\end{align*}
where $\textbf{\emph{EXP}}$ is as in \eqref{eq:thisone}.
\end{lemma}

\begin{proof}
Recall in \cref{line:bandit estimator} of \cref{alg:low-rank bandit}, $b^t(\pi) \coloneqq \sqrt{d}H\epsilon\sum_{h=2}^{H} \left\|\hphi_{h-1}(\pi) \right\|_{\left(\Sigma_{h-1}^t\right)^{-1}}$, for all $\pi\in\Pi'$. By \cref{lem:value transfer}, we have
\begin{align}
\textbf{EXP} &= \E\left [\sum_{t=1}^T \sum_{\pi} p^t(\pi)\hV^{\pi}_1(\x_1; \ell^t)-  \hV^{\pi^\star}_1(\x_1; \ell^t) \right]\nn 
\\&\le \sum_{t=1}^T \sum_{\pi} p^t(\pi)\E_{\bpi^t \sim \tilp^t} \E^{\bpi^t}\left[\hell^t(\pi;\bpi^t, \x_{1:H},\a_{1:H})\right] - \sum_{t=1}^T \E_{\bpi^t \sim \tilp^t} \E^{\bpi^t}\left[\hell^t(\pi^\star; \bpi^t,\x_{1:H},\a_{1:H})\right]\nn 
\\&\qquad + \sqrt{d}H\epsilon \sum_{t=1}^T \sum_{\pi} p^t(\pi)\sum_{h=2}^{H} \left\|\hphi_{h-1}(\pi) \right\|_{\left(\Sigma_{h-1}^t\right)^{-1}} + \sqrt{d}H\epsilon \sum_{t=1}^T \sum_{h=2}^{H} \left\|\hphi_{h-1}(\pi^\star) \right\|_{\left(\Sigma_{h-1}^t\right)^{-1}}, \nn
\\&= \sum_{t=1}^T \sum_{\pi} p^t(\pi)\cdot \left(\E_{\bpi^t \sim \tilp^t} \E^{\bpi^t}\left[\hell^t(\pi;\bpi^t, \x_{1:H},\a_{1:H})\right]- b^t(\pi) \right) \nn \\
& \qquad - \sum_{t=1}^T \left(\E_{\bpi^t \sim \tilp^t} \E^{\bpi^t}\left[\hell^t(\pi^\star; \bpi^t,\x_{1:H},\a_{1:H})\right]-b^t(\pistar)\right) \nn 
\\& \qquad + \sum_{t=1}^T \sum_{\pi} p^t(\pi) b^t(\pi)- \sum_{t=1}^T  b^t(\pi^\star) +   \sqrt{d}H\epsilon \sum_{t=1}^T \sum_{\pi} p^t(\pi)\sum_{h=2}^{H} \left\|\hphi_{h-1}(\pi) \right\|_{\left(\Sigma_{h-1}^t\right)^{-1}} \nn \\
& \qquad + \sqrt{d}H\epsilon \sum_{t=1}^T \sum_{h=2}^{H} \left\|\hphi_{h-1}(\pi^\star) \right\|_{\left(\Sigma_{h-1}^t\right)^{-1}} \nn 
\\&= \sum_{t=1}^T \sum_{\pi} p^t(\pi)\cdot \left(\E_{\bpi^t \sim \tilp^t} \E^{\bpi^t}\left[\hell^t(\pi;\bpi^t, \x_{1:H},\a_{1:H})\right]- b^t(\pi) \right) \nn \\
& \qquad - \sum_{t=1}^T \left(\E_{\bpi^t \sim \tilp^t} \E^{\bpi^t}\left[\hell^t(\pi^\star; \bpi^t,\x_{1:H},\a_{1:H})\right]-b^t(\pistar)\right) \nn
\\& \qquad + 4\sqrt{d}H\epsilon \sum_{t=1}^T \sum_{\pi} \tilp^t(\pi)\sum_{h=2}^{H} \left\|\hphi_{h-1}(\pi) \right\|_{\left(\Sigma_{h-1}^t\right)^{-1}}, \quad \text{(since $p^t(\pi) \le 2 \tilp^t(\pi)$)}  \label{eq:EXP2FTRL}
\\&\le \textbf{FTRL} + 4dH^2\epsilon T,\nn 
\end{align}
where 
\begin{align}
 \textbf{FTRL}&\coloneqq \sum_{t=1}^T \sum_{\pi} p^t(\pi)\cdot \left(\E_{\bpi^t \sim \tilp^t} \E^{\bpi^t}\left[\hell^t(\pi;\bpi^t, \x_{1:H},\a_{1:H})\right]- b^t(\pi) \right) \nn \\
& \qquad - \sum_{t=1}^T \left(\E_{\bpi^t \sim \tilp^t} \E^{\bpi^t}\left[\hell^t(\pi^\star; \bpi^t,\x_{1:H},\a_{1:H})\right]-b^t(\pistar)\right).
\end{align}
We now bound the FTRL term. Since $\tilp^t(\pi) = (1-\gamma)p^t(\pi) + \gamma \sum_{h=1}^H J(\pi,h)$, define $\Sigma_{J} = \sum_{\pi \in \Pi}\sum_{h=1}^H J(\pi,h) \hphi_{h}(\pi) \hphi_{h}(\pi)^\top$. Note that for any $t \in [T]$ and $h \in [H]$, we have $\Sigma_h^t \succeq \gamma \Sigma_{J}$.

By the triangle inequality, we have for any $\pi,\pi^t$ and $(x_{1:H},a_{1:H})$, 
\begin{align*}
    \left|\hell^t(\pi; \pi^t, x_{1:H},a_{1:H})\right| &\le  \left|\frac{\pi_1(\a_1\mid \x_1)}{\pi_1^t(\a_1 \mid \x_1)}\cdot  \ell_1^t(\x_1,\a_1)\right| \nn \\
    & \quad + \sum_{h=2}^H \left|\hphi_{h-1}(\pi)^\top \left(\Sigma_{h-1}^t\right)^{-1} \hphi_{h-1}(\pi^t) \frac{\pi_h(a_h\mid x_h)}{\pi_h^t(a_h\mid x_h)}\ell_h^t(x_h,a_h)\right|,
    \\&\le  \frac{|\cA|}{\beta} + \frac{|\cA|}{\beta}\sum_{h=2}^H \left|\hphi_{h-1}(\pi)^\top \left(\Sigma_{h-1}^t\right)^{-1} \hphi_{h-1}(\pi^t)\right|, \tag{$\beta$-mixure of uniform policy}
    \\&\le \frac{|\cA|}{\beta} + \frac{|\cA|}{\beta \gamma}\sum_{h=2}^H \left\|\hphi_{h-1}(\pi)\right\|_{\Sigma_{J}^{-1}} \left\|\hphi_{h-1}(\pi^t)\right\|_{\Sigma_{J}^{-1}},
    \\&\le \frac{|\cA|}{\beta} + \frac{|\cA| H d}{\beta \gamma},\nn \\
    & \le \frac{2|\cA| H d}{\beta \gamma}, 
    \intertext{and}
    \\\left|b^t(\pi)\right|& = \left|\sqrt{d}H\epsilon\sum_{h=2}^{H} \left\|\hphi_{h-1}(\pi) \right\|_{\left(\Sigma_{h-1}^t\right)^{-1}}\right|  \le d H^2\frac{\epsilon}{\sqrt{\gamma}}.
\end{align*}
To ensure $\eta \left|\hell^t(\pi; \pi^t, x_{1:H},a_{1:H}) - b^t(\pi)\right| \le 1$, it suffices to set $\eta \le \left(\frac{2|\cA| H d}{\beta \gamma} + d H^2\frac{\epsilon}{\sqrt{\gamma}}\right)^{-1}$. Under this constraint, from \cref{lem:EXP bound}, we have
\begin{align*}
    \textbf{FTRL} &\le \frac{\log|\Pi|}{\eta} + \underbrace{2\eta \sum_{t=1}^T\sum_{\pi \in \Pi'} p^t(\pi)\cdot \E_{\bpi^t \sim \tilp^t}\E^{\bpi^t}\left[\hell^t(\pi;\bpi^t,\x_{1:H},\a_{1:H})^2\right]}_{\textbf{Stability-1}} \nn \\ 
    & \qquad + \underbrace{\E\left[2\eta \sum_{t=1}^T\sum_{\pi \in \Pi'} p^t(\pi) b^t(\pi)^2\right]}_{\textbf{Stability-2}}
\end{align*}

For any $t \in [T]$, we have
\begin{align}
&\sum_{\pi \in \Pi'} p^t(\pi)\cdot \E_{\bpi^t \sim \tilp^t}\E^{\bpi^t}\left[\hell^t(\pi;\bpi^t,\x_{1:H},\a_{1:H})^2\right] \nn 
\\&\le H \E_{\bpi^t \sim \tilp^t}\E^{\bpi^t}\left[\frac{\pi_1(\a_1\mid \x_1)^2}{\bpi_1^t(\a_1\mid \x_1)^2} \ell_1^t(\x_1,\a_1)^2 \right] \nn 
\\&\quad + H\sum_{h=2}^H \E_{\bpi^t \sim \tilp^t}\E^{\bpi^t}\left[\sum_{\pi \in \Pi'} p^t(\pi) \hphi_{h-1}(\pi)^\top \left(\Sigma_{h-1}^t\right)^{-1} \hphi_{h-1}(\bpi^t) \hphi_{h-1}(\bpi^t)^\top \left(\Sigma_{h-1}^t\right)^{-1} \hphi_{h-1}(\pi) \frac{\pi_h(\a_h\mid \x_h)^2}{\bpi_h^t(\a_h\mid \x_h)^2}\ell_h^t(\x_h,\a_h)^2\right], \nn 
\\&\le \frac{H |\cA|}{\beta} +  2H\sum_{h=2}^H \E_{\bpi^t  \sim \tilp^t}\E^{\bpi^t}\left[\tr\left(\hphi_{h-1}(\bpi^t)
 \hphi_{h-1}(\bpi^t)^\top  \left(\Sigma_{h-1}^t\right)^{-1}\right) \sum_{a \in \cA}\frac{\pi_h(a\mid \x_h)^2}{\bpi_h^t(a\mid \x_h)}\ell_h^t(\x_h,a)^2\right], \label{eq:just}
 \\&\le \frac{H |\cA|}{\beta} +  \frac{2H|\cA|}{\beta} \sum_{h=2}^H \E_{\bpi^t  \sim \tilp^t}\left[\tr\left(\hphi_{h-1}(\bpi^t)
 \hphi_{h-1}(\bpi^t)^\top  \left(\Sigma_{h-1}^t\right)^{-1}\right)\right],\nn 
 \\&\le \frac{3dH^2|\cA|}{\beta}.\nn 
\end{align}
where \eqref{eq:just} follows by the fact that $p^t(\pi) \le \frac{1}{1-\gamma}\tilp^t(\pi)$ and $\frac{1}{1-\gamma} \le 2$.
Thus,
\begin{align*}
    \textbf{Stability-1} &= 2\eta \sum_{\pi \in \Pi'} p^t(\pi)\cdot \E_{\bpi^t \sim \tilp^t}\E^{\bpi^t}\left[\hell^t(\pi;\bpi^t,\x_{1:H},\a_{1:H})^2\right] \le \frac{6d\eta H^2|\cA|T}{\beta}.
\end{align*}
Moreover, 
\begin{align*}
 \textbf{Stability-2} &= 2\eta \E\left[\sum_{t=1}^T\sum_{\pi \in \Pi'} p^t(\pi) b^t(\pi)^2\right],
 \\&=  2\eta dH^3 \epsilon^2 \E\left[\sum_{t=1}^T\sum_{h=2}^{H}\sum_{\pi \in \Pi'} p^t(\pi) \left\|\hphi_{h-1}(\pi) \right\|_{\left(\Sigma_{h-1}^t\right)^{-1}}^2\right],
 \\&\le  4\eta dH^3 \epsilon^2 \E\left[\sum_{t=1}^T\sum_{h=1}^{H-1}\sum_{\pi \in \Pi'} \tilp^t(\pi) \left\|\hphi_{h-1}(\pi) \right\|_{\left(\Sigma_{h-1}^t\right)^{-1}}^2\right],
 \\&\le  4\eta d^2 H^4 \epsilon^2 T.
\end{align*}
\end{proof}

Combing \cref{lem: EXP Error}, \cref{lem: EXP bias} and \cref{lem:EXP EXP}, if $\eta \le \left(\frac{2|\cA| H d}{\beta \gamma} + d H^2\frac{\epsilon}{\sqrt{\gamma}}\right)^{-1}$, we have
\begin{align*}
    \text{\rm Reg}_T(\pi^\star_0) & \le H\gamma T + 2H^2\beta T + H^2T\epsilon + \frac{\log|\Pi|}{\eta} + \frac{6d\eta H^2|\cA|T}{\beta} \nn \\
    & \quad +  4\eta d^2 H^4 \epsilon^2 T + 4dH^2\epsilon T + \order\left(\frac{H^3d^2|\cA|(d^2 + |\cA|)}{\epsilon^2}\log^2\left(TdH|\Phi||\Upsilon|\right)\right).
\end{align*}
By setting $\epsilon = T^{-\frac{1}{3}}$, $\gamma = T^{-\frac{1}{3}}$, $\beta = T^{-\frac{1}{3}}$, $\eta = \frac{1}{4Hd|\cA|} T^{-\frac{2}{3}}$, we have for any $\pi^\star_0 \in \Pi$,
\begin{align*}
     \text{\rm Reg}_T(\pi^\star_0)  \le \order\left(d^2 H^3 |\cA|(d^2 + |\cA|)T^{\frac{2}{3}}\log|\Pi|\log^2\left(TdH|\Phi||\Upsilon|\right)\right).
\end{align*}
\clearpage

%Our goal is to bound 
%\begin{align*}
%    \sum_{i=1}^n  (\phi_\star(x^i, a^i)^\top \theta_\star - \hat{\phi}(x^i, a^i)^\top \hat{\theta})^2 
%\end{align*}
%where 
%\begin{align*}
%    (\hat{\phi}, \hat{\theta}) = \argmin_{(\phi, \theta)\in\Phi\times\mathbb{B}(H\sqrt{d})}\sum_{i=1}^n (\phi(x^i,a^i)^\top \theta - z^i)^2, 
%\end{align*}
%and 
%\begin{align*}
%    z^i = \phi_\star(x^i, a^i)^\top \theta_\star^i + \epsilon^i, \qquad \E[\epsilon^i\mid x^i, a^i] = 0, \qquad \theta_\star = \frac{1}{n}\sum_{i=1}^n \theta_\star^i.   
%\end{align*}

\clearpage
\section{More Details of Inefficient Model-Free Algorithm in \cref{sec:model-free inefficient}} 
\label{app:model-free ineff}

\subsection{Algorithm Description}
In this section, we give a more detailed introduction of the algorithm mentioned in \cref{sec:model-free inefficient}. This algorithm is model-free and achieves $T^{\frac{2}{3}}$ regret, but it is computationally inefficient. We consider the low-rank MDPs with linear losses that satisfies \cref{assm:normalizing} and \cref{assm:linearlossweak}. 

Let $\mathcal{C}\left(S, \epsilon' \right)$ be $\epsilon'$-net of space $S$. We define necessary policy and function classes in \cref{def:pi and F}.

\begin{definition}
We define linear policy class and its discretization as 
\begin{align*}
    \Pi_{\text{\rm lin}} &= \left\{\pi :\cX\rightarrow \Delta(\cA) \ \Big\mid \   \pi_h(a\mid x ) = \mathbb{I}\big\{a=\argmin_{a \in \cA} \phi_h(x, a)^\top \theta_h\big\}, \ h\in[H],\  \theta_h \in \mathbb{B}_d \left(\sqrt{d}H T\right), \ \phi \in \Phi\right\}.\\
     \Pi_{\text{\rm lin}}^{\text{\rm cov}}(\epsilon') &= \left\{\pi :\cX\rightarrow \Delta(\cA) \ \Big\mid \   \pi_h(a\mid x ) = \mathbb{I}\big\{a=\argmin_{a \in \cA} \phi_h(x, a)^\top \theta_h\big\}, \ h\in[H],\ \theta_h \in \mathcal{C}\left(\mathbb{B}_d \left(\sqrt{d}H T\right), \epsilon'\right), \ \phi \in \Phi\right\}.
\end{align*}
Define corresponding function class as follows
\begin{align*}
     \cF^{\pi} &= \left\{ f:\cX\to [-1,1] ~\bigg\mid~ f(x) = \sum_a \pi(a|x) \phi(x,a)^\top\theta, \ \ \text{for \ } \theta\in\mathbb{B}^d(\sqrt{d}) \text{\ and \ } \phi\in\Phi \right\} \\
     \cF &= \left\{f:\cX\to [-1,1] ~\bigg\mid~ f \in \bigcup_{\pi\in\Pi_{\text{\rm lin} }} \cF^\pi  \right\}.
\end{align*}
\label{def:pi and F}
\end{definition}

Our main algorithm is given in \cref{alg:model-free policy}, which shares the same structure as \cref{alg:low-rank bandit}, but with a different initial phase to learn expected feature  estimator $\hphi_h(\pi) = \sum_{(x,a)\in \cX\times \cA} \hd_{h}^\pi(x) \pi(a|x) \hphi_{h}(x,a)$ to approximate $\phi_h^\star(\pi) = \sum_{(x,a)\in \cX\times \cA} d_{h}^\pi(x) \pi(a|x) \phi_{h}^\star(x,a)$ for every $h \in [H]$. In \cref{alg:low-rank bandit}, under \cref{assm: model-base}, it is feasible to use established model-based approach to learn an accurate estimated transition $\widehat{P}$ together with its feature $\hphi$. The occupancy estimator $\hd_{1:H}^\pi$ is induced by $\widehat{P}$ which also enjoy small errors. However, when we move to model-free settings with \cref{assm:real}, there is no existing approach that could guarantee a good estimation for $\hd_{h}^\pi(x)$ and $\hphi$.

To tackle this challenge, we first call $\texttt{VoX}$ \citep{mhammedi2023efficient} to construct a policy cover $\Psi_{1:H}^{\cov}$, and then play every policy in $\Psi_{1:H}^{\cov}$ for $n$ episodes to collect data. Subsequently, these data are fed into  \cref{alg: estimate occupancy} to jointly solve estimated occupancy $\hd^\pi_{1:H}$ and feature $\hphi_{1:H}$. \cref{alg: estimate occupancy} is similar to  \cite[Algorithm 1]{liu2023towards}, which is used to estimate occupancy on the fly for linear MDP. In \cref{alg: estimate occupancy}, given a target policy $\pi$, we jointly solve $\hat{\phi}_{1:H} \in \Phi$, $\hat{d}_{1:H}^\pi \in [0,1]^{|\cX|}$, and $(\hat{\xi}_{1:H,f})_{f \in \cF^\pi} \subset \mathbb{B}^d\left(\sqrt{d}\right)$ that satisfies four constrains, where $\hat{\xi}_{h,f}$ is the estimation of $\xi_{h,f}^\star :=\sum_{x' \in \cX}\mu_{h+1}^\star(x')f(x')$. The first constraint \cref{eq:con_1} ensures the estimated occupancy $\hat{d}_{1:H}^\pi$ are valid distrbutions. The second constrant \cref{eq:con_2} enforces the estimated values to follow the dynamic programming relationship between the occupancy of layer $h$ and layer $h+1$, which helps to control the propagation of estimation errors across layers through the bias of $\hat{\phi}_h(x,a)^\top \hat{\xi}_{h,f}$. The third constrain \cref{eq:con_3} and fourth constrain \cref{eq:con_4} are then used to bound the estimated bias of $\hat{\phi}_h(x,a)^\top \hat{\xi}_{h,f}$ by utilizing the data collected from policies in policy covers $\Psi_{1:H}^{\cov}$. Note that applying \cref{eq:con_4} requires access to the whole state space, which is an additional assumption not needed in previous algorithms. The gurantee of \cref{alg:model-free policy} is given in \cref{thm:model-free-p guaran}, where the $\wtilde{O}$ hides the logarithmic dependence on $d,H,|\cA|, T$.

\begin{theorem} \label{thm:model-free-p guaran}
    \cref{alg:model-free policy} ensures $ \mathrm{Reg}_T \le  \wtilde{O}\left(d^{8} H^6 |\cA|  T^{\frac{2}{3}} \log(|\Phi|))\right) $.
\end{theorem}

\begin{algorithm}[H]
\caption{Model-Free Algorithm for Bandit Feedback} 
\label{alg:model-free policy}
\begin{algorithmic}[1] 
    \Input Policy class $\Pi = \Pi_{\text{\rm lin}}^{\text{\rm cov}}(\frac{1}{T})$. 
    \State Set $\epsilon = 18^{-1}d^{\frac{5}{2}}T^{-\frac{1}{3}}$, $\gamma = T^{-\frac{1}{3}}$, $\beta = T^{-\frac{1}{3}}$, $\eta = ({4Hd|\cA|})^{-1} T^{-\frac{2}{3}}$, $n = 11250 d^{\frac{5}{2}}|\cA|T^{\frac{2}{3}}\log\frac{3dnHT|\Phi|}{\delta}$ and $T_0 = \wtilde{O}\left(\epsilon^{-2} |\cA| d^{13} H^6 \log(|\Phi|/\delta))\right)$.
    \State Get $\Psi^\cov_{1:H}\gets \texttt{VoX}(\Phi, \veps, \delta)$ using $T_0$ episodes.
%\hfill  \algcommentlight{Call $\texttt{VoX}$ as in \cite{mhammedi2023efficient} to compute a policy cover.}
\State For every policy $\pi' \in \Psi^\cov_{1:H}$, play it for $n$ episodes and get the data set $\left(\calD_h^{\pi'}\right)_{h \in [H]}$ where $\calD_h^{\pi'}$ consists of tuples $(x,a,x')$ such that $(x,a) \sim d_h^{\pi'}$ and  $x' \sim P^\star(\cdot ~\mid~ x,a)$.

% \State  Get $\hat{\phi}_h(\pi)\gets \text{\rm EOM-PC}\left(\pi, \left(\calD_h^{\pi'}\right)_{h \in [H], \pi' \in \Psi^\cov_{1:H}}\right)$ from \cref{alg: estimate occupancy} for every $\pi \in  \Pi$.
    %\State Define transitions $\hatp_{1:H-1}$ as 
    %$$\hatp_h(x'\mid  x,a) = \hphi_h(x,a)^\top \hpsi_{h+1}(x'), \quad \forall (x,a,x') \in \cX \times \cA \times \cX.$$
    \State Define the policy space $\Pi' = \{\pi':~ \exists \pi\in\Pi, \ \ \pi_h'(\cdot\mid x)= (1-\beta)\pi_h(\cdot\mid x) +  \beta/|\cA|, \ \  \forall x, h\}$. %\zm{Define $\pi_\unif$ in prelim} 

    \State  Get $\hat{\phi}_h(\cdot)\gets \text{\rm EOM-PC}\left(\Pi', \left(\calD_h^{\pi'}\right)_{h \in [H], \pi' \in \Psi^\cov_{1:H}}\right)$ from \cref{alg: estimate occupancy}.
   
    \For{$t = T_0 + 1,\, T_0 + 2,\ldots, T$}
         \State Define $ p^t(\pi) \propto \exp\left(-\eta \sum_{i=1}^{t-1}  \left(\hatell^i(\pi) - b^i(\pi)\right)\right),$ for all $\pi \in \Pi'$. 
         \State Let $\tilp^t(\pi) = (1-\gamma)p^t(\pi) +  \frac{\gamma}{H-1}\sum_{h=1}^{H-1}  J_h$, where $J_h = \john(\hatphi_h(\cdot), \Pi')$. \algcommentlight{$\john$ as in \S\ref{sec:prelim}}  %\zm{Do we define \text{John} somewhere? Also use macro john}  \cwcomment{I made some edit here. John's exp defined in prelim}
         \State Execute policy $\bpi^t\sim \tilp^t$ and observe trajectory $(\x^t_{1:H}, \a^t_{1:H})$ and losses $\bell^t_h =\ell^t_h(\x_h^t,\a_h^t)$.
         \State Define $\Sigma^t_{h} = \sum_{\pi \in \Pi'}\tilp^t(\pi)\cdot \hatphi_h(\pi)\hatphi_h(\pi)^\top$, $b^t(\pi) =d^{\frac{11}{2}} H T^{-\frac{1}{3}}  \cdot \sum_{h=1}^{H-1} \|\hatphi_h(\pi)\|_{(\Sigma^t_{h})^{-1}}$, and
             \begin{align*}  
            \hell^t(\pi) = \frac{\pi_1(\a_1^t\mid \x_1^t)}{\bpi_1^t(\a_1^t\mid \x_1^t)}\bell_1^t + \sum_{h=2}^H \hphi_{h-1}(\pi)^\top\left(\Sigma_{h-1}^t\right)^{-1} \hphi_{h-1}(\bpi^t) \frac{\pi_h(\a_h^t\mid \x_h^t)}{\bpi_h^t(\a_h^t\mid \x_h^t)}\bell_h^t.
             \end{align*}
         \EndFor
    \end{algorithmic}
\end{algorithm}

\begin{algorithm}[H]
     \caption{EOM-PC$\left(\Pi, \left(\calD_h^{\pi'}\right)_{h \in [H], \pi' \in \Psi^\cov_{1:H}}\right)$ (Estimate Occupancy Measure with Policy Cover)} \label{alg: estimate occupancy}
    \begin{algorithmic}
    \Input The policy class $\Pi$, datasets $\left(\calD_h^{\pi'}\right)_{h \in [H]}$ for every $\pi' \in \Psi^\cov_{1:H}$
        \State Jointly find $\hat{\phi}_{h} \in \Phi$, $(\hat{d}_h^\pi)_{\pi\in\Pi} \in [0,1]^{|\cX|}$, and $(\hat{\xi}_{h,f})_{f \in \cF} \subset \mathbb{B}^d\left(\sqrt{d}\right)$ for any $h \in [H]$ such that for all $\pi\in\Pi$, 
        \vspace{10pt}
        \begin{align}
            &\sum_{x\in \cX}\hat{d}_h^{\pi}(x) = 1, \qquad \forall h\in[H]  \label{eq:con_1}\\[15pt]
            &\sum_{x'\in\cX} \hat{d}_{h+1}^{\pi}(x') f(x') = \sum_{x\in \cX}\sum_{a \in \cA} \hat{d}^{\pi}_{h}(x) \pi(a|x) \hat{\phi}_h(x,a)^\top \hat{\xi}_{h,f}, \qquad \forall f\in \cF, h\in[H] \label{eq:con_2}\\[15pt]
            &\sum_{x,a,x'\in\calD_h^{\pi'}} \left(f(x') - \hat{\phi}_h(x,a)^\top \hat{\xi}_{h,f}\right)^2 - \min_{(\phi, \xi) \in \Phi \times \mathbb{B}_d(\sqrt{d})}\sum_{x,a,x'\in\calD_h^{\pi'}} \left(f(x') - \phi(x,a)^\top \xi\right)^2 \nonumber \\
            &\qquad \qquad \qquad \qquad \quad \qquad \leq  132 d^{\frac{3}{2}}\log(3dnHT|\Phi|/\delta), \ \ \ \  \ \ \ \forall \pi' \in \Psi^\cov_{1:H},  \ f\in\cF, h\in[H]  \label{eq:con_3} \\[15pt]
            &\max_{x,a}\left|\hat{\phi}_h(x,a)^\top \hat{\xi}_{h,f}\right|\leq 1 \qquad \forall f\in \cF, h\in[H] \label{eq:con_4}
        \end{align}
        \vspace{10pt}
   \item[\textbf{Output:}]  $\hat{\phi}_h:\Pi\rightarrow\mathbb{R}^d,\hat{\phi}_{h}(\pi) = \sum_{(x,a)\in \cX\times \cA} \hat{d}_h^\pi(x) \pi(a|x) \hat{\phi}_h(x,a)$, $\forall h \in [H]$.
    \end{algorithmic}
\end{algorithm}

\subsection{Analysis of Occupancy Estimation from \cref{alg: estimate occupancy}}
\begin{lemma}
    With probability $1-\delta$,  
        $\phi_{1:H}^\star, \left(d_{1:H}^\pi\right)_{\pi \in \Pi'}
    $, and $\xi^\star_{h,f}:=\sum_{x' \in \cX}\mu_{h+1}^\star(x')f(x'), \forall f \in \cF, \forall h \in [H]$ is a solution to \cref{alg: estimate occupancy}. 
\end{lemma}

\begin{proof}
       Since for any policy $\pi$ and any $h \in [H]$, $\sum_{x \in \cX} d_h^\pi(x) = 1$, \cref{eq:con_1} holds. For any policy $\pi$, any $f \in \cF^\pi$ and any $h \in [H]$, we have
       \begin{align*}
           \sum_{x'\in\cX} d_{h+1}^{\pi}(x') f(x') &=  \sum_{x'\in\cX} \sum_{x\in \cX} \sum_{a \in \cA} d^{\pi}_{h}(x) \pi(a|x) P^\star_h(x' \mid x,a) f(x')
           \\&=  \sum_{x\in \cX_{h}}\sum_{a \in \cA} d^{\pi}_{h}(x) \pi(a|x) \phi_h^\star(x,a)^\top \sum_{x'\in\cX_{h+1}}  \mu_{h+1}^\star(x') f(x')
           \\&=  \sum_{x\in \cX_{h}}\sum_{a \in \cA} d^{\pi}_{h}(x) \pi(a|x) \xi_{h,f}^\star.
       \end{align*}
       Thus, \cref{eq:con_2} holds.  From Exercise 27.6 of \cite{lattimore2020bandit}, the $\epsilon$-net of $\mathbb{B}_d(R)$ is $\left(\frac{3R}{\epsilon}\right)^d$. Thus. $\left|\Pi'\right|= \left| \Pi_{\text{\rm lin}}^{\text{\rm cov}}(\frac{1}{T})\right| = |\Phi|\left(3\sqrt{d}HT^2\right)^d$. We also have $\left|\mathcal{C}\left(\mathbb{B}_d(\sqrt{d}), \frac{1}{T}\right)\right| = \left(3\sqrt{d}T\right)^d$ and for any policy $\pi$,  $\left|\mathcal{C}\left(\cF^\pi, \frac{1}{T}\right)\right| = |\Phi|\left(3\sqrt{d}T\right)^d$. To consider all possible instances, define
       \begin{align*}
           \mathcal{N}_T := \left| \Pi'\right|\left|\mathcal{C}\left(\mathbb{B}_d(\sqrt{d}), \epsilon\right)\right|\left|\mathcal{C}\left(\cF^\pi, \epsilon\right)\right| \left|\Psi^{\cov}_{1:H}\right| |\Phi| H \le dH^2|\Phi|^3\left(3\sqrt{d}HT^2\right)^{3d}
       \end{align*}
       
       Thus, by union bound, with probability of $1-\delta$, for every $\pi \in \Pi'$, every $\pi' \in \Psi^{\cov}_{1:H}$, every $f \in \mathcal{C}\left(\cF^\pi, \epsilon\right)$, every $\xi \in \mathcal{C}\left(\mathbb{B}_d(\sqrt{d}), \epsilon\right)$, every $\phi \in \Phi$ and every $h \in [H]$, we have
       \begin{align*}
           &\sum_{x,a,x'\in \calD_h^{\pi'}} \left(f(x') - \phi^\star_h(x,a)^\top \xi^\star_{h,f}\right)^2 -  \sum_{x,a,x'\in\calD_h^{\pi'}} \left(f(x') - \phi(x,a)^\top \xi\right)^2
           \\&= -2\sum_{x,a,x'\in \calD_h^{\pi'}}  \left(f(x') - \phi^\star_h(x,a)^\top \xi^\star_{h,f}\right) \left(\phi^\star_h(x,a)^\top \xi^\star_{h,f} - \phi(x,a)^\top \xi\right) 
           \\&\qquad - \sum_{x,a,x'\in \calD_h^{\pi'}} \left(\phi^\star_h(x,a)^\top \xi^\star_{h,f} - \phi(x,a)^\top \xi\right)^2
           \\&= -2\sum_{x,a,x'\in \calD_h^{\pi'}}  \left(f(x') - \E_{x' \sim P^\star(\cdot \mid x,a)}[f(x')]\right) \left(\phi^\star_h(x,a)^\top \xi^\star_{h,f} - \phi(x,a)^\top \xi\right) 
           \\&\qquad - \sum_{x,a,x'\in \calD_h^{\pi'}} \left(\phi^\star_h(x,a)^\top \xi^\star_{h,f} - \phi(x,a)^\top \xi\right)^2
           \\&\le  8 \sqrt{\sum_{x,a,x'\in \calD_h^{\pi'}}\left(\phi^\star_h(x,a)^\top \xi^\star_{h,f} - \phi(x,a)^\top \xi\right)^2\log\frac{|\mathcal{N}_T|}{\delta}} + 4\sqrt{d}\log\frac{|\mathcal{N}_T|}{\delta}
           \\&\qquad - \sum_{x,a,x'\in \calD_h^{\pi'}} \left(\phi^\star_h(x,a)^\top \xi^\star_{h,f} - \phi(x,a)^\top \xi\right)^2 \tag{Freedman's Inequality}
           \\&\le  20 \sqrt{d}\log\frac{|\mathcal{N}_T|}{\delta}\tag{AM-GM}
           \\&\le 120 d^{\frac{3}{2}}\log\frac{3dHT|\Phi|}{\delta}
       \end{align*}
        Bounding the distance through $\frac{1}{T}$-net, we have with probability of $1-\delta$, for every $\pi \in \Pi'$, every $\pi' \in \Psi^{\cov}_{1:H}$, every $f \in \cF^\pi$, every $\xi \in \mathbb{B}_d(\sqrt{d})$, every $\phi \in \Phi$ and every $h \in [H]$,
        \begin{align*}
            \sum_{x,a,x'\in \calD_h^{\pi'}} \left(f(x') - \phi^\star_h(x,a)^\top \xi^\star_{h,f}\right)^2 -  \sum_{x,a,x'\in\calD_h^{\pi'}} \left(f(x') - \phi(x,a)^\top \xi\right)^2 &\le  120 d^{\frac{3}{2}}\log\frac{3dHT|\Phi|}{\delta \epsilon} + \frac{12\sqrt{d}n}{T} 
            \\&\le 132 d^{\frac{3}{2}}\log\frac{3dnHT|\Phi|}{\delta}
        \end{align*}
       Thus, \cref{eq:con_3} also holds. Finally, for all $x,a$, we have $$\left|\phi_h^\star(x,a)^\top \xi^\star_{h,f}\right| = \left|\sum_{x' \in \cX}\phi_h^\star(x,a)^\top \mu_{h+1}^\star(x')f(x')\right| = \left|\sum_{x' \in \cX}P^\star\left(x'~\mid~ x,a\right) f(x')\right| \le 1.$$  Thus, $
        \phi^\star_h, d_h^\pi
    $, and $\xi^\star_{h,f}:=\sum_{x' \in \cX}\mu_{h+1}^\star(x')f(x'), \forall f \in \cF^\pi, h \in [H]$ satisfy all \cref{eq:con_1} -- \cref{eq:con_4} and is a solution to \cref{alg: estimate occupancy}. 
\end{proof}

 % \begin{align*}
 %           &\sum_{x,a,x'\in \calD_h^{\pi'}} \left(f(x') - \phi^\star_h(x,a)^\top \xi^\star_{h,f}\right)^2 -  \sum_{x,a,x'\in\calD_h^{\pi'}} \left(f(x') - \phi_{\min, h}^\pi(x,a)^\top \xi_{\min, h,f}^\pi\right)^2
 %           \\&= -2\sum_{x,a,x'\in \calD_h^{\pi'}}  \left(f(x') - \phi^\star_h(x,a)^\top \xi^\star_{h,f}\right) \left(\phi^\star_h(x,a)^\top \xi^\star_{h,f} - \phi_{\min, h}^\pi(x,a)^\top \xi_{\min, h,f}^\pi\right) 
 %           \\&\qquad - \sum_{x,a,x'\in \calD_h^{\pi'}} \left(\phi^\star_h(x,a)^\top \xi^\star_{h,f} - \phi_{\min, h}^\pi(x,a)^\top \xi_{\min, h,f}^\pi\right)^2
 %           \\&= -2\sum_{x,a,x'\in \calD_h^{\pi'}}  \left(f(x') - \E_{x' \sim P^\star(\cdot \mid x,a)}[f(x')]\right) \left(\phi^\star_h(x,a)^\top \xi^\star_{h,f} - \phi_{\min, h}^\pi(x,a)^\top \xi_{\min, h,f}^\pi\right) 
 %           \\&\qquad - \sum_{x,a,x'\in \calD_h^{\pi'}} \left(\phi^\star_h(x,a)^\top \xi^\star_{h,f} - \phi_{\min, h}^\pi(x,a)^\top \xi_{\min, h,f}^\pi\right)^2
 %           \\&\le  8 \sqrt{\sum_{x,a,x'\in \calD_h^{\pi'}}\left(\phi^\star_h(x,a)^\top \xi^\star_{h,f} - \phi_{\min, h}^\pi(x,a)^\top \xi_{\min, h,f}^\pi\right)^2\log\frac{|\mathcal{N}_{\epsilon}| T}{\delta}} + 4\sqrt{d}\log\frac{|\mathcal{N}_{\epsilon}| T}{\delta}
 %           \\&\qquad - \sum_{x,a,x'\in \calD_h^{\pi'}} \left(\phi^\star_h(x,a)^\top \xi^\star_{h,f} - \phi_{\min, h}^\pi(x,a)^\top \xi_{\min, h,f}^\pi\right)^2 \tag{Freedman's Inequality}
 %           \\&\le  20 \sqrt{d}\log\frac{|\mathcal{N}_{\epsilon}| T}{\delta}\tag{AM-GM}
 %       \end{align*}

\begin{lemma} \label{lem: recursion err}
     With probability $1-\delta$, for all $f\in\cF$, any solution $\hat{d}^\pi$ from \cref{alg: estimate occupancy} for any $\pi \in \Pi'$ satisfies
    \begin{align*}
         \left|\sum_x  (\hat{d}^\pi_h(x) - d^\pi_h(x)) f(x) \right| \le d^{5}H T^{-\frac{1}{3}}
    \end{align*}
\end{lemma}

\begin{proof}
% Define $L_{h}^\pi(\phi_h, \xi_{h,f}) = \sum_{x,a,x'\in \calD_h^{\pi'}} \left(f(x') - \phi_h(x,a)^\top \xi_{h,f}\right)^2$ and let $\phi_{\min,h}^\pi, \xi_{\min,h, f} = \argmin_{\phi, \xi_f} L_{h}^\pi(\phi, \xi_f) $

% Define $\phi_{\min, h}^\pi, \xi_{\min, h,f}^\pi = \argmin_{(\phi, \xi) \in \Phi \times \mathbb{B}_d(\sqrt{d})}\sum_{x,a,x'\in\calD_h^{\pi'}} \left(f(x') - \phi(x,a)^\top \xi\right)^$ for every $\pi' \in \Psi^{\cov}_{1:H}$. 

For every solution $\hphi_{1:H}, \hat{\xi}_{1:H, f \in \cF^\pi}, \left(\hat{d}_{1:H}^\pi\right)_{\pi \in \Pi'}$ of \cref{alg: estimate occupancy}, we have
\begin{align}
    &\sum_{x,a,x'\in \calD_h^{\pi'}} \left(f(x') - \hat{\phi}_h(x,a)^\top \hat{\xi}_{h,f}\right)^2 -  \sum_{x,a,x'\in \calD_h^{\pi'}} \left(f(x') - \phi_h^\star(x,a)^\top \xi^\star_{h,f}\right)^2 \nn
    \\&= \sum_{x,a,x'\in \calD_h^{\pi'}} \left(f(x') - \hat{\phi}_h(x,a)^\top \hat{\xi}_{h,f}\right)^2 -  \min_{(\phi, \xi)}\sum_{x,a,x'\in\calD_h^{\pi'}} \left(f(x') - \phi(x,a)^\top \xi \right)^2 \nn
    \\&\qquad + \underbrace{\min_{(\phi, \xi)} \sum_{x,a,x'\in\calD_h^{\pi'}} \left(f(x') - \phi(x,a)^\top \xi \right)^2 -\sum_{x,a,x'\in \calD_h^{\pi'}} \left(f(x') - \phi_h^\star(x,a)^\top \xi^\star_{h,f}\right)^2}_{\le 0} \nn
    \\&\le 132 d^{\frac{3}{2}}\log\frac{3dnHT|\Phi|}{\delta}  \label{eq:err_upper} 
\end{align}
where the last step comes from the constrain \cref{eq:con_3}. On the other hand 
 \begin{align}
   &2\sum_{x,a,x'\in \calD_h^{\pi'}} \left(f(x') - \hat{\phi}_h(x,a)^\top \hat{\xi}_{h,f}\right)^2 -   2\sum_{x,a,x'\in \calD_h^{\pi'}} \left(f(x') - \phi_h^\star(x,a)^\top \xi^\star_{h,f}\right)^2 \nn
   \\&= 4\sum_{x,a,x'\in \calD_h^{\pi'}}  \left(f(x') - \phi^\star_h(x,a)^\top \xi^\star_{h,f}\right) \left(\phi^\star_h(x,a)^\top \xi^\star_{h,f} - \hat{\phi}_h(x,a)^\top \hat{\xi}_{h,f}\right) \nonumber \\
   &\qquad \qquad + 2\sum_{x,a,x'\in \calD_h^{\pi'}} \left(\phi^\star_h(x,a)^\top \xi^\star_{h,f} - \hat{\phi}_h(x,a)^\top \hat{\xi}_{h,f}\right)^2 \nn
   \\&= 4\sum_{x,a,x'\in \calD_h^{\pi'}}  \left(f(x') - \E_{x' \sim P^\star(\cdot \mid x,a)}[f(x')]\right) \left(\phi^\star_h(x,a)^\top \xi^\star_{h,f} - \hat{\phi}_h(x,a)^\top \hat{\xi}_{h,f}\right)\nonumber  \\
   &\qquad \qquad + 2\sum_{x,a,x'\in \calD_h^{\pi'}} \left(\phi^\star_h(x,a)^\top \xi^\star_{h,f} - \hat{\phi}_h(x,a)^\top \hat{\xi}_{h,f}\right)^2 \label{eq:err_eq}
\end{align}

Combing \cref{eq:err_upper} and \cref{eq:err_eq}, we have
\begin{align*}
    &\sum_{x,a,x'\in \calD_h^{\pi'}} \left(\phi^\star_h(x,a)^\top \xi^\star_{h,f} - \hat{\phi}_h(x,a)^\top \hat{\xi}_{h,f}\right)^2  
    \\&\le -4\sum_{x,a,x'\in \calD_h^{\pi'}}  \left(f(x') - \E_{x' \sim P^\star(\cdot \mid x,a)}[f(x')]\right) \left(\phi^\star_h(x,a)^\top \xi^\star_{h,f} -  \hat{\phi}_h(x,a)^\top \hat{\xi}_{h,f}\right) \\
    &\qquad - \sum_{x,a,x'\in \calD_h^{\pi'}} \left(\phi^\star_h(x,a)^\top \xi^\star_{h,f} - \hat{\phi}_h(x,a)^\top \hat{\xi}_{h,f}\right)^2 
    \\&\qquad + 264 d^{\frac{3}{2}}\log\frac{3dnHT|\Phi|}{\delta} 
    \\&\le 288 d^{\frac{3}{2}}\log\frac{3dnHT|\Phi|}{\delta} \tag{\cref{lem:freedhelp} with $\lambda = \frac{1}{8}$}
\end{align*}

Since for every data tuple $(x,a,x') \in \calD_h^{\pi'}$, $(x,a) \sim d_h^{\pi'}$ independently, by \cref{lem:corbern}, for every $\pi' \in \Psi_{1:H}^{\cov}$ we have 
\begin{align*}
    \E^{\pi'}\left[\left(\phi^\star_h(x,a)^\top \xi^\star_{h,f} - \hat{\phi}_h(x,a)^\top \hat{\xi}_{h,f}\right)^2 \right] &\le \frac{2}{n}\sum_{x,a \in \calD_h^{\pi'}} \left(\phi^\star_h(x,a)^\top \xi^\star_{h,f} - \hat{\phi}_h(x,a)^\top \hat{\xi}_{h,f}\right)^2  + \frac{24d\log\left(\frac{6\sqrt{d}|\Phi|}{\delta}\right)}{n}
    \\&\le \frac{600 d^{\frac{3}{2}}}{n} \log\frac{3dnHT|\Phi|}{\delta}
\end{align*}
This implies there exists a representation $\hat{\phi}_{1:H}(x,a)$ such that for every $f \in \cF$ and any $h \in [H]$, there exists $\hat{\xi}_{h,f}$ such that
\begin{align} \label{eq:learn o rep}
   \max_{\pi' \in \Psi_{1:H}^{\cov}} \E^{\pi'}\left[\left(\phi^\star_h(x,a)^\top \xi^\star_{h,f} - \hat{\phi}_h(x,a)^\top \hat{\xi}_{h,f}\right)^2 \right] \le \frac{600 d^{\frac{3}{2}}}{n} \log\frac{3dnHT|\Phi|}{\delta}.
\end{align}
\cref{eq:learn o rep} matches \cref{eq:repp3} with a different error bound on the right hand. From \cref{lem:vox}, we have $\Psi_h^\cov$ is a $(\frac{1}{8Ad},\veps)$-policy cover for layer $h$,  following the rest of the proof in \cref{lem:rep}, for every $\pi$ and every $f \in \cF, h \in [H]$, we have
\begin{align} \label{eq:learn o rep 2}
    \left|\E^{\pi}\left[\phi^\star_h(x,a)^\top \xi^\star_{h,f} - \hat{\phi}_h(x,a)^\top \hat{\xi}_{h,f} \right] \right| \le \frac{75 d^{\frac{5}{4}}}{\sqrt{n}} \sqrt{|\cA|\log\frac{3dnHT|\Phi|}{\delta}} + 9d^{\frac{5}{2}}\epsilon.
\end{align}
The $d^{\frac{5}{2}}$ in the second term of \cref{eq:learn o rep 2} improves the $d^{\frac{7}{2}}$ term in \cref{eq:replearn} because $\hat{\xi}_{h,f} \in \mathbb{B}_d(\sqrt{d})$ rather than $\mathbb{B}_d(d^{\frac{3}{2}})$. Putting the choice $n = 11250 d^{\frac{5}{2}}|\cA|T^{\frac{2}{3}}\log\frac{3dnHT|\Phi|}{\delta}$ and $\epsilon = \frac{d^{\frac{5}{2}}}{18}T^{-\frac{1}{3}}$ into \cref{eq:learn o rep 2}, we have for every $\pi$ and every $f \in \cF, h \in [H]$, we have
\begin{align} \label{eq:repl-s}
    \left|\E^{\pi}\left[\phi^\star_h(x,a)^\top \xi^\star_{h,f} - \hat{\phi}_h(x,a)^\top \hat{\xi}_{h,f} \right] \right| \le d^{5} T^{-\frac{1}{3}}
\end{align}

 % \\&\le  8 \sqrt{\sum_{x,a,x'\in \calD_h^{\pi'}}\left(\phi^\star_h(x,a)^\top \xi^\star_{h,f} - \phi_{\min, h}^\pi(x,a)^\top \xi_{\min, h,f}^\pi\right)^2\log\frac{|\mathcal{N}_{\epsilon}| T}{\delta}} + 4\sqrt{d}\log\frac{|\mathcal{N}_{\epsilon}| T}{\delta}
 %           \\&\qquad - \sum_{x,a,x'\in \calD_h^{\pi'}} \left(\phi^\star_h(x,a)^\top \xi^\star_{h,f} - \phi_{\min, h}^\pi(x,a)^\top \xi_{\min, h,f}^\pi\right)^2 \tag{Freedman's Inequality}
 %           \\&\le  20 \sqrt{d}\log\frac{|\mathcal{N}_{\epsilon}| T}{\delta}\tag{AM-GM}

% $\sum_{x,a,x'\in \calD_h^{\pi'}} \left(f(x') - \phi^\star(x,a)^\top \xi^\star_{h,f}\right)^2 -  \sum_{x,a,x'\in\calD_h^{\pi'}} \left(f(x') - \phi_{\min, h}^\pi(x,a)^\top \xi_{\min, h,f}^\pi\right)^2$

Utilizing above results,  for every $\pi$, every $f \in \cF$ and any $h \in [H-1]$, we have

    \begin{align*}
         &\left|\sum_x  (\hat{d}^\pi_{h+1}(x) - d^\pi_{h+1}(x)) f(x) \right| 
         \\&= \left|\sum_{x,a} \hat{d}^\pi_h(x)\pi_h(a\mid x) \hat{\phi}_h(x,a)^\top \hat{\xi}_{h,f}  - \sum_{x,a} d^\pi_h(x)\pi(a\mid x) \phi_h^\star(x,a)^\top \xi^\star_{h,f}\right| \tag{\cref{eq:con_2}} \\
         &\leq \left|\sum_{x,a}d^\pi_h(x)\pi_h(a\mid x) \left(\phi_h^\star(x,a)^\top \xi_{h,f}^\star - \hat{\phi}_h(x,a)^\top \hat{\xi}_{h,f}\right) \right| \\
         &\qquad + \left|\sum_{x} 
 \left(\hat{d}^\pi_h(x) - d^\pi_h(x)\right)\underbrace{\sum_{a} \pi_h(a|x)\hat{\phi}_h(x,a)^\top \hat{\xi}_{h,f}}_{\in \cF} \right| \\
         &\leq d^{5} T^{-\frac{1}{3}}  + \left|\sum_{x} 
 \left(\hat{d}^\pi_h(x) - d^\pi_h(x)\right)f'(x) \right| \tag{\cref{eq:repl-s}}
    \end{align*}
\end{proof}
where in the last step, we define $f'(x) = \sum_{a} \pi_h(a|x)\hat{\phi}_h(x,a)^\top \hat{\xi}_{h,f} \in \cF$. This allow us to use recursion to finish the proof.

%Define 
%\begin{align*}
%    \hat{\phi}_{h}(\pi) = \sum_{x\in\cX \sum_a \hat{d}_h^\pi(x) \pi(a|x) \hat{\phi}(x,a). 
%\end{align*}

\subsection{Regret Analysis}
We begin by proving \cref{lem:pdis}, showing that the policy class  $\Pi_{\text{\rm lin}}^{\text{\rm cov}}(\epsilon')$ suffices to approximate all policies with small error.
\begin{lemma}
    For any policy $\pi$, there exists a policy $\pi' \in  \Pi_{\text{\rm lin}}^{\text{\rm cov}}(\epsilon')$ such that 
    \begin{align*}
        \sum_{t=1}^T V_1^{\pi'}(x_1; \ell^t) - \sum_{t=1}^T V_1^{\pi}(x_1; \ell^t) \le H\epsilon'
    \end{align*}
\label{lem:pdis}
\end{lemma}

\begin{proof}
Let $\theta\indd{t,\pi}_h \in \bbB_d(H\sqrt{d})$ be such that 
        \begin{align}
          \forall (x,a)\in \cX\times \cA,\quad  Q^{\pi}_h(x,a;\ell^t) = \phi^\star_h(x,a)^\top \theta^{\pi,t}_h.\nn 
        \end{align}
 Such a $\theta^{\pi, t}_h$ is guaranteed to exist by the low-rank MDP structure (\cref{assm:normalizing}) and \cref{assm:linearlossweak}. For every $h \in [H]$, since $\sum_{t=1}^T \theta^{\pi,t}_h \in \bbB_d(\sqrt{d}HT)$, we define $\theta_h' \in \mathcal{C}\left(\mathbb{B}_d \left(\sqrt{d}H T\right), \epsilon'\right)$ such that $\|\theta_h'  - \sum_{t=1}^T \theta^{\pi,t}_h \|_2 \le \epsilon'$, and let $ \pi_h(a\mid x ) = \mathbb{I}\big\{a=\argmin_{a \in \cA} \phi_h^\star(x, a)^\top \theta_h'\big\}$ for every $h \in [H]$. We have $\pi' \in \Pi_{\text{\rm lin}}^{\text{\rm cov}}(\epsilon')$. From \cref{lem:PDL}, we have
\begin{align*}
    &\sum_{t=1}^T V_1^{\pi'}(x_1; \ell^t) - \sum_{t=1}^T V_1^{\pi}(x_1; \ell^t)  
    \\&= \sum_{t=1}^T \sum_{h=1}^H \E_{x \sim d_h^{\pi'}}\left[\sum_{a \in \cA} \left(\pi_h'(a|x) - \pi_h(a|x)\right) Q^{\pi}_h(x,a; \ell^t)\right]
    \\&= \sum_{h=1}^H \E_{x \sim d_h^{\pi'}}\left[\sum_{a \in \cA} \left(\pi'(a|x) - \pi(a|x)\right)\phi^\star_h(x,a)^\top \sum_{t=1}^T \theta^{\pi,t}_h\right]
    \\&= \underbrace{\sum_{h=1}^H \E_{x \sim d_h^{\pi'}}\left[\sum_{a \in \cA} \left(\pi_h'(a|x) - \pi_h(a|x)\right)\phi^\star_h(x,a)^\top \theta_h'\right]}_{\le 0} + \sum_{h=1}^H \E_{x \sim d_h^{\pi'}}\left[\sum_{a \in \cA} \left(\pi_h'(a|x) - \pi_h(a|x)\right)\phi^\star_h(x,a)^\top \left(\sum_{t=1}^T \theta^{\pi,t}_h - \theta_h'\right)\right]
    \\&\le H\epsilon' 
\end{align*}
where the last inequality comes from the fact that $\pi_h(a\mid x ) = \mathbb{I}\big\{a=\argmin_{a \in \cA} \phi_h^\star(x, a)^\top \theta_h'\big\}$ and $\|\theta_h'  - \sum_{t=1}^T \theta^{\pi,t}_h \|_2 \le \epsilon'$ for every $h \in [H]$.
\end{proof}

From \cref{lem:pdis}, for any policy $\pi_1^\star$, there exists a policy $\pi_0^\star \in  \Pi_{\text{\rm lin}}^{\text{\rm cov}}(\frac{1}{T})$ such that
\begin{align}
    \text{\rm Reg}_T(\pi_1^\star) &= \E\left [\sum_{t=1}^T V_1^{\bpi^t}(x_1; \ell^t) - \sum_{t=1}^T V_1^{\pi_1^\star}(x_1; \ell^t)\right] \nonumber 
    \\&\le \E\left[\sum_{t=1}^T V_1^{\bpi^t}(x_1; \ell^t) - \sum_{t=1}^T V_1^{\pi_0^\star}(x_1; \ell^t)\right] + 1 \label{eq:reg-1}
\end{align}

Our candidate policy space $\Pi'$ has a $\beta$ mixture of the random policy. For any policy $\pi_0^\star \in  \Pi_{\text{\rm lin}}^{\text{\rm cov}}(\frac{1}{T})$, define policy $\pi^\star$ such that for any state $x \in \cX$, we have $\pi^\star(\cdot\mid x) = (1-\beta) \pi^\star_0(\cdot\mid x) + \frac{\beta}{|\cA|}$. We have $\pi^\star \in \Pi'$. Define
\begin{align*}
    \hV^\pi_1(x_1; \ell) = \sum_{h=1}^H \sum_{x_h \in \cX} \sum_{a_h \in \cA} \hd_h^\pi(x_h) \pi_h(a_h\mid x_h) \cdot \ell_h^t(x_h,a_h)
\end{align*}

Utilizing \cref{eq:reg-1} and the decomposition in \cref{eq:thisone}, together with the fact that $|\Psi_h^{\cov}| \le d$ from \cref{lem:vox},  we have for any policy $\pi_1^\star$, 
\begin{align}
&\text{\rm Reg}_T(\pi^\star_1) \nonumber
\\&\le \underbrace{\E\left [\sum_{t=1}^T \sum_{\pi} p^t(\pi) \left(V^{\pi}_1(\x_1; \ell^t)-  \hV^{\pi}_1(\x_1; \ell^t)\right) \right]}_{\textbf{Bias1}} + \underbrace{\E\left [\sum_{t=1}^T \hV^{\pi^\star}_1(\x_1; \ell^t)-  V^{\pi^\star}_1(\x_1; \ell^t) \right]}_{\textbf{Bias2}}  \nn 
\\&\quad + \underbrace{\E\left [\sum_{t=1}^T \sum_{\pi} p^t(\pi) \hV^{\pi}_1(\x_1; \ell^t)-  \hV^{\pi^\star}_1(\x_1; \ell^t) \right]}_{\textbf{EXP}} \,+\,\,  \underbrace{\E\left [\sum_{t=1}^T\sum_{\pi} \left(\tilp^t(\pi) - p^t(\pi)\right) V^{\pi}_1(\x_1; \ell^t) \right]}_{\textbf{Error1}}  \nn
\\&\qquad + \underbrace{\E\left [\sum_{t=1}^T V^{\pi^\star}_1(\x_1; \ell^t) -  V^{\pi_0^\star}_1(\x_1; \ell^t) \right]}_{\textbf{Error2}} 
+  \wtilde{O}\left(\epsilon^{-2} |\cA| d^{13} H^6 \log(|\Phi|/\delta))\right) + dHn\label{eq:model-free decom}
\end{align}

Following \cref{lem: EXP Error}, we have
\begin{align}
        \textbf{\emph{Error1}} + \textbf{\emph{Error2}} \le H\gamma T +  2H^2\beta T \le 3H^2T^{\frac{2}{3}}.  \label{eq:model-free error}
\end{align}

\begin{lemma}
   \begin{align*}
        \textbf{Bias1} + \textbf{Bias2} \le 2 d^{5}H^2 T^{\frac{2}{3}} 
   \end{align*}
   \label{lem:model-free bias}
\end{lemma}

\begin{proof}
    For any policy $\pi$ and any $t \in [T]$, 
    \begin{align*}
        \left|V^{\pi}_1(\x_1; \ell^t)-  \hV^{\pi}_1(\x_1; \ell^t)\right|  &= \sum_{h=1}^H \sum_{x_h \in \cX}  \left| d_h^\pi(x_h) - \hd_h^\pi(x_h) \right| \sum_{a_h \in \cA}\pi_h(a_h\mid x_h) \cdot \ell_h^t(x_h,a_h)
        \\&\le d^{5}H^2 T^{-\frac{1}{3}}  \tag{\cref{lem: recursion err}}
    \end{align*}
    Thus, 
    \begin{align*}
        \textbf{Bias1} \le  d^{5}H^2 T^{\frac{2}{3}} \qquad \text{and}  \qquad  \textbf{Bias2} \le  d^{5}H^2 T^{\frac{2}{3}} 
    \end{align*}
\end{proof}

We now prove a modle-free counterpart of \cref{lem:value transfer} in \cref{lem: model-free VT}.
\begin{lemma} 
For any episode $t \in [T]$, for any policy $\pi$ we have
    \begin{align*}
        \hV_1^\pi(x_1; \ell^t) &\le \E_{\bpi^t \sim \tilp^t} \E^{\bpi^t}\left[\hell^t(\pi; \bpi^t, \x_{1:H},\a_{1:H})\right] + d^{\frac{11}{2}} H T^{-\frac{1}{3}}  \sum_{h=2}^{H} \left\|\hphi_{h-1}(\pi) \right\|_{\left(\Sigma_{h-1}^t\right)^{-1}},\\
         \hV_1^\pi(x_1; \ell^t) &\ge \E_{\bpi^t \sim \tilp^t} \E^{\bpi^t}\left[\hell^t(\pi; \bpi^t, \x_{1:H},\a_{1:H})\right]  - d^{\frac{11}{2}} H T^{-\frac{1}{3}}  \sum_{h=2}^{H} \left\|\hphi_{h-1}(\pi) \right\|_{\left(\Sigma_{h-1}^t\right)^{-1}}.   
    \end{align*}
\label{lem: model-free VT}
\end{lemma}

\begin{proof}
%First, from the definition in \cref{line:est feature} of \cref{alg:low-rank bandit}, for any $x \in \cX$ with $h \ge 2$, we have
%\begin{align*}
%    \hd_h^{\pi}(x) = \sum_{x',a'} \hd_{h-1}^{\pi}(x',a') \cdot \hatp_{h-1}(x\mid x', a') =  \hphi_{h-1}(\pi)^\top \hpsi_{h}(x)
%\end{align*}
%where $\hphi_{h-1}(\pi) = \sum_{x',a'} \hd_{h-1}^{\pi}(x',a')\hphi_{h-1}(x',a')$.

We now prove the first inequality:
    \begin{align}
         &\hV_1^\pi(x_1; \ell^t) \nn
         \\&= \sum_{h=1}^H \sum_{x_h \in \cX} \sum_{a_h \in \cA} \hd_h^\pi(x_h) \pi_h(a_h\mid x_h) \cdot \ell_h^t(x_h,a_h), \nn
         \\&=  \underbrace{\sum_{a_1\in \cA} \pi_1(a_1\mid x_1)\cdot  \ell_1^t(x_1,a_1)}_{\textbf{First}}+ \underbrace{\sum_{h=2}^H \sum_{x_h \in \cX} \sum_{a_h \in \cA} \hd^\pi_h(x_h)\pi_h(a_h\mid x_h) \phi_h(x_h, a_h)^\top g^t_h}_{\textbf{Remain}} \label{eq:emain}
    \end{align}
    Through importance sampling, we have
    \begin{align*}
           \textbf{First}  =  \E_{\bpi^t \sim \tilp^t}\E^{\bpi^t}\left[\frac{\pi_1(\a_1\mid \x_1)}{\pi_1^t(\a_1 \mid \x_1)}\cdot  \ell_1^t(\x_1,\a_1)\right].
    \end{align*}
    % \begin{align*}
    %     \textbf{First} =   \E_{\bpi^t \sim \tilp^t}\E^{\bpi^t}\left[\frac{\pi_1(\a_1\mid \x_1)}{\sum_{\pi' \in \Pi} \tilp^t(\pi') \pi_1'(\a_1\mid \x_1)}\cdot  \ell_1^t(\x_1,\a_1)\right].
    % \end{align*}
    We now bound the remaining term in \eqref{eq:emain}
    Since $\Sigma_{h-1}^t = \E_{\bpi^t \sim \tilp^t}\left[\hphi_{h-1}(\bpi^t)\hphi_{h-1}(\bpi^t)^\top\right]$, we have
    \begin{align}
           &\textbf{Remain}  \nonumber
           \\&=  \sum_{h=2}^H   \sum_{x_h \in \cX}  \hat{d}^\pi_h(x_h) \underbrace{\sum_{a_h \in \cA} \pi_h(a_h\mid x_h) \phi_{h}(x_h,a_h)^\top g^t_h}_{:=f(x_h)}, \nonumber 
           \\&=  \sum_{h=2}^H  \sum_{x_{h-1}\in\cX_{h-1}}\sum_{a_{h-1}} \hd^\pi_{h-1}(x)\pi_{h-1}(a_{h-1}|x_{h-1}) \hat{\phi}(x_{h-1}, a_{h-1})^\top \hat{\xi}_{h-1,f} \tag{by \cref{eq:con_2}}\nonumber 
           \\&=  \sum_{h=2}^H  \hat{\phi}_{h-1}(\pi)^\top  \hat{\xi}_{h-1,f} \nonumber 
           \\&=  \sum_{h=2}^H  \hat{\phi}_{h-1}(\pi)^\top  (\Sigma^t_{h-1})^{-1} 
\E_{\bpi^t\sim \rho^t}[\hat{\phi}_{h-1}(\bpi^t)\hat{\phi}_{h-1}(\bpi^t)^\top] \hat{\xi}_{h-1,f} \nonumber 
\\&= \E_{\bpi^t\sim \rho^t}\left[ \sum_{h=2}^H  \hat{\phi}_{h-1}(\pi)^\top  (\Sigma^t_{h-1})^{-1} 
\hat{\phi}_{h-1}(\bpi^t)\hat{\phi}_{h-1}(\bpi^t)^\top \hat{\xi}_{h-1,f} \right]\nonumber 
\\&= \E_{\bpi^t\sim \rho^t}\left[ \sum_{h=2}^H  \hat{\phi}_{h-1}(\pi)^\top  (\Sigma^t_{h-1})^{-1} 
\hat{\phi}_{h-1}(\bpi^t)\sum_{x_{h-1}\in \cX_{h-1}}\sum_{a_{h-1}} \hat{d}_{h-1}^{\bpi^t}(x)\pi_{h-1}(a_{h-1}|x_{h-1}) \hat{\phi}(x_{h-1},a_{h-1})^\top \hat{\xi}_{h-1,f} \right]\nonumber 
\\&= \E_{\bpi^t\sim \rho^t}\left[ \sum_{h=2}^H  \hat{\phi}_{h-1}(\pi)^\top  (\Sigma^t_{h-1})^{-1} 
\hat{\phi}_{h-1}(\bpi^t)\sum_{x_{h}\in \cX_{h}} \hat{d}_{h}^{\bpi^t}(x_h)f(x_h) \right]\nonumber 
\\&= \E_{\bpi^t\sim \rho^t}\left[ \sum_{h=2}^H  \hat{\phi}_{h-1}(\pi)^\top  (\Sigma^t_{h-1})^{-1} 
\hat{\phi}_{h-1}(\bpi^t)\sum_{x_{h}\in \cX_{h}} \sum_{a_h} \hat{d}_{h}^{\bpi^t}(x_h) \pi_h(a_h|x_h)\phi_h(x_h,a_h)^\top g^t_h \right]\nonumber 
\\&= \E_{\bpi^t\sim \rho^t}\left[ \sum_{h=2}^H  \hat{\phi}_{h-1}(\pi)^\top  (\Sigma^t_{h-1})^{-1} 
\hat{\phi}_{h-1}(\bpi^t)\sum_{x_{h}\in \cX_{h}} \sum_{a_h} d_{h}^{\bpi^t}(x_h) \pi_h(a_h|x_h)\phi_h(x_h,a_h)^\top g^t_h \right]\nonumber 
\\&\qquad + \E_{\bpi^t\sim \rho^t}\left[ \sum_{h=2}^H  \hat{\phi}_{h-1}(\pi)^\top  (\Sigma^t_{h-1})^{-1} 
\hat{\phi}_{h-1}(\bpi^t)\sum_{x_{h}\in \cX_{h}} \sum_{a_h} \left(\hat{d}_h^{\bpi^t}(x_h) - d_{h}^{\bpi^t}(x_h)\right) \pi_h(a_h|x_h)\phi_h(x_h,a_h)^\top g^t_h \right]\nonumber 
%           \\&= \E_{\bpi^t \sim \tilp^t}\left[\sum_{h=2}^H \sum_{x_h \in \cX} \sum_{a_h \in \cA} \hphi_{h-1}(\pi)^\top \left(\Sigma_{h-1}^t\right)^{-1} \hphi_{h-1}(\bpi^t) \hd_{h}^{\bpi^t}(x_h)\pi_h(a_h\mid x_h) \ell_h^t(x_h,a_h)\right], \nonumber
           \\&=  \E_{\bpi^t \sim \tilp^t}\left[\sum_{h=2}^H \sum_{x_h \in \cX} \sum_{a_h \in \cA} \hphi_{h-1}(\pi)^\top \left(\Sigma_{h-1}^t\right)^{-1} \hphi_{h-1}(\bpi^t) d_{h}^{\bpi^t}(x_h)\pi_h(a_h\mid x_h) \ell_h^t(x_h,a_h)\right] \nonumber
           \\&\qquad +  d^5H T^{-\frac{1}{3}} \E_{\bpi^t \sim \tilp^t}\left[\sum_{h=2}^H \sum_{x_h \in \cX} \sum_{a_h \in \cA} \left|\hphi_{h-1}(\pi)^\top \left(\Sigma_{h-1}^t\right)^{-1} \hphi_{h-1}(\bpi^t) \right|\right], \tag{by \cref{lem: recursion err}} 
           \\&\le  \E_{\bpi^t \sim \tilp^t}\left[\sum_{h=2}^H \sum_{x_h \in \cX} \sum_{a_h \in \cA} \hphi_{h-1}(\pi)^\top \left(\Sigma_{h-1}^t\right)^{-1} \hphi_{h-1}(\bpi^t) d_{h}^{\bpi^t}(x_h,a_h)\frac{\pi_h(a_h\mid x_h)}{\bpi_h^t(a_h\mid x_h)}\ell_h^t(x_h,a_h)\right]  \nonumber
           \\&\qquad +  d^5H T^{-\frac{1}{3}}\sum_{h=2}^H \left\|\hphi_{h-1}(\pi) \right\|_{\left(\Sigma_{h-1}^t\right)^{-1}}  \E_{\bpi^t \sim \tilp^t}\left[\left\|\hphi_{h-1}(\bpi^t)\right\|_{\left(\Sigma_{h-1}^t\right)^{-1}}\right],  \label{eq:model-free cau}
           \\&\le  \E_{\bpi^t \sim \tilp^t}\E^{\bpi^t}\left[\sum_{h=2}^H \hphi_{h-1}(\pi)^\top \left(\Sigma_{h-1}^t\right)^{-1} \hphi_{h-1}(\bpi^t) \frac{\pi_h(\a_h^t\mid \x_h^t)}{\bpi_h^t(\a_h^t\mid \x^t_h)}\ell_h^t(\x_h^t,\a_h^t)\right]   \nonumber
           \\&\qquad +  d^{\frac{11}{2}} H T^{-\frac{1}{3}} \sum_{h=2}^H \left\|\hphi_{h-1}(\pi) \right\|_{\left(\Sigma_{h-1}^t\right)^{-1}}, \nonumber
    \end{align}
    Adding up \textbf{First} and \textbf{Remain}, and using the defintion of $\hat\ell$ in \eqref{eq:lhat} implies the first inequality of the lemma.
 
    The second inequality follows the same procedure except for applying Cauchy-Schwarz inequality in the opposite direction in \cref{eq:model-free cau}.
\end{proof}

Given \cref{lem: model-free VT}, we could follow the same procedure in \cref{lem:EXP EXP} except replacing the factor of bonus $b^t(\pi)$ from $\sqrt{d}H\epsilon$ to $d^{\frac{11}{2}} H T^{-\frac{1}{3}}$. This leads to the following changes. Firstly, by a similar argument as \cref{eq:EXP2FTRL}, we have
\begin{align*}
    \textbf{EXP} = \textbf{FTRL} + 4d^6H^2 T^{\frac{2}{3}}
\end{align*}
where 
\begin{align*}
 \textbf{FTRL}&\coloneqq \sum_{t=1}^T \sum_{\pi} p^t(\pi)\cdot \left(\E_{\bpi^t \sim \tilp^t} \E^{\bpi^t}\left[\hell^t(\pi;\bpi^t, \x_{1:H},\a_{1:H})\right]- b^t(\pi) \right) \nn \\
& \qquad - \sum_{t=1}^T \left(\E_{\bpi^t \sim \tilp^t} \E^{\bpi^t}\left[\hell^t(\pi^\star; \bpi^t,\x_{1:H},\a_{1:H})\right]-b^t(\pistar)\right).
\end{align*}

Secondly, now we have
\begin{align*}
 \left|b^t(\pi)\right|& = \left|d^{\frac{11}{2}}HT^{-\frac{1}{3}}\sum_{h=2}^{H} \left\|\hphi_{h-1}(\pi) \right\|_{\left(\Sigma_{h-1}^t\right)^{-1}}\right|  \le d^6  H^2\frac{T^{-\frac{1}{3}}}{\sqrt{\gamma}} =  d^6  H^2 T^{-\frac{1}{6}}. \tag{$\gamma = T^{-\frac{1}{3}}$}
\end{align*}

To ensure $\eta \left|\hell^t(\pi; \pi^t, x_{1:H},a_{1:H}) - b^t(\pi)\right| \le 1$, it suffices to set $\eta \le \left(\frac{2|\cA| H d}{\beta \gamma} + d^6  H^2 T^{-\frac{1}{6}} \right)^{-1} = \left(2|\cA| H d T^{\frac{2}{3}} + d^6  H^2 T^{-\frac{1}{6}} \right)^{-1}$ from $\beta = \gamma = T^{-\frac{1}{3}}$. Thus our choice $\eta = ({4Hd|\cA|})^{-1} T^{-\frac{2}{3}}$ satisfies the condition if we assume $T \ge d^6H^2$. Moreover, now we have
\begin{align*}
 \textbf{Stability-2} &= 2\eta \E\left[\sum_{t=1}^T\sum_{\pi \in \Pi'} p^t(\pi) b^t(\pi)^2\right],
 \\&=  2\eta d^{11}H^3 T^{-\frac{2}{3}} \E\left[\sum_{t=1}^T\sum_{h=2}^{H}\sum_{\pi \in \Pi'} p^t(\pi) \left\|\hphi_{h-1}(\pi) \right\|_{\left(\Sigma_{h-1}^t\right)^{-1}}^2\right],
 \\&\le  4\eta d^{11}H^3 T^{-\frac{2}{3}} \E\left[\sum_{t=1}^T\sum_{h=1}^{H-1}\sum_{\pi \in \Pi'} \tilp^t(\pi) \left\|\hphi_{h-1}(\pi) \right\|_{\left(\Sigma_{h-1}^t\right)^{-1}}^2\right],
 \\&\le  4\eta d^{11}H^4 T^{\frac{1}{3}}
 \\&\le d^{10}H^3T^{-\frac{1}{3}}
 \\&\le d^4H T^{\frac{2}{3}}. \tag{Assume $T \ge d^6H^2$}
\end{align*}

Putting these two changes back to the proof into \cref{lem:EXP EXP}, given $\eta = ({4Hd|\cA|})^{-1} T^{-\frac{2}{3}}$  we have
\begin{align}
     \textbf{EXP} &=  \frac{\log| \Pi_{\text{\rm lin}}^{\text{\rm cov}}(\frac{1}{T})|}{\eta} + \frac{6d\eta H^2|\cA|T}{\beta} +  4\eta d^2 H^4 \epsilon^2 T + d^4H T^{\frac{2}{3}} + 4d^6H^2 T^{\frac{2}{3}} \nonumber
     \\&= \wtilde{\order}\left(d^6|\cA|H^2 T^{\frac{2}{3}}\log(|\Phi|)\right) \label{eq:model-free EXP}
\end{align}
Putting \cref{eq:model-free error}, \cref{lem:model-free bias}, \cref{eq:model-free EXP} into \cref{eq:model-free decom} together with $\epsilon = 18^{-1} d^{\frac{5}{2}}T^{-\frac{1}{3}}$, we have
\begin{align*}
   \mathrm{Reg}_T \le  \wtilde{O}\left(d^{8} H^6 |\cA|  T^{\frac{2}{3}} \log(|\Phi|))\right) 
\end{align*}

\clearpage
\section{Proof of \cref{thm:oraclealg} (Model-Free, Banfit Feedback)}
\label{sec:oracleeff}

\label{sec:dup}
We start by introducting some notation. We let $\cI\ind{k}$ denote the rounds in the $k$-th epoch: 
\begin{align}
    \cI\ind{k} \coloneqq \{ T_0+(k-1)\cdot N_\reg+1,\ \dots,\ T_0+k \cdot N_\reg \}, \label{eq:interval0}
\end{align}
where $T_0$ is as in \cref{line:T00} of \cref{alg:oraceleff}.
Throughout the analysis, we condition on the event \begin{align}\cE\coloneqq \cE^\cov,
\label{eq:event0}
\end{align} where $\cE^{\cov}$ is as in \cref{lem:vox}. Further, for any $k\in[K]$, let $\rho\ind{k}$ be the distribution of the random policy: 
\begin{align}
 \mathbb{I}\{\bzeta = 0\} \cdot  \pihat\ind{k}+ \mathbb{I}\{\bzeta = 1\} \cdot \bpi\circ_{\bh}   \pi_\unif \circ_{\bh+1} \pihat\ind{k}, \label{eq:this}
 \end{align}
 with $\bzeta \sim \mathrm{Ber}(\nu)$, $\bh\sim \unif([H])$, and $\bpi \sim \unif(\Psi^{\texttt{cov}}_{\bh})$.
 %For any distribution $\rho\in \Delta(\Pi)$ and loss $\ell$, we use the notation \[V_1(\rho;\ell)=\E_{\bpi\sim \rho}\E[V^{\bpi}_1(\x_1;\ell)].\]

We start our analysis by applying the performance difference lemma.
\paragraph{Applying the performance difference lemma.} 
For any $k\in[K]$, $t\in \cI\ind{k}$, and $\rho\ind{k}$ as just defined, we have 
\begin{align}
%&V_1(\rho\ind{k};\ell^t)- V_1(\pi^\star;\ell^t) \nn \\
  & \E_{\bpi\sim \rho\ind{k}}\E\left[V_1^{\bpi}(x_1;\ell^t)\right] - V^{\pi^\star}_1(x_1;\ell^t)\nn \\
   & = (1-\nu) \cdot\left( V_1^{\pihat\ind{k}}(x_1;\ell^t)-V_1^{\pi^\star}(x_1;\ell^t)\right) \nn \\
   & \quad + \frac{\nu}{H d}\sum_{h\in[H]}\sum_{\pi \in \Psi^\cov_h} \left(V_1^{\pi \circ_{h} \pi_\unif \circ_{h+1} \pihat\ind{k}}(x_1;\ell^t)- V_1^{\pi^\star}(x_1;\ell^t)\right), \nn \\
   & \le (1-\nu) \cdot \sum_{h=1}^H \E^{\pi^\star}\left[\sum_{a\in {\cA}} \left(\pihat\ind{k}_h(a\mid \x_h)- \pi^\star_h(a\mid \x_h)  \right)\cdot Q_h^{\pihat\ind{k}}(\x_h,a; \ell^t)\right] + H \nu, \nn \\
& =  (1-\nu) \cdot \sum_{h=1}^H \E^{\pi^\star}\left[\sum_{a\in {\cA}} \left(\pihat\ind{k}_h(a\mid \x_h)- \pi^\star_h(a\mid \x_h)  \right)\cdot \Qhat_h\ind{k}(\x_h,a)\right] + H \nu \nn \\
& \quad + (1-\nu) \cdot \sum_{h=1}^H \E^{\pi^\star}\left[\sum_{a\in {\cA}} \pihat\ind{k}_h(a\mid \x_h)\cdot  \left(Q_h^{\pihat\ind{k}}(\x_h,a;\ell^t) - \Qhat_h\ind{k}(\x_h,a) \right)\right] \nn \\
&  \quad + (1-\nu) \cdot \sum_{h=1}^H \E^{\pi^\star}\left[\sum_{a\in {\cA}} \pistar_h(a\mid \x_h)\cdot \left( \Qhat_h\ind{k}(\x_h,a)- Q_h^{\pihat\ind{k}}(\x_h,a;\ell^t)  \right) \right].
\nn \end{align}
Thus, by the triangle inequality, we have 
\begin{align}
  &  \sum_{t\in \cI\ind{k}} \left(\E_{\bpi\sim \rho\ind{k}}\E\left[V_1^{\bpi}(x_1;\ell^t)\right] - V^{\pi^\star}_1(x_1;\ell^t)\right) \nn  \\
& \leq (1-\nu) \cdot \sum_{t\in \cI\ind{k}} \sum_{h=1}^H \E^{\pi^\star}\left[\sum_{a\in {\cA}} \left(\pihat\ind{k}_h(a\mid \x_h)- \pi^\star_h(a\mid \x_h)  \right)\cdot \Qhat_h\ind{k}(\x_h,a)\right] + H N_\reg \nu,\nn 
  \\
& \quad + 2(1-\nu) \cdot \sum_{h=1}^H\max_{\pi'\in \Pi} \left|\sum_{t\in \cI\ind{k}} \E^{\pi^\star}\left[\sum_{a\in {\cA}} \pi'_h(a\mid \x_h) \cdot \left(Q_h^{\pihat\ind{k}}(\x_h,a) - \Qhat_h\ind{k}(\x_h,a) \right)\right]\right|. \label{eq:postperform0}
\end{align}

We start by bound the first term in the right-hand side of \eqref{eq:postperform0} (the regret term).

\subsection{Bounding the Regret Term}
Fix $h\in[H]$ and define $\tilde\pi^\star_h(\cdot\mid x) \coloneqq (1-\gamma/T)\cdot \pistar_h(\cdot\mid x) + \gamma \pi_\unif(\cdot \mid x)/T$ for all $x$, where $\gamma$ is as in \cref{alg:oraceleff}. We have that $|\Qhat\ind{k}_h(x, a)|=|\phihat\ind{k}_h(x,a)^\top \thetahat\ind{k}_h| \leq  H\sqrt{d}$ (since $\thetahat\ind{k}_h\in \bbB_{d}(H \sqrt{d})$ and $\|\phihat\indd{k}_h(x,a)\|\leq 1$).  By applying \cref{lem:EXP bound}, we have that for any $\eta \le \frac{1}{H\sqrt{d}}$ and $x\in \cX$:
\begin{align}
   & \sum_{k\in[K]} \sum_{t\in\cI\ind{k}} \sum_{a\in {\cA}} \left(\pihat\ind{k}_h(a\mid x)- \pi^\star_h(a\mid x)  \right)\cdot \Qhat_h\ind{k}(x,a) \nn \\ 
&\leq  \sum_{k\in[K]} \sum_{t\in\cI\ind{k}} \sum_{a\in {\cA}} \left(\pihat\ind{k}_h(a\mid x)- \tilde\pi^\star_h(a\mid x)  \right)\cdot \Qhat_h\ind{k}(x,a)+ H \gamma, \nn \\ 
& \leq \frac{N_\reg \log (T/\gamma)}{\eta} + \eta \sum_{k=1}^K\sum_{t\in\cI\ind{k}}\sum_{a\in\cA}   \pihat_h\ind{k}(a\mid x)\cdot  \Qhat\ind{k}_h(x, a)^2 + H \gamma,\nn \\
&\leq \frac{N_\reg\log (T/\gamma)}{\eta} +  H^2 d \eta T  + H \gamma,
\label{eq:carr0} 
\end{align}

\subsection{Bounding the Bias Term}
 Fix epoch $k \in[K]$, round $t\in\cI\ind{k}$, layer $h \in[2 \ldotst H]$, and $\pi'\in \Pi$. Further, let \begin{align}\label{eq:reachpre}\cX_{h,\veps}\coloneqq \left\{x\in \cX : \max_{\pi\in \Pi} d^{\pi}_h(x) \geq \veps \cdot \|\mu^\star_h(x)\| \right\}
    \end{align} be the set of \emph{$\veps$-reachable} states. With this notation, we now bound the bias term in \eqref{eq:postperform0}; we have
    \begin{align}
    &\left|\sum_{t\in \cI\ind{k}}\E^{\pistar}\left[\sum_{a\in\cA} \pi'_h(a\mid \x_h)\cdot\left(Q^{\pihat\ind{k}}_h(\x_h,a;\ell^t)-\Qhat\ind{k}_h(\x_h,a)\right)\right]\right| \nn \\
    &\leq \left|\sum_{t\in \cI\ind{k}}\E^{\pistar}\left[\mathbb{I}\{\x_h \not\in \cX_{h,\veps}\} \sum_{a\in\cA} \pi'_h(a\mid \x_h)\cdot\left(Q^{\pihat\ind{k}}_h(\x_h,a;\ell^t)-\Qhat\ind{k}_h(\x_h,a)\right)\right]\right| \nn \\
    & \quad + \left|\sum_{t\in \cI\ind{k}}\E^{\pistar}\left[\mathbb{I}\{\x_h \in \cX_{h,\veps}\} \sum_{a\in\cA} \pi'_h(a\mid \x_h)\cdot\left(Q^{\pihat\ind{k}}_h(\x_h,a;\ell^t)-\Qhat\ind{k}_h(\x_h,a)\right)\right]\right|,\nn \\
    \shortintertext{and so by \cref{lem:rideoff},}
    &\leq \left|\sum_{t\in \cI\ind{k}}\E^{\pistar}\left[\mathbb{I}\{\x_h \in \cX_{h,\veps}\} \sum_{a\in\cA} \pi'_h(a\mid \x_h)\cdot\left(Q^{\pihat\ind{k}}_h(\x_h,a;\ell^t)-\Qhat\ind{k}_h(\x_h,a)\right)\right]\right| + 2 N_\reg H d^2 \veps,\nn \\
    & = N_\reg \cdot \left|\E^{\pistar}\left[\mathbb{I}\{\x_h \in \cX_{h,\veps}\} \sum_{a\in\cA} \pi'_h(a\mid \x_h)\cdot\left(\frac{1}{N_\reg}\sum_{t\in \cI\ind{k}} Q^{\pihat\ind{k}}_h(\x_h,a;\ell^t) - \Qhat\ind{k}_h(\x_h,a)\right)\right]\right| + 2 N_\reg H d^2 \veps.\nn 
    \end{align}
    Thus, by letting 
    \begin{align}
        \Qbar_h^{\pihat\ind{k}}(\cdot,\cdot) \coloneqq \frac{1}{N_\reg} \sum_{t\in\cI\ind{k}} Q^{\pihat\ind{k}}_h(\cdot,\cdot;\ell^t), \label{eq:theQbar}
    \end{align}
    and using Jensen's inequality (twice), we get
    \begin{align}
        &\left|\sum_{t\in \cI\ind{k}}\E^{\pistar}\left[\sum_{a\in\cA} \pi'_h(a\mid \x_h)\cdot\left(Q^{\pihat\ind{k}}_h(\x_h,a;\ell^t)-\Qhat\ind{k}_h(\x_h,a)\right)\right]\right| \nn \\
        &\leq  N_\reg \cdot \E^{\pistar}\left[\mathbb{I}\{\x_h \in \cX_{h,\veps}\} \sum_{a\in\cA} \pi'_h(a\mid \x_h)\cdot\left|\Qbar^{\pihat\ind{k}}_h(\x_h,a)-\Qhat\ind{k}_h(\x_h,a)\right|\right]    + 2 N_\reg H d^2 \veps,\nn \\
        &\leq  N_\reg \cdot \sqrt{ \E^{\pistar}\left[\mathbb{I}\{\x_h \in \cX_{h,\veps}\} \sum_{a\in\cA} \pi'_h(a\mid \x_h)\cdot\left(\Qbar^{\pihat\ind{k}}_h(\x_h,a)-\Qhat\ind{k}_h(\x_h,a)\right)^2\right]}    + 2 N_\reg H d^2 \veps,\nn \\
        &\leq  N_\reg \cdot \sqrt{ \E^{\pistar}\left[\mathbb{I}\{\x_h \in \cX_{h,\veps}\}\cdot \max_{a\in\cA} \left(\Qbar^{\pihat\ind{k}}_h(\x_h,a)-\Qhat\ind{k}_h(\x_h,a)\right)^2\right]}    + 2 N_\reg H d^2 \veps,\nn \\
    \intertext{and now by the fact that $\Psi^\cov_{h}$ is an $(\frac{1}{8 Ad},\veps)$ policy cover for layer $h$ (see \cref{lem:vox} and \cref{def:cover}):}
    &\leq N_\reg \cdot \sqrt{8 A d \max_{\pi\in \Psi^{\cov}_{h}}\E^{\pi}\left[\mathbb{I}\{\x_h \in \cX_{h,\veps}\}\cdot \max_{a\in\cA} \left(\Qbar_h^{\pihat\ind{k}}(\x_h,a)-\Qhat\ind{k}_h(\x_h,a)\right)^2\right]} + 2 N_\reg H d^2 \veps,\nn\\
    &\leq N_\reg \cdot \sqrt{8A^2 d  \sum_{\pi\in \Psi^{\cov}_{h}}\E^{\pi\circ_h\pi_\unif}\left[ \left(\Qbar_h^{\pihat\ind{k}}(\x_h,\a_h)-\Qhat\ind{k}_h(\x_h,\a_h)\right)^2\right]} + 2 N_\reg H d^2 \veps. \label{eq:trad0}
    \end{align}
Next, we bound the regression error term in the right-hand side of \eqref{eq:trad0}.

\subsection{Bounding the Regression Error}
\begin{lemma}
\label{lem:regevent}
Let $\delta\in(0,1)$ and $\nu\in(0,1/4)$ be given. There is an event $\cE^\reg$ of probability at least $1-2\delta$ under which
\begin{align}
\sum_{\pi \in \Psi_h}  \E^{\pi \circ_h \pi_\texttt{\em unif}}\left[\left( \phihat\ind{k}_h(\x_h,\a_h)^\top \hat\theta\ind{k}_h - \overline{Q}_h^{\pihat\ind{k}}(\x_h,\a_h)\right)^2\right]& \leq \veps_\reg^2 \coloneqq \frac{40 H^3 d \log (2 |\cF|/\delta)}{\nu N_\reg}. \label{eq:est0}
\end{align}
\end{lemma}
  
  \begin{proof} 
        Fix $\delta\in(0,1)$, $h\in[H]$ and $k\in[K]$, and let $(\x^t_h, \a^t_h,\bm{\zeta}^t,\bm{h}^t)$ be as in \cref{alg:oraceleff}. With this, define 
        \begin{align}
        \bm{I}^t_h \coloneqq \mathbb{I}\{\bzeta^t = 0 \text{ or }  \bm{h}^t \leq h\},
        \end{align}
        and note that $(\x^t_h, \a^t_h,\bm{I}^t_h)_{t\in \cI\ind{k}}$ are identically and independently distributed.
        Further, for $t\in \cI\ind{k}$ and $\pi\in \Pi$, let $\theta\indd{t,\pi}_h \in \bbB_d(H\sqrt{d})$ be such that 
        \begin{align}
          \forall (x,a)\in \cX\times \cA,\quad  Q^{\pi}_h(x,a;\ell^t) = \phi^\star_h(x,a)^\top \theta^{\pi,t}_h.\nn 
        \end{align}
        Such a $\theta^{\pi, t}_h$ is guaranteed to exist by the low-rank MDP structure (\cref{assm:normalizing}) and \cref{assm:linearlossweak}. With this, note that for $\Qbar_h^{\pihat\ind{k}}$ as in \eqref{eq:theQbar}, we have
        \begin{align}
            \Qbar_h^{\pihat\ind{k}}(x,a) = \phistar_h(x,a)^\top \theta\ind{k}_h,\quad \text{where} \quad \theta\ind{k}_h \coloneqq \frac{1}{N_\reg}\sum_{t\in \cI\ind{k}}\theta\indd{\pihat\ind{k},t}_h. \label{eq:avgtheta}
        \end{align}
    For the rest of this proof, we let $\cF$ be the function class \[\cF \coloneqq \{f: (x,a)\mapsto  \phi_h(x,a)^\top \theta\mid  \theta \in \cC\cup \{\theta\ind{k}_h\}, \phi\in\Phi\},\] where $\cC$ is a minimal $(N_\reg)^{-1}$-cover of $\mathbb{B}(H\sqrt{d})$ in $\|\cdot\|$ distance. Further, for $\tau\in \cI\ind{k}$ we let 
\begin{gather}
\z^\tau_h \coloneqq \sum_{l=h}^H \ell^\tau_l(\x^\tau_l, \a^\tau_l); \nn \\
\bm{\veps}^\tau_h \coloneqq  \sum_{l=h}^H\ell^\tau_l(\x_l^\tau,\a_l^\tau) - \frac{1}{N_\reg}\sum_{t\in \cI\ind{k}} Q^{\pihat\ind{k}}_h(\x_h^\tau,\a_h^\tau;\ell^t);\label{eq:bveps} 
%\\\bm{b}^s_h\coloneqq  \frac{1}{N_\reg}\sum_{t\in \cI\ind{k}}\left(\E_{\bpi\sim \rho\ind{k}}[Q^{\bpi}_h(\x_h^s,\a_h^;\ell^t)]-Q_h^{\pihat\ind{k}}(\x_h^t,\a_h^t;\ell^t)  \right),  
\shortintertext{and}
\Lhat(f) \coloneqq \sum_{t\in \cI\ind{k}} \bm{I}^t_h\cdot (f(\x^t_h,\a^t_h)- \z^t_h)^2,
\end{gather}
    for $f\in \cF$. Finally, let $f_\star(x,a)\coloneqq \phistar_h(x,a)^\top \theta\ind{k}_h$, where $\theta\ind{k}_h$ is as in \eqref{eq:avgtheta} and $\fhat(x,a)\coloneqq \Qhat\ind{k}_h(x,a)$. With this, note that $f_\star$ and $\fhat$ satisfy $f_\star(\x_h^t,\a_h^t) = \z^t_h - \bm{\veps}_h^t$ and $\fhat \in \argmin_{f\in \cF}\Lhat(f)$.
    
    Now, since $\fhat \in \argmin_{f\in \cF}\Lhat(f)$, we have
		\begin{align}
		0 \geq \Lhat(\fhat)- \Lhat(f_\star) = \nabla \Lhat(f_\star)[\fhat - f_\star] + \|\fhat - f_\star\|^2, \label{eq:toarrange}
		\end{align}
  where $\nabla$ denotes directional derivative and \[\|\fhat - f_\star\|^2 \coloneqq \sum_{t\in \cI\ind{k}}\bm{I}^t_h\cdot (\fhat(\x^t_h,\a^t_h)- f_\star(\x^t_h,\a^t_h))^2.\] Rearranging \eqref{eq:toarrange} and using that $f_\star(\x_h^t,\a_h^t) = \z^t_h - \bm{\veps}_h^t$, we get that 
	\begin{align}
	\|\fhat - f_\star\|^2 &\leq -  \nabla \Lhat(f_\star)[\fhat - f_\star] , \nn \\
	& = 2 \sum_{t\in \cI\ind{k}} \bm{I}^t_h\cdot(\z^t_h - f_\star(\x^t_h,\a^t_h)) (\fhat(\x^t_h,\a^t_h)- f_\star(\x^t_h,\a^t_h)), \nn \\
	& = 2\sum_{t\in \cI\ind{k}}\bm{I}^t_h\cdot \bm{\veps}^t_h \cdot (\fhat(\x^t_h,\a^t_h) - f_\star(\x^t_h,\a^t_h)).\label{eq:wrond2}
   \end{align}
We now bound the right-hand side of \eqref{eq:wrond2}. For any $h \in [H]$, $k \in [K]$ and any $\fhat, f_\star \in \mathcal{F}$, we apply \cref{lem:freedhelp} with
\begin{itemize}
\item  $\mathfrak{F}^i=\sigma(\{(\x^{j}_h,\a^{j}_h,\bm{\veps}_h^{j}, \bzeta^{j}): j\leq t_i\})$ where $t_i \coloneqq (k-1)\cdot N_\reg + i$;

\item The random variable $\bm{w}^i$ set as the difference 
\begin{align*}
    \w^i &=\mathbb{I}\{\bzeta^{t_i} =0\text{ or } \bm{h}^{t_i} \le h\}\cdot \bm{\veps}^{t_i}_h \cdot (\fhat(\x^{t_i}_h,\a^{t_i}_h)- f_\star(\x^{t_i}_h,\a^{t_i}_h)) 
    \\& \quad - \E\left[\mathbb{I}\{\bzeta^{t_i} =0\text{ or } \bm{h}^{t_i} \le h\}\cdot \bm{\veps}^{t_i}_h \cdot (\fhat(\x^{t_i}_h,\a^{t_i}_h)- f_\star(\x^{t_i}_h,\a^{t_i}_h))\mid \mathfrak{F}^{i-1}\right]
\end{align*}
where $t_i \coloneqq (k-1)\cdot N_\reg + i$;

% \HL{It seems in \cref{lem:freedhelp}, we have  $\E[\w^i\mid \mathfrak{F}^{i-1}]=0$, should the $\w^i$ here be $\w^i - \E[\w^i\mid \mathfrak{F}^{i-1}]$? I am also a bit confused by the $\min$ in $j \coloneqq \min \cI\ind{k}-1$.}
\item $n = N_\reg = |\cI\ind{k}|$.
\item $R = 4 H^2$; and
\item $\lambda = 1/(16 H^2)$;
\end{itemize} to get that there is an event $\cE$ of porbability at least $1-\delta$ under which 
\begin{align}
&\sum_{t\in \cI\ind{k}} \bm{I}^t_h \cdot \bm{\veps}^t_h \cdot (\fhat(\x^t_h,\a^t_h)- f_\star(\x^t_h,\a^t_h)) \nn \\
&\leq \sum_{t\in \cI\ind{k}} \E_t[ \bm{I}^t_h \cdot\bm{\veps}^t_h \cdot (\fhat(\x^t_h,\a^t_h)- f_\star(\x^t_h,\a^t_h))] \nn\\
&\quad + \frac{1}{8 H^2}  \sum_{t\in \cI\ind{k}} \E_t[\bm{I}^t_h \cdot (\bm{\veps}_h^t)^2 \cdot (\fhat(\x^t_h,\a^t_h)- f_\star(\x^t_h,\a^t_h))^2] + 16 H^2 \log (|\cF|/\delta), \nn \\
& \leq  \sum_{t\in \cI\ind{k}} \E_t[ \bm{I}^t_h \cdot\bm{\veps}^t_h \cdot (\fhat(\x^t_h,\a^t_h)- f_\star(\x^t_h,\a^t_h))]\nn \\
& \quad + \frac{1}{4} \sum_{t\in \cI\ind{k}}\E_t[\bm{I}^t_h \cdot(\fhat(\x^t_h,\a^t_h)- f_\star(\x^t_h,\a^t_h))^2] + 16 H^2 \log (|\cF|/\delta), \label{eq:sense}
\end{align}
where $\E_t[\cdot]$ is defined as $\E\left[\cdot \mid \mathfrak{F}^{t-1} \right]$ and the last step uses that $|\bm\veps^t_h|\leq 2 H$, for all $t\in \cI\ind{k}$.
For the rest of the proof, we condition on $\cE$ and to simplify notation let \begin{align}\bm{\Delta}^t_h \coloneqq \fhat(\x^t_h,\a^t_h)- f_\star(\x^t_h,\a^t_h).
\end{align}
Using the expression of $\bm{\veps}^t_h$ in \eqref{eq:bveps}, we have 
\begin{align}
 &\sum_{t\in \cI\ind{k}} \E_t[\bm{I}^t_h\cdot \bm{\veps}^t_h \cdot \bm{\Delta}^t_h] \nn \\
 &= \sum_{t\in \cI\ind{k}} \E_t\left[\bm{I}^t_h\cdot\left(\sum_{l=h}^H\ell^t_l(\x_l^t,\a_l^t) - \frac{1}{N_\reg}\sum_{\tau \in \cI\ind{k}}Q^{\pihat\ind{k}}_h(\x_h^t,\a_h^t;\ell^\tau) \right) \cdot \bm{\Delta}^t_h\right].\label{eq:prob}
 \end{align}
 Now, since $\bm{I}_h^t= \mathbb{I}\{\bzeta^t=0 \text{ or } \bm{h}^t \leq h\}$, we have \[\E_t\left[\sum_{l=h}^H\ell^t_l(\x_l^t,\a_l^t)\mid \x_h^t,\a^t_h, \bm{I}^t_h=1\right]=Q^{\pihat\ind{k}}_h(\x^t_h,\a^t_h;\ell^t).\] Plugging this into \eqref{eq:prob} and using the law of total expectation, we have
 \begin{align}
  & \sum_{t\in \cI\ind{k}} \E_t[\bm{I}^t_h\cdot \bm{\veps}^t_h \cdot \bm{\Delta}^t_h] \nn \\
    &= \sum_{t\in \cI\ind{k}} \E_t\left[\bm{I}^t_h\cdot\left(Q^{\pihat\ind{k}}_h(\x^t_h,\a^t_h;\ell^t) - \frac{1}{N_\reg}\sum_{\tau\in \cI\ind{k}}Q^{\pihat\ind{k}}_h(\x^t_h,\a^t_h;\ell^\tau) \right) \cdot \bm{\Delta}^t_h\right], \nn\\
&  =\sum_{t\in \cI\ind{k}} \E_t\left[\bm{I}^t_h\cdot Q^{\pihat\ind{k}}_h(\x^t_h,\a^t_h;\ell^t) \cdot \bm{\Delta}^t_h\right] - \frac{1}{N_\reg}\sum_{\tau\in \cI\ind{k}} \sum_{t\in \cI\ind{k}} \E_t\left[ \bm{I}^t_h\cdot Q^{\pihat\ind{k}}_h(\x^t_h,\a^t_h;\ell^\tau) \cdot \bm{\Delta}^t_h\right].
        \label{eq:priorpush}
        \end{align}
On the other hand, since $(\x^t_h,\a_h^t, \bm{I}^t_h)_{t\in \cI\ind{k}}$ are i.i.d. and $\ell^t$ is chosen by an oblivious adversary, we have 
\begin{align}
\forall t\in \cI\ind{k},\quad \E_t\left[\bm{I}^t_h\cdot Q^{\pihat\ind{k}}_h(\x^t_h,\a^t_h;\ell^t) \cdot \bm{\Delta}^t_h\right]
 &= \frac{1}{N_\reg}\sum_{\tau\in \cI\ind{k}} \E_{\tau}\left[\bm{I}^\tau_h\cdot Q^{\pihat\ind{k}}_h(\x^\tau_h,\a^\tau_h;\ell^t) \cdot \bm{\Delta}^\tau_h\right].\nn
\end{align}
Plugging this into \eqref{eq:priorpush} shows that
\begin{align}
\sum_{t\in \cI\ind{k}} \E_t[\bm{I}^t_h\cdot \bm{\veps}^t_h \cdot \bm{\Delta}^t_h] =0.
\end{align}
Combining this with \eqref{eq:sense} and \eqref{eq:wrond2}, we get that 
\begin{align}
    \sum_{t\in\cI\ind{k}} \bm{I}^t_h \cdot(\fhat(\x^t_h,\a^t_h)- f_\star(\x^t_h,\a^t_h))^2 &\leq \frac{1}{4} \sum_{t\in \cI\ind{k}}\E_t[\bm{I}^t_h \cdot(\fhat(\x^t_h,\a^t_h)- f_\star(\x^t_h,\a^t_h))^2] \nn \\
    & \quad + 16 H^2 \log (|\cF|/\delta). \label{eq:thisonehere}
\end{align}
Now, since $(\x_h^t,\a_h^t,\bm{I}^t_h)_{t\in \cI\ind{k}}$ are i.i.d., we have by \cref{lem:corbern} that there is an event $\cE'$ of probability at least $1-\delta$ under which we have 
\begin{align}
 \sum_{t\in \cI\ind{k}}\E_t\left[\bm{I}^t_h \cdot (\fhat(\x^t_h,\a^t_h)- f_\star(\x^t_h,\a^t_h))^2\right] &\leq 2  \sum_{t\in\cI\ind{k}} \bm{I}_h^t \cdot (\fhat(\x^t_h,\a^t_h)- f_\star(\x^t_h,\a^t_h))^2\nn \\
  & \quad + 8 H^2 \log(2|\cF|/\delta).\label{eq:tocombined}
\end{align}
Combining this with \eqref{eq:thisonehere} and rearranging, we get that under $\cE^\reg \coloneqq \cE \cap \cE'$:
\begin{align}
\frac{1}{N_\reg} \sum_{t\in\cI\ind{k}}\E_t\left[\bm{I}^t_h\cdot (\fhat(\x^t_h,\a^t_h)- f_\star(\x^t_h,\a^t_h))^2\right] &\leq \frac{40 H^2 \log (2 |\cF|/\delta)}{N_\reg}.\label{eq:mile} 
\end{align}
% \HL{For \eqref{eq:mile}, should we also divide $N_\reg$ for the LHS?}

On the other hand, we have that for all $t\in \cI\ind{k}$:
\begin{align}
\E_t\left[\bm{I}^t_h \cdot (\fhat(\x^t_h,\a^t_h)- f_\star(\x^t_h,\a^t_h))^2\right] & \geq \E_t\left[\mathbb{I}\{\bzeta^t = 1, \bm{h}^t = h\}\cdot (\fhat(\x^t_h,\a^t_h)- f_\star(\x^t_h,\a^t_h))^2\right], \nn \\
& = \frac{\nu}{H d} \sum_{\pi \in \Psi_h}  \E^{\pi \circ_h \pi_\unif}\left[(\fhat(\x^t_h,\a^t_h)- f_\star(\x^t_h,\a^t_h))^2\right].
\end{align}
Plugging this into \eqref{eq:mile} and using the expressions of $\fhat$ and $f_\star$, we get
\begin{align}
\sum_{\pi \in \Psi_h}  \E^{\pi \circ_h \pi_\unif}\left[\left( \phihat\ind{k}_h(\x_h,\a_h)^\top \hat\theta\ind{k}_h - \overline{Q}_h^{\pihat\ind{k}}(\x_h,\a_h)\right)^2\right]& \leq \frac{40 H^3 d \log (2 |\cF|/\delta)}{\nu N_\reg}.\nn 
\end{align}
\end{proof}

\subsection{Putting It All Together}

By combining \eqref{eq:postperform0}, \eqref{eq:trad0}, \eqref{eq:est0}, and \eqref{eq:carr0}, we have that under the event $\cE\coloneqq \cE^\cov \cap \cE^\reg$ (where $\cE^\cov$ and $\cE^\reg$ are as in \cref{lem:vox} and \cref{lem:regevent}, respectively):
\begin{align}
   \mathrm{Reg}_T &= HT_0+ \sum_{k\in[K]} \sum_{t\in \cI\ind{k}} \left(\E_{\bpi\sim \rho\ind{k}}\E\left[V_1^{\bpi}(x_1;\ell^t)\right] - V^{\pi^\star}_1(x_1;\ell^t)\right), \nn \\
   & \leq HT_0+HT\nu +T \cdot \sqrt{{8A^2 d} \cdot \veps^2_\reg } +\frac{N_\reg\log (T/\gamma)}{\eta} +  H^2 d \eta T  + H \gamma.
\end{align}
Thus, plugging in the expression of $T_0$ from \cref{alg:oraceleff}, and ignoring polynormal factors $d, A, H,\log(|\Phi|\veps^{-1}\delta^{-1})$, we get that 
\begin{align}
\mathrm{Reg}_T &\prec \frac{1}{\veps^2} + T \veps +T \nu + N_\reg \cdot \sqrt{{8A^2 d} \cdot \veps_\reg^2 } + 2 N_\reg H d \veps + \frac{N_\reg}{\eta} + \eta T, \nn \\
& \prec T^{2/3} +\nu T+ T \cdot \sqrt{8A^2d\cdot \veps_\reg^2} + \sqrt{TN_\reg}, \quad \text{(by setting $\veps = T^{-1/3}$ and $\eta = (N_\reg/T)^{1/2}$)}\nn \\
& = T^{2/3} +\nu T+ T \cdot \sqrt{\frac{1}{\nu N_\reg}} + \sqrt{TN_\reg},  \quad (\text{used the expression of $\veps^2_\reg$ in \eqref{eq:est0}}) \nn \\
& \prec T^{2/3} +\nu T+ \sqrt{T} \cdot \left(\frac{T}{\nu }\right)^{1/4} ,  \quad \text{(by setting $N_\reg = (T/\nu)^{1/2}$)} \nn \\
& \prec T^{4/5}, 
\end{align}
where the last step follows by setting $\nu = T^{-1/5}$.

% \cwcomment{The $\sqrt{\nu}T$ term might be improved to $\nu T$.  This $\sqrt{\nu}T$ comes from the bias in regression. Can we change \cref{eq:regproblem0}  to 

% \begin{align}
% 	(\phihat\ind{k}_h,\hat{\theta}\ind{k}_h) &\gets \argmin_{(\phi,\theta)\in \Phi\times \bbB_{d}(H\sqrt{d})}  \sum_{t\in \cI\ind{k}} \left(\phi_h(\x\indd{t}_h,\a\indd{t}_h)^\top \theta -\sum_{s=h}^H \bell\indd{t}_s\right)^2 I^t_h 
%  \end{align} 
%  where 
% \begin{align*}
%     I^t_h=\mathbb{I}\{\bzeta\indd{t}=0 \text{\ or\ } (\bzeta\indd{t}=1 \text{\ and\ } \bh\indd{t}\leq h)\}
% \end{align*}
 
%  }

%\clearpage
%\input{neurips/appendix_modelfree_adaptive}
\clearpage
\section{Proof of \cref{thm:adaptive} (Model-Free, Bandit Feedback, Adaptive Adversary)}
We let $\cI\ind{k}$ denote the rounds in the $k$th epoch: 
\begin{align}
    \cI\ind{k} \coloneqq \{ T_0+(k-1)\cdot N_\reg+1,\ \dots,\ T_0+k \cdot N_\reg \}, \label{eq:interval}
\end{align}
where $T_0$ is as in \cref{line:T0} of \cref{alg:algorithm_name}.
Throughout the analysis, we condition on the event \begin{align}\cE\coloneqq \cE^\cov \cap \cE^{\rep+\spanner}\cap \cE^{\freed},
\label{eq:event}
\end{align} where $\cE^{\cov}$, $\cE^{\rep+\spanner}$, and $\cE^{\freed}$ are as in \cref{lem:vox}, \cref{cor:repspanner}, and \cref{lem:freed}, respectively. 

We start our analysis by applying the performance difference lemma.
\paragraph{Applying the performance difference lemma.}
For any $k\in[K]$, $t\in \cI\ind{k}$, and let $\rho\ind{k}$ be the distribution of the random policy: 
\begin{align}
\mathbb{I}\{\bzeta\indd{t} = 0\} \cdot  \pihat\ind{k}+ \mathbb{I}\{\bzeta\indd{t} = 1\} \cdot \bpi\indd{t}\circ_{\bh\indd{t}+1}    \pihat\ind{k},
 \end{align}
 with $\bzeta\indd{t} \sim \mathrm{Ber}(\nu)$, $\bh\indd{t}\sim \unif([H])$, and $\bpi\indd{t} \sim \unif(\Psi^{\texttt{span}}_{\bh\indd{t}})$.
\begin{align}
 & \E_{\bpi\sim \rho\ind{k}}\E\left[V_1^{\bpi}(x_1;\bcH^{t-1}) \right] - V^{\pi^\star}_1(x_1;\bcH^{t-1})\nn \\
   & = (1-\nu) \cdot\left( V_1^{\pihat\ind{k}}(x_1;\bcH^{t-1})-V_1^{\pi^\star}(x_1;\bcH^{t-1})\right)  \nn \\
   & \quad + \frac{\nu}{H d}\sum_{h\in[H]}\sum_{\pi \in \Psi^\cov_h}\left( V_1^{\pi \circ_{h+1}  \pihat\ind{k} }(x_1;\bcH^{t-1})- V_1^{\pi^\star}(\x_1;\bcH^{t-1})\right), \nn \\
   & = (1-\nu) \cdot \sum_{h=1}^H \E^{\pi^\star}\left[\sum_{a\in {\cA}} \left(\pihat\ind{k}_h(a\mid \x_h)- \pi^\star_h(a\mid \x_h)  \right)\cdot Q_h^{\pihat\ind{k}}(\x_h,a;\bcH\indd{t-1})\mid \bcH\indd{t-1}\right] + HN_\reg \nu, \nn \\
& =  (1-\nu) \cdot \sum_{h=1}^H \E^{\pi^\star}\left[\sum_{a\in {\cA}} \left(\pihat\ind{k}_h(a\mid \x_h)- \pi^\star_h(a\mid \x_h)  \right)\cdot \Qhat_h\ind{k}(\x_h,a)\mid \bcH\indd{t-1}\right] + H N_\reg \nu \nn \\
& \quad + (1-\nu) \cdot \sum_{h=1}^H \E^{\pi^\star}\left[\sum_{a\in {\cA}} \pihat\ind{k}_h(a\mid \x_h)\cdot  \left(Q_h^{\pihat\ind{k}}(\x_h,a;\bcH\indd{t-1}) - \Qhat_h\ind{k}(\x_h,a) \right)\mid \bcH\indd{t-1}\right] \nn \\
&  \quad + (1-\nu) \cdot \sum_{h=1}^H \E^{\pi^\star}\left[\sum_{a\in {\cA}} \pistar_h(a\mid \x_h)\cdot \left( \Qhat_h\ind{k}(\x_h,a)- Q_h^{\pihat\ind{k}}(\x_h,a;\bcH\indd{t-1})  \right) \mid \bcH\indd{t-1}\right].
\nn \end{align}
Thus, by the triangle inequality, we have 
\begin{align}
  &  \sum_{t\in \cI\ind{k}} \left(\E_{\bpi\sim \rho\ind{k}}\E\left[V_1^{\bpi}(x_1;\bcH^{t-1}) \right] - V^{\pi^\star}_1(x_1;\bcH^{t-1})\right) \nn  \\
& \leq (1-\nu) \cdot \sum_{t\in \cI\ind{k}} \sum_{h=1}^H \E^{\pi^\star}\left[\sum_{a\in {\cA}} \left(\pihat\ind{k}_h(a\mid \x_h)- \pi^\star_h(a\mid \x_h)  \right)\cdot \Qhat_h\ind{k}(\x_h,a)\mid \bcH\indd{t-1}\right] + H N_\reg \nu\nn  \\
& \quad + 2(1-\nu) \cdot \sum_{h=1}^H\max_{\pi'\in \Pi} \left|\sum_{t\in \cI\ind{k}} \E^{\pi^\star}\left[\sum_{a\in {\cA}} \pi'_h(a\mid \x_h) \cdot \left(Q_h^{\pihat\ind{k}}(\x_h,a;\bcH\indd{t-1}) - \Qhat_h\ind{k}(\x_h,a) \right)\mid \bcH\indd{t-1}\right]\right|. \label{eq:postperform}
\end{align}
We start by bounding the first term on the right-hand side of \eqref{eq:postperform} (the regret term).
\subsection{Bounding the Regret Term}
Fix $h\in[H]$ and define $\tilde\pi^\star_h(\cdot\mid x) \coloneqq (1-\gamma/T)\cdot \pistar_h(\cdot\mid x) + \gamma \pi_\unif(\cdot \mid x)/T$ for all $x$, where $\gamma$ is as in \cref{alg:algorithm_name}. We have that $|\Qhat\ind{k}_h(x, a)|=|\phibar^\rep_h(x,a)^\top \thetahat\ind{k}_h| \leq 8 Hd^2$ (since $\thetahat\ind{k}_h\in \bbB_{2d}(4 H d^2)$ and $\|\phibar^\rep_h(x,a)\|\leq \|\phi_h^\loss(x,a)\|+\|\phi^\rep_h(x,a)\|\leq 2$). By applying \ref{lem:EXP bound}, we have that for any $\eta \le \frac{1}{8Hd^2}$ and $x\in \cX$:
\begin{align}
   & \sum_{k\in[K]} \sum_{t\in\cI\ind{k}} \sum_{a\in {\cA}} \left(\pihat\ind{k}_h(a\mid x)- \pi^\star_h(a\mid x)  \right)\cdot \Qhat_h\ind{k}(x,a) \nn \\ 
&\leq  \sum_{k\in[K]} \sum_{t\in\cI\ind{k}} \sum_{a\in {\cA}} \left(\pihat\ind{k}_h(a\mid x)- \tilde\pi^\star_h(a\mid x)  \right)\cdot \Qhat_h\ind{k}(x,a)+ H \gamma, \nn \\ 
& \leq \frac{N_\reg \log (T/\gamma)}{\eta} + \eta \sum_{k=1}^K\sum_{t\in\cI\ind{k}}\sum_{a\in\cA}   \pihat_h\ind{k}(a\mid x)\cdot  \Qhat\ind{k}_h(x, a)^2 + H \gamma,\nn \\
&\leq \frac{N_\reg\log (T/\gamma)}{\eta} + 64 H^2 d^4 \eta T  + H \gamma,
\label{eq:carr} 
\end{align}
% where the last inequality follows by the fact that $|\Qhat\ind{k}_h(x, a)|=|\phibar^\rep_h(x,a)^\top \thetahat\ind{k}_h| \leq 8 Hd^2$ (since $\thetahat\ind{k}_h\in \bbB_{2d}(4 H d^2)$ and $\|\phibar^\rep_h(x,a)\|\leq \|\phi_h^\loss(x,a)\|+\|\phi^\rep_h(x,a)\|\leq 2$). 
\subsection{Bounding the Bias Term}
To bound the bias term (second term in \eqref{eq:postperform}), we make use of the following result.
\begin{lemma}
\label{lem:linearQ}
For $t\in [T]$, $h\in [H]$, and $\pi\in\Pi$, there exists $\thetabar_h\indd{t,\pi} \in \bbB_{2d}(H \sqrt{d})$ such that for all $(x,a)\in \cX\times \cA$ and history $\cH\indd{t-1}=(x_{1:H}\indd{1:t-1},a_{1:H}\indd{1:t-1})$,
\begin{align}
    Q^\pi_h(x,a;\cH\indd{t-1})= \phibar^\star_h(x,a)^\top \thetab\indd{t,\pi}_h, \quad \text{where} \quad \phibar^\star_h \coloneqq [\phi^\loss, \phistar]\in \reals^{2d}.  \label{eq:theQ}
\end{align}
\end{lemma}
\begin{proof}
    Fix $t\in [T]$, $h\in [H]$, and $\pi\in\Pi$.
By the low-rank MDP structure and the normalizing assumption on $\mu^\star_h$ in \cref{assm:normalizing}, there exists $w^{t,\pi}_{h+1}\in \bbB_d((H-s)\sqrt{d})$ such that for all $(x,a)\in \cX\times \cA$ and history $\cH\indd{t-1}=(x_{1:H}\indd{1:t-1},a_{1:H}\indd{1:t-1})$,
\begin{align}
    \E^\pi\left[\sum_{s=h+1}^H \ell_s(\x_s,\a_s;\cH\indd{t-1}) \mid \x_h = x, \a_h =a\right] = \phi^\star_h(x,a)^\top w^{t,\pi}_{h+1}. \label{eq:linearfor}
\end{align}
Now, with $\g \indd{t}_h$ as in \cref{assm:linearloss}, \eqref{eq:linearloss} and \eqref{eq:linearfor} implies that $\thetab^{t,\pi}_h \coloneqq [\g \indd{t}_h, w\indd{t,\pi}_{h+1}]\in \reals^{2d}$ (where $[\cdot, \cdot]$ denotes the vertical stacking of vectors) satisfies the desired property.
\end{proof}
    Fix epoch $k \in[K]$, round $t\in\cI\ind{k}$, layer $h \in[2 \ldotst H]$, and $\pi'\in \Pi$. By \cref{lem:linearQ}, there exists $\bm\thetabar\indd{t,\pihat\ind{k}}_h\in \bbB_{2d}(Hd^{1/2})$ such that for all $(x,a)\in \cX\times \cA$, \[Q^{\pihat\ind{k}}_{h+1}(x,a;\bcH\indd{t-1})=\phibar^\star_{h+1}(x,a)^\top \bm\thetabar^{t,\pihat\ind{k}}_{h+1}.\] With this, let $\bm{w}^{t,\pihat\ind{k}}_{h+1}\in \bbB_d(3Hd^2)$ be as in \cref{cor:repspanner} with $f(x)\coloneqq \frac{1}{Hd^{1/2}}\max_{a\in \cA}\phibar^\star_{h+1}(x,a)^\top \bm\thetabar^{t,\pihat\ind{k}}_{h+1}$; note that this function belongs to the function class $\cF_{h+1}$ in \cref{alg:algorithm_name}. With this, we define 
    \begin{align}
        \bar\varthetab\indd{t,\pihat\ind{k}}_h \coloneqq \left[\bm{g}^t_h, \bm{w}^{t,\pihat\ind{k}}_{h+1}\right]\in \bbB_{2d}(4Hd^{2}),
    \end{align}
    where $\bm{g}^t_h\in \bbB_d(1)$ is such that $\ell_h(x,a;\bcH\indd{t-1})=\phi_h^\loss(x,a)^\top \bm{g}_h\indd{t}$. 
    
  With this notation, we now bound the bias term in \cref{eq:postperform}: we have
    \begin{align}
    &\left|\sum_{t\in \cI\ind{k}}\E^{\pistar}\left[\sum_{a\in\cA} \pi'_h(a\mid \x_h)\cdot\left(Q^{\pihat\ind{k}}_h(\x_h,a;\bcH\indd{t-1})-\Qhat\ind{k}_h(\x_h,a)\right)\mid \bcH\indd{t-1}\right]\right| \nn \\
    &\leq \left|\sum_{t\in \cI\ind{k}}\E^{\pistar}\left[\sum_{a\in\cA} \pi'_h(a\mid \x_h)\cdot\left(\phibar^\rep_h(\x_h,a)^\top  \bvthetab\indd{t,\pihat\ind{k}}_h-\Qhat\ind{k}_h(\x_h,a)\right)\mid \bcH\indd{t-1}\right]\right| \nn \\
    & \quad + \left|\sum_{t\in \cI\ind{k}}\E^{\pistar}\left[\sum_{a\in\cA} \pi'_h(a\mid \x_h)\cdot\left(Q^{\pihat\ind{k}}_h(\x_h,a;\bcH\indd{t-1}) - \phibar^\rep_h(\x_h,a)^\top \bvthetab\indd{t,\pihat\ind{k}}_h  \right)\mid \bcH\indd{t-1}\right]\right|,\nn \\
    \shortintertext{and by \cref{cor:repspanner} (in particular \eqref{eq:replearnee})}
    &\leq \left|\sum_{t\in \cI\ind{k}}\E^{\pistar}\left[\sum_{a\in\cA} \pi'_h(a\mid \x_h)\cdot\left(\phibar^\rep_h(\x_h,a)^\top  \bvthetab\indd{t,\pihat\ind{k}}_h-\phibar^\rep_h(\x_h,a)^\top \thetahat_h\ind{k} \right)\mid \bcH\indd{t-1}\right]\right| + N_\reg \cdot \veps_\rep,\nn \\
    \intertext{and by \cref{cor:repspanner} again (in particular \eqref{eq:spannernew}) and the triangle inequality}
    &\leq 2 \sum_{\pi \in \Psi_h^\spanner}\left|\sum_{t\in \cI\ind{k}}\E^{\pi}\left[\phibar^\rep_h(\x_h,\a_h)^\top  \bvthetab\indd{t,\pihat\ind{k}}_h-\phibar^\rep_h(\x_h,\a_h)^\top \thetahat_h\ind{k}\mid \bcH\indd{t-1}\right]\right| + N_\reg \cdot \veps_{\spanner} + N_\reg \cdot \veps_\rep,\nn\\
    &\leq 2 \sum_{\pi \in \Psi_h^\spanner}\left|\sum_{t\in \cI\ind{k}}\E^{\pi}\left[Q^{\pihat\ind{k}}_h(\x_h,a;\bcH\indd{t-1})-\phibar^\rep_h(\x_h,\a_h)^\top \thetahat_h\ind{k}\mid \bcH\indd{t-1}\right]\right| + N_\reg \cdot \veps_{\spanner} + N_\reg \cdot \veps_\rep\nn \\
    & \quad + 2 \sum_{\pi \in \Psi_h^\spanner}\left|\sum_{t\in \cI\ind{k}}\E^{\pi}\left[\phibar^\rep_h(\x_h,\a_h)^\top  \bvthetab\indd{t,\pihat\ind{k}}_h-Q^{\pihat\ind{k}}_h(\x_h,a;\bcH\indd{t-1})\mid \bcH\indd{t-1}\right]\right|, \quad (\text{triangle inequality}) \nn \\
    & \leq 2 \sum_{\pi \in \Psi_h^\spanner}\left|\sum_{t\in \cI\ind{k}}\E^{\pi}\left[Q^{\pihat\ind{k}}_h(\x_h,a;\bcH\indd{t-1})-\phibar^\rep_h(\x_h,\a_h)^\top \thetahat_h\ind{k}\mid \bcH\indd{t-1}\right]\right| +N_\reg \cdot \veps_{\spanner} + 3 d N_\reg \cdot \veps_\rep, \label{eq:trad}
    \end{align}
    where the last inequality follows by \cref{cor:repspanner} (in particular \eqref{eq:replearnee}).

Next, we bound the estimation error term in the right-hand side of \eqref{eq:trad}.
\subsection{Bound the Regression Error}
    By \cref{lem:freed} (Freedman's inequality), we have that for all $\pi \in \Psi^\spanner_h$ and $\thetabar\in\bbB_{2d}(4 H d^{2})$:
    \begin{align}
     N_\reg\cdot \veps_\mathrm{freed} &\geq \left| \sum_{t\in \cI\ind{k}} \mathbb{I}\{\bh\indd{t}=h,\bpi\indd{t}=\pi, \bzeta\indd{t}=1\}\cdot \left(\phibar^\rep_h(\x\indd{t}_h,\a\indd{t}_h)^\top \thetabar -\sum_{s=h}^H \bm\ell\indd{t}_s\right)\right. \nn \\
       & \qquad - \left. \sum_{t\in \cI\ind{k}} \E\left[ \mathbb{I}\{\bh\indd{t}=h,\bpi\indd{t}=\pi, \bzeta\indd{t}=1\}\cdot \left(\phibar^\rep_h(\x\indd{t}_h,\a\indd{t}_h)^\top \thetabar-\sum_{s=h}^H \bm\ell\indd{t}_s\right) \mid \bcH\indd{t-1}\right]\right|. \label{eq:freedd}
    \end{align}
    On the other hand, we have 
    \begin{align}
       & \sum_{t\in \cI\ind{k}} \E\left[ \mathbb{I}\{\bh\indd{t}=h,\bpi\indd{t}=\pi, \bzeta\indd{t}=1\}\cdot \left(\phibar^\rep_h(\x\indd{t}_h,\a\indd{t}_h)^\top \thetabar-\sum_{s=h}^H \bm\ell\indd{t}_s\right) \mid \bcH\indd{t-1}  \right]\nn \\
       & = \frac{\nu}{Hd} \sum_{t\in \cI\ind{k}} \E^{\pi \circ_{h+1}\pihat\ind{k}}\left[ \phibar^\rep_h(\x_h,\a_h)^\top \thetabar-\sum_{s=h}^H \bm\ell\indd{t}_s \mid \bcH\indd{t-1} \right] ,\nn \\
       & =\frac{\nu}{Hd} \sum_{t\in \cI\ind{k}} \E^{\pi}\left[ \phibar^\rep_h(\x_h,\a_h)^\top \thetabar-Q^{\pihat\ind{k}}_h(\x_h,a;\bcH\indd{t-1}) \mid \bcH\indd{t-1} \right]. \label{eq:fact2}
    \end{align}
    Now, by \cref{cor:repspanner} (in particular \eqref{eq:replearnee}) and the triangle inequality, we have 
    \begin{align}
      N_\reg \cdot \veps_\rep & \geq  \left| \sum_{t\in \cI\ind{k}} \E^{\pi}\left[ \phibar^\rep_h(\x_h,\a_h)^\top \thetabar -Q^{\pihat\ind{k}}_h(\x_h,a;\bcH\indd{t-1}) \mid \bcH\indd{t-1}  \right] \right.\nn \\
   &    \left. \qquad - \sum_{t\in \cI\ind{k}} \E^{\pi}\left[ \phibar^\rep_h(\x_h,\a_h)^\top \thetabar-\phibar^\rep_h(\x_h,\a_h)^\top \bvthetab\indd{t,\pihat\ind{k}}_h \mid \bcH\indd{t-1}  \right]  \right|,\nn \\
   & =\left| \sum_{t\in \cI\ind{k}} \E^{\pi}\left[ \phibar^\rep_h(\x_h,\a_h)^\top \thetabar -Q^{\pihat\ind{k}}_h(\x_h,a;\bcH\indd{t-1}) \mid \bcH\indd{t-1}  \right] \right.\nn \\
   &   \left. \qquad -  \E^{\pi}\left[ \phibar^\rep_h(\x_h,\a_h)^\top \right] \left(N_\reg \cdot \thetabar- \sum_{t\in \cI\ind{k}} \bvthetab\indd{t,\pihat\ind{k}}_h\right)   \right|. \label{eq:fact3}
    \end{align}
Thus, by combining \eqref{eq:freedd}, \eqref{eq:fact2}, and \eqref{eq:fact3}, we have 
    \begin{align} 
       &\sum_{\pi\in \Psi^\spanner_h} \left| \sum_{t\in \cI\ind{k}} \mathbb{I}\{\bh\indd{t}=h,\bpi\indd{t}=\pi, \bzeta\indd{t}=1\}\cdot \left(\phibar^\rep_h(\x\indd{t}_h,\a\indd{t}_h)^\top \thetabar -\sum_{s=h}^H \bm\ell\indd{t}_s\right) \right|\nn \\
       & \leq d N_\reg \cdot \veps_{\mathrm{freed}}+ \frac{\nu  N_\reg  \veps_\rep}{H}  + \frac{\nu}{Hd} \sum_{\pi\in \Psi^\spanner_h} \left|  \sum_{t\in \cI\ind{k}} \E^{\pi}\left[ \phibar^\rep_h(\x_h,\a_h)^\top \right] \left(N_\reg \cdot \thetabar- \sum_{t\in \cI\ind{k}} \bvthetab\indd{t,\pihat\ind{k}}_h\right) \right|.\nn 
    \end{align}
    Using that $\thetahat\ind{k}_h$ is the minimizer over $\bbB_{2d}(4Hd^2)$ of the left-hand side, and the right-hand side evaluted to $d N_\reg \cdot \veps_{\mathrm{freed}}+ \frac{\nu  N_\reg  \veps_\rep}{H}$ with $\thetabar = \frac{1}{N_\reg} \sum_{t\in \cI\ind{k}} \bvthetab\indd{t,\pihat\ind{k}}_h\in \bbB_{2d}(4Hd^2)$, we have that
    \begin{align}
        \sum_{\pi\in \Psi^\spanner_h} \left| \sum_{t\in \cI\ind{k}} \mathbb{I}\{\bh\indd{t}=h,\bpi\indd{t}=\pi, \bzeta\indd{t}=1\}\cdot \left(\phibar^\rep_h(\x\indd{t}_h,\a\indd{t}_h)^\top \thetahat\ind{k}_h -\sum_{s=h}^H \bm\ell\indd{t}_s\right) \right| \leq d N_\reg \cdot \veps_{\mathrm{freed}}+ \frac{\nu  N_\reg  \veps_\rep}{H}  .\nn 
    \end{align}
    Now, combining this with \eqref{eq:freedd} and \eqref{eq:fact2} with $\thetabar = \thetahat\ind{k}_k$, we get that  
    \begin{align}
       &  \sum_{\pi\in\Psi^\spanner_h} \left| \sum_{t\in \cI\ind{k}} \E^{\pi}\left[ \phibar^\rep_h(\x_h,\a_h)^\top \thetahat\ind{k}_h -Q^{\pihat\ind{k}}_h(\x_h,a;\bcH\indd{t-1}) \mid \bcH\indd{t-1} \right] \right| \nn \\
        & \leq \frac{ Hd^2 N_\reg \veps_{\mathrm{freed}}}{\nu }+ \frac{Hd}{\nu}\sum_{\pi\in \Psi^\spanner_h} \left| \sum_{t\in \cI\ind{k}} \mathbb{I}\{\bh\indd{t}=h,\bpi\indd{t}=\pi, \bzeta\indd{t}=1\}\cdot \left(\phibar^\rep_h(\x\indd{t}_h,\a\indd{t}_h)^\top \thetahat\ind{k}_h -\sum_{s=h}^H \bm\ell\indd{t}_s\right) \right|,\nn \\
         & \leq \frac{ 2Hd^2N_\reg \veps_{\mathrm{freed}}}{\nu }+ d  N_\reg  \veps_\rep. \label{eq:est}
    \end{align}

\subsection{Putting It All Together}
By combining \eqref{eq:postperform}, \eqref{eq:trad}, \eqref{eq:est}, and \eqref{eq:carr}, we have that under the event $\cE$ in \eqref{eq:event}:
\begin{align}
   \mathrm{Reg}_T &= HT_0+ \sum_{k\in[K]} \sum_{t\in \cI\ind{k}} \left(\E_{\bpi\sim \rho\ind{k}}\E\left[V_1^{\bpi}(x_1;\bcH^{t-1}) \right] - V^{\pi^\star}_1(x_1;\bcH^{t-1})\right), \nn \\
   & \leq HT_0+HT\nu + T \cdot \veps_{\spanner} + 3 d T \cdot \veps_\rep+\frac{2Hd^2 T\veps_{\mathrm{freed}}}{\nu } \nn \\
   & \qquad + d  T  \veps_\rep +\frac{N_\reg\log (T/\gamma)}{\eta} + 64 H^2 d^4 \eta T  + H \gamma.
\end{align}
Thus, plugging in the expression of $T_0$, $\veps_{\freed}$, and $(\veps_\spanner, \veps_\rep)$ from \cref{alg:algorithm_name}, \cref{lem:freed}, and \cref{cor:repspanner}, respectively, and ignoring polynormal factors $d, A, H,\log(|\Phi|\veps^{-1}\delta^{-1})$, we get that 
\begin{align}
\mathrm{Reg}_T &\prec \frac{1}{\veps^2} + T \veps +T \nu + \frac{T}{\nu} \sqrt{\frac{\nu}{N_\reg}} + \frac{N_\reg}{\eta} + \eta T, \nn \\
& \prec T^{2/3} + T N_\reg^{-1/3} + \sqrt{TN_\reg}, \quad \text{(by setting $\veps = T^{-1/3}$, $\eta = (N_\reg/T)^{1/2}$, $\nu = N^{-1/3}_\reg$)}\\
& \prec T^{4/5}  \quad \text{($N_\reg = T^{3/5}$)},
\end{align}
where the last step follows by setting $N_\reg = T^{2/3}$.
% \cwcomment{If the Freedman's ineq can be improved, then the $\frac{T}{\nu}\sqrt{\frac{1}{N_\reg}}$ can be improved to $T\sqrt{\frac{1}{\nu N_\reg}}$. }

\subsection{Spanner Guarantee}
\label{sec:spannersection}

\begin{algorithm}[h]
	\caption{\texttt{Spanner}: Computing an Approximate Spanner.}
	\label{alg:spanner}
	\begin{algorithmic}[1]
		\Require Layer $h$, feature classes $\Phi$, policy covers $\Psi_{1:H}$, feature map $\phib:\cX\times \cA \rightarrow \reals^{2d}$, \# of episodes $n$.
		  \State Define $\cG = \{g:(x,a)\mapsto \phi(x,a)^\top w \mid \phi \in \Phi , w \in \bbB_d(2\sqrt{d})\}$.
		%\Statex \algcommentbiglight{Computing an approximate
            %      spanner using learned features.}
              \State For $\theta \in \reals^{2d}$ and $(x,a)\in \cX\times \cA$, define \label{line:reward}
              \begin{align} r_{t}(x,a;\theta)\coloneqq  \left\{\begin{array}{ll} \phib_{h}(x,a)^\top {\theta}, &  \text{for }
              t=h, \\ 0, &  \text{otherwise}.
              \end{array}\right. \label{eq:rewards}
              \end{align} 
              \State Set $\cG_h =\{(x,a)\mapsto  \phib_h(x,a)^\top {\theta} : \theta \in \bbB_{2d}(1) \}$, and for $t\in[h-1]$, set $\cG_{t}=\cG$.
        \State For each $t\in[h]$, set $P_t=\unif(\Psi_{t})$.
                \State For $\theta\in\reals^{2d}$, define
                $\apx(\theta)=\psdp(h, r_{1:h}(\cdot,\cdot;\theta), \cG_{1:h},P_{1:h},
                n) \in \Pi$. \label{line:psdp}\label{line:linopt} \hfill\algcommentlight{\psdp{} as in \cite{mhammedi2023efficient}.}
                   \State For $\theta \in \reals^{2d}$ and $\pi\in\Pi$, define
                $\est(\pi)=\veceval(h,\phib_{h}, \pi,
                n)$. \label{line:est} \hfill\algcommentlight{\veceval{} as in \cite{mhammedi2023efficient}.}
                \State \label{line:spanner}Set $\pi_{1:2d} =  \rspanner\left(\apx(\cdot), \est(\cdot), 2, \sqrt{\frac{Ad^2 \log (n d H |\Phi|/\delta)}{\alpha n}}\right)$. \hfill\algcommentlight{\rspanner{} as in \cite{mhammedi2023efficient}.} 
		\State \textbf{Return:} Policy cover
		$\{\pi_1,\dots,\pi_{2d}\}$. 
	\end{algorithmic}
\end{algorithm}

\begin{lemma}[Spanner Guarantee]
    \label{lem:span}
    Let $\veps,\alpha, \delta\in (0,1)$, $h\in[H]$, $n\geq 1$, and $\phibar_h :\cX\times \cA \rightarrow \reals^{2d}$ be given. Suppose that \cref{assm:normalizing} and \cref{assm:real} hold, and let $\Psi_{1:h}$ be such that for all $s\in[h]$, $\Psi_s$ is an $(\alpha,\veps)$-policy cover for layer $s$ with $|\Psi_s|=2d$. Then, the output $\Psi^\spanner_h=\emph{\texttt{Spanner}}(h,\Phi, \Psi_{1:h}, \phibar_h, n)$ (\cref{alg:spanner}) is such that $|\Psi^\spanner_h|=2d$ and, with probability at least $1-\delta$, for all $\pi' \in \Pi$, there exist $\{\beta_{\pi}\in[-2,2]: \pi \in \Psi^\spanner_h\}$ such that 
    \begin{gather}
        \label{eq:spanner}
       \left\|\E^{\pi'}[\phibar_h(\x_h,\a_h)] - \sum_{\pi \in \Psi^\spanner_h}\beta_\pi \cdot \E^{\pi}[\phibar_h(\x_h,\a_h)]\right\| \leq  \veps_\spanner(n,\alpha,\delta), 
       \intertext{where}
       \veps_\spanner(n, \alpha, \delta) \coloneqq c H^2 d \sqrt{\frac{d A \cdot (d \log(2 n \sqrt{d} H) + \log(n |\Phi|/\delta))}{\alpha n}} + H^2 d^{5/2} \veps.  \label{eq:theeps} 
    \end{gather}
    where $c>0$ is a large enough absolute constant. Furthermore, the number of episodes $T_\spanner(n)$ used by the call to $\emph{\texttt{Spanner}}$ is at most $\wtilde{O}(H^2d^2n )$.
    \end{lemma}
    \begin{proof}
    % \zm{check this is the right version}
        To derive the desired bound, we will use the generic guarantee of $\rspanner$ from \cite[Proposition E.1]{mhammedi2023efficient}. To invoke this result, we first need to derive guarantees for the optimization and estimation subroutines $\apx$ and $\est$ withn the $\texttt{Spanner}$ algorithm (\cref{alg:spanner}). In particular, we need to show that there is some $\veps'\in(0,1)$ such that (with high probability) for any $\thetabar \in \reals^{2d}\setminus\{0\}$ and $\pi\in \Pi$, the outputs $\hat\pi_{\thetabar} \coloneqq \apx(\thetabar/\|\thetabar\|)$ and $\phihat^\pi \coloneqq \est(\pi)$ satisfy 
        \begin{align}
        \sup_{\pi\in \Pi} \thetabar^\top \E^{\pi}[\phibar_h(\x_h,\a_h)]
         \leq 	\thetabar^\top  \E^{\hat\pi_{\thetabar}}[\phibar_h(\x_h,\a_h)] +\veps' \cdot \|\thetabar\|\quad \text{and} \quad 
          \|\phihat^\pi - \E^{\pi}[\phibar_h(\x_h,\a_h)]\| \leq \veps'. \label{eq:target}
        \end{align}
        With this, we can apply \cite[Proposition E.1]{mhammedi2023efficient} to get that the output \[\pi_{1:2d} = \rspanner(\apx(\cdot),\est(\cdot),2,\veps)\] for $\veps\leq 2 \veps'$ is such that for all $\pi \in \Pi$, there exist $\beta_1,\dots, \beta_d\in[-2,2]$ satisfying 
        \begin{gather}
            \label{eq:replearnnew}
           \left\|\E^{\pi}[\phibar_h(\x_h,\a_h)] - \sum_{i=1}^d\beta_i \cdot \E^{\pi_i}[\phibar_h(\x_h,\a_h)]\right\| \leq 6 d \veps'. 
        \end{gather}
        Since $\apx$ is based on $\psdp$ as in \cref{line:psdp} of \cref{alg:spanner} and $\Psi_{1},\dots, \Psi_h$ are $(\alpha,\veps)$-policy covers for layers $1$ to $h$, respectively, \cite[Corollary H.1]{mhammedi2023efficient} implies that there is an event $\cE^\psdp$ of probability at least $1-\delta/2$ under which for any $\thetabar \in \reals^d \setminus \{0\}$, the output $\hat\pi_{\thetabar}=\apx(\thetabar)$ satisfies
        \begin{align}
           & \sup_{\pi\in \Pi} \thetabar^\top \E^{\pi}[\phibar_h(\x_h,\a_h)]
         \leq 	\thetabar^\top  \E^{\hat\pi_{\thetabar}}[\phibar_h(\x_h,\a_h)]\nn \\
         & \quad  + \|\thetabar\| \cdot \left(c H^2  \sqrt{\frac{d A \cdot (d \log(2 n \sqrt{d} H) + \log(n |\Phi|/\delta))}{\alpha n}} + H^2  d^{3/2} \veps \right), \nn
        \end{align}
        for a large enough absolute constant $c>0$.
        On the other hand, since $\est$ is based on $\veceval$ as in \cref{line:est}, \cite[Lemma G.3]{mhammedi2023efficient} implies that there is an event $\cE^{\veceval}$ of probability at least $1-\delta/2$ under which for all $\pi\in \Pi$, the output $\phihat^\pi \coloneqq \est(\pi)$ satisfies
        \begin{align}
            \|\phihat^\pi - \E^{\pi}[\phibar_h(\x_h,\a_h)]\| \leq  c\cdot \sqrt{\frac{\log (2/\delta)}{n}}, \nn
        \end{align}
        for a large enough absolute constant $c>0$. Therefore, under $\cE^\psdp \cap \cE^\veceval$, $\apx$ and $\est$ satisfy \eqref{eq:target} with 
        \begin{align}
            \veps' \coloneqq c H^2  \sqrt{\frac{d A \cdot (d \log(2 n \sqrt{d} H) + \log(n |\Phi|/\delta))}{\alpha n}} + H^2  d^{3/2} \veps.\nn 
        \end{align}
        Therefore, by \cite[Proposition]{mhammedi2023efficient} and the fact that $d\sqrt{\frac{A \log (n d H |\Phi|/\delta)}{\alpha n}} \leq \veps'$, the output \[\pi_{1:2d} = \rspanner\left(\apx(\cdot),\est(\cdot),2, d\sqrt{\frac{A \log (n d H |\Phi|/\delta)}{\alpha n}}\right)\] is such that for all $\pi \in \Pi$, there exist $\beta_1,\dots, \beta_{2d}\in[-2,2]$ satisfying 
        \begin{gather}
           \left\|\E^{\pi}[\phibar_h(\x_h,\a_h)] - \sum_{i=1}^{2d}\beta_i \cdot \E^{\pi_i}[\phibar_h(\x_h,\a_h)]\right\| \leq \veps_\spanner(n, \alpha, \delta), \nn 
        \end{gather}
        where $\veps_\spanner(n, \alpha, \delta)$ is as in \eqref{eq:theeps}. 
        \paragraph{Bounding the number of episodes}
        By \cite[Proposition E.1]{mhammedi2023efficient}, $\rspanner$ calls $\apx$ and $\est$ as most $\wtilde O(d^2)$ times. Each call to $\apx$ [resp.~$\est$] requires $H^2n$ episodes. This implies the desired bound on the number of iterations.
    \end{proof}

    \subsection{Representation + Spanner}

    \begin{corollary}
        \label{cor:repspanner}
        Let $\veps,\delta\in (0,1)$, $\phi^\rep_{1:H}$, and $\Psi_h^{\spanner}$ be as in \cref{alg:algorithm_name}. Then, for all $h\in[H]$, $\phi^\rep_h\in \Phi$ and $|\Psi^\spanner_h|=2d$ and there is an event $\cE^{\rep+\spanner}$ of probability $1-3\delta/4$ under which for all $h\in[H]$:
        \begin{itemize}
            \item For $f\in \cF_{h+1}$, with $\cF_{h+1}$ as in \eqref{eq:cF}, there exists $\ww^f_{h+1}\in \bbB_d(3d^{3/2})$ such that:
        \begin{align}
        \forall \pi \in \Pi,\quad \left|\E^{\pi}\left[\phi^\rep_h(\x_h,\a_h)^\top \ww^f_{h+1} -\E[f(\x_{h+1})\mid \x_h,\a_h]  \right]\right| \leq \veps_\rep \coloneqq 10 d^{7/2} \veps; \label{eq:replearnee}
        \end{align}
        \item For all $\pi' \in \Pi$, there exist $\{\beta_{\pi}\in[-2,2]: \pi \in \Psi^\spanner_h\}$ such that for $\phibar^\rep_h \coloneqq [\phi^\loss_h,\phi^\rep_h]$
        \begin{gather}
            \label{eq:spannernew}
           \left\|\E^{\pi'}[\phibar^\rep_h(\x_h,\a_h)] - \sum_{\pi \in \Psi^\spanner_h}\beta_\pi \cdot \E^{\pi}[\phibar^\rep_h(\x_h,\a_h)]\right\| \leq  \veps_\spanner \coloneqq  2 H^2d^{5/2} \veps.
        \end{gather} 
    \end{itemize}
    \end{corollary}
    \begin{proof}
        From \cref{alg:algorithm_name}, we have 
    \begin{gather}
        \phi^\rep_h = \texttt{RepLearn}(h,\cF_{h+1},\Phi,\unif(\Psi^\cov_{h}),T_\rep) \quad \text{and} \quad \Psi^{\texttt{span}}_{h}= \texttt{Spanner}(h,\Phi,\Psi^\cov_{1:h}, \phibar^\rep_h, T_{\spanner}),\nn 
        \intertext{where} 
        \Psi^\cov_{1:H}= \texttt{VoX}(\Phi, \veps, \delta/4), \qquad T_\rep \coloneqq \Trepval , \qquad T_\spanner = \Tspanval,  \label{eq:params}
    \end{gather}
    and $\alpha\coloneqq \frac{1}{8Ad}$.
    By \cref{lem:vox}, there is an event $\cE^{\cov}$ of probability at least $1-\delta/4$ under which, for all $h\in[H]$, $\Psi^\cov_{h}$ is an $(\alpha,\veps)$-policy cover for layer $h$ with $|\Psi^\cov_h|=d$. In what follows, we condition on $\cE^\cov$. By \cref{lem:rep} and \cref{lem:span}, there are events $\cE^\rep$ and $\cE^{\spanner}$ of probability at least $1-\delta/4$ each such that under $\cE^\rep\cap \cE^\spanner$ \eqref{eq:replearnee} and \eqref{eq:spannernew} hold; this follows from \eqref{eq:replearn} and \eqref{eq:spanner} and the choices of $T_\rep$ and $T_\spanner$ in \eqref{eq:params}. Finally, by the union bound, we have $\P[\cE^\cov\cap \cE^\rep\cap \cE^\spanner]\geq 1 - \delta$ which completes the proof. 
    \end{proof}
    \subsection{Martingal Concentration}
\begin{lemma}
    \label{lem:freed}
    Let $K$, $N_\reg$, $\phibar^\rep_{1:H}$, and $\cI\ind{k}$ be as in \cref{alg:algorithm_name} for $k\in[K]$. There is an event $\cE^{\freed}$ of probability at least $1 - \delta/4$ under which for all $\thetabar \in \bbB_{2d}(4 Hd^{2})$, $h\in[H]$, $k\in [K]$, and $\pi \in \Psi^\spanner_h$:
\begin{align}
&\frac{1}{N_\reg}\left| \sum_{t\in \cI\ind{k}} \mathbb{I}\{\bh\indd{t}=h,\bpi\indd{t}=\pi, \bzeta\indd{t}=1\}\cdot \left(\phibar^\rep_h(\x\indd{t}_h,\a\indd{t}_h)^\top \thetabar -\sum_{s=h}^H \bm\ell\indd{t}_s\right)\right. \nn \\
       & - \left. \sum_{t\in \cI\ind{k}} \E\left[ \mathbb{I}\{\bh\indd{t}=h,\bpi\indd{t}=\pi, \bzeta\indd{t}=1\}\cdot \left(\phibar^\rep_h(\x\indd{t}_h,\a\indd{t}_h)^\top \thetabar-\sum_{s=h}^H \bm\ell\indd{t}_s\right) \mid \bcH\indd{t-1}\right]\right| \leq \veps_\freed  \coloneqq  4Hd^2 \sqrt{\frac{\nu \log(d K H N_\reg /\delta)}{N_\reg}},\nn
\end{align}
where the random variables $\bh\indd{t}, \bzeta\indd{t}, \bpi\indd{t}$, and $\bcH\indd{t-1}$ are as in \cref{alg:spanner}.
\end{lemma}
\begin{proof}
   Fix $\thetabar\in \bbB_{2d}(4d^{2})$, $h\in[H]$, $k\in[K]$, and $\pi \in \Psi^{\spanner}_h$. We apply \cref{lem:freedhelp} (Freedman's inequality)  with 
    \begin{itemize}
        \item $R = 4 H d^{2}$;
        \item $n = N_\reg$;
        \item The random variable $\bm{w}^i$ set as the difference 
         \begin{align}
           \bm{w}^i & \coloneqq  \mathbb{I}\{\bh\indd{t_i}=h,\bpi\indd{t_i}=\pi, \bzeta\indd{t_i}=1\}\cdot \left(\phibar^\rep_h(\x\indd{t_i}_h,\a\indd{t_i}_h)^\top \thetabar -\sum_{s=h}^H \bm\ell\indd{t_i}_s\right) \nn \\ 
       & \quad -  \E\left[ \mathbb{I}\{\bh\indd{t_i}=h,\bpi\indd{t_i}=\pi, \bzeta\indd{t_i}=1\}\cdot \left(\phibar^\rep_h(\x\indd{t_i}_h,\a\indd{t_i}_h)^\top \thetabar-\sum_{s=h}^H \bm\ell\indd{t_i}_s\right) \mid \bcH\indd{t-1}\right],
         \end{align}
         where $t_i \coloneqq (k-1)\cdot N_\reg + i$;
         \item The filtration $\mathfrak{F}^i$ set as the $\sigma$-algebra $\sigma(\bcH\indd{t_i-1})$;
         \item  The variance term $V_n$ has the following upper bound
         \begin{align*}
             V_n =  \sum_{i=1}^{N_\reg} \E\left[ (\bm{w}^i)^2 \mid \mathfrak{F}^{i-1}\right] &\le  \sum_{i=1}^{N_\reg}\E\left[ \mathbb{I}\{\bh\indd{t_i}=h,\bpi\indd{t_i}=\pi, \bzeta\indd{t_i}=1\}\cdot \left(\phibar^\rep_h(\x\indd{t_i}_h,\a\indd{t_i}_h)^\top \thetabar -\sum_{s=h}^H \bm\ell\indd{t_i}_s\right)^2 \mid \mathfrak{F}^{i-1}\right]
             \\&\le 8Hd^3\nu N_\reg;
         \end{align*}

         % $V_n =  \sum_{i=1}^{N_\reg} \E\left[ (\bm{w}^i)^2 \mid \mathfrak{F}^{i-1}\right] \le \E\left[ \mathbb{I}\{\bh\indd{t_i}=h,\bpi\indd{t_i}=\pi, \bzeta\indd{t_i}=1\}\cdot \left(\phibar^\rep_h(\x\indd{t_i}_h,\a\indd{t_i}_h)^\top \thetabar -\sum_{s=h}^H \bm\ell\indd{t_i}_s\right)^2 \mid \mathfrak{F}^{i-1}\right]$
         \item $\lambda = H^{-1} \left(\frac{d^2 \nu {N_\reg}}{\log(K H N_\reg /\delta )}\right)^{-1/2}$;
    \end{itemize}
    to get that there is an event $\cE^\freed_{h,k,\pi}(\thetabar)$ of probability at last $1-(N_\reg)^{-d} H^{-1} K^{-1}d^{-1} \delta/8$ under which 
    \begin{align}
       &  \frac{1}{N_\reg}\left| \sum_{t\in \cI\ind{k}} \mathbb{I}\{\bh\indd{t}=h,\bpi\indd{t}=\pi, \bzeta\indd{t}=1\}\cdot \left(\phibar^\rep_h(\x\indd{t}_h,\a\indd{t}_h)^\top \thetabar -\sum_{s=h}^H \bm\ell\indd{t}_s\right)\right. \nonumber \\ &\qquad - \left. \sum_{t\in \cI\ind{k}} \E\left[ \mathbb{I}\{\bh\indd{t}=h,\bpi\indd{t}=\pi, \bzeta\indd{t}=1\}\cdot \left(\phibar^\rep_h(\x\indd{t}_h,\a\indd{t}_h)^\top \thetabar-\sum_{s=h}^H \bm\ell\indd{t}_s\right) \mid \bcH\indd{t-1}\right]\right| \nonumber
       \\&\leq  4Hd^2 \sqrt{\frac{\nu \log(d K H N_\reg /\delta)}{N_\reg}}.\label{eq:eed}
    \end{align}
    %\cwcomment{Can we improve the right-hand side to 
    %\begin{align*}
    %     \text{poly}(d,H)\left(\sqrt{\frac{\nu}{N_\reg}} + \frac{1}{N_\reg}\right)? 
    %\end{align*}
    %This is because the variance of $\mathbb{I}\{\bh\indd{t}=h,\bpi\indd{t}=\pi, \bzeta\indd{t}=1\}\cdot \left(\phibar^\rep_h(\x\indd{t}_h,\a\indd{t}_h)^\top \thetabar-\sum_{s=h}^H \bm\ell\indd{t}_s\right)$ scales with $\nu$. 
    %}

    Let $\cC$ be a minimal $(d H N_\reg)^{-1}$-cover of $\bbB_{2d}(4Hd^{2})$ with respect to the $\|\cdot\|$ distance. Under the event \[\cE^\freed\coloneqq \bigcap_{h\in[H],k\in[K],\pi\in \Psi^\spanner_h,\thetabar\in\bbB_{2d}(4 H d^{2})} \cE^\freed_{h,k,\pi}(\thetabar),\] \cref{eq:eed} holds for all $h\in[H],k\in[K]$, and $\thetabar \in \bbB_{2d}(4 H d^{2})$ up to an additive $O(1/N_\reg)$ error. By the union bound, we have $\P[\cE^\freed]\geq 1 - \delta/4$ which completes the proof.
\end{proof}

\clearpage
\section{Policy Cover and Representation Learning Algorithms}
In this section, we present guarantees for $\texttt{VoX}$, $\texttt{RepLearn}$, and $\texttt{RobustSpanner}$ which we need in the analysis of our oracle efficient algorithm. The results are based on \citep{mhammedi2023efficient}.
\subsection{Policy Cover}
The following result is a restatement of \cite[Theorem 12]{mhammedi2023efficient}.
\begin{lemma}[VoX Guarantee]
    \label{lem:vox}
  Let $\veps,\delta\in(0,1)$ be given. Suppose \cref{assm:normalizing} and \cref{assm:real} hold. Then, there is an event $\cE^{\cov}$ of probability at least $1-\delta$ under which the output $\Psi_{1:H}^{\cov} = \emph{ \texttt{VoX}}(\Phi,\veps,\delta)$ is such that for all $h\in[H]$:
  \begin{itemize}
    \item $\Psi_h^\cov$ is a $(\frac{1}{8Ad},\veps)$-policy cover for layer $h$; 
    \item $|\Psi^\cov_h|\leq d$.
  \end{itemize}
  Furthermore, the number of episodes $T_{\cov}(\veps)$ used by the call to $\emph{\texttt{VoX}}$ is bounded by $\wtilde O(A d^{13} H^6 \log(\Phi/\delta))/\veps^2$.
\end{lemma}

\subsection{Representation Learning}

\begin{lemma}[Representation Learning Guarantee]
\label{lem:rep}
Let $\veps,\alpha, \delta\in (0,1)$, $h\in[H-1]$, and $n\geq 1$ be given and define the function class 
\begin{align}
    \cF_{h+1} \coloneqq \left\{f : (x,a)\mapsto \max_{a\in \cA} \phibar_{h+1}(x,a)^\top \thetabar \mid \phibar_{h+1} = [\phi^\loss_{h+1},\phi_{h+1}],\phi \in \Phi, \thetabar \in \bbB_{2d}(1) \right\}.  \label{eq:cF}
\end{align}
Further, let $\Psi$ be an $(\alpha,\veps)$-policy cover for layer $h$ with $|\Psi|=d$, and suppose \cref{assm:normalizing} and \cref{assm:real} hold. Then, with probability at least $1-\delta$, the output $\phibar^\rep_h=\emph{\texttt{RepLearn}}(h, \cF_{h+1},\Phi, \emph{\unif}(\Psi), n)$ is such that for all $f\in \cF_{h+1}$ there exists $\ww^f_{h+1}\in \bbB_d(3d^{3/2})$ such that:
\begin{align}
\forall \pi \in \Pi,\quad \left|\E^{\pi}\left[\phi^\rep_h(\x_h,\a_h)^\top \ww^f_{h+1} -\E[f(\x_{h+1})\mid \x_h,\a_h]  \right]\right| \leq c  \cdot \sqrt{\frac{A H d^5 \log(|\Phi|/\delta)}{\alpha n}} + 9 d^{7/2} \veps, \label{eq:replearn}
\end{align}
where $c>0$ is a large enough absolute constant. Furthermore, the number of episodes $T_\rep(\veps)$ used by the call to $\emph{\texttt{RepLearn}}$ is equal to $n$.
\end{lemma}
\begin{proof}
By \cite[Theorem F.1]{mhammedi2023efficient} and the assumption that $|\Psi|=d$, there is an event $\cE$ of probability at least $1-\delta/2$ under which $\phi^\rep_h$ satisfies:
\begin{align}
    \sup_{f\in \cF_{h+1}} \inf_{w\in \bbB_d(3d^{3/2})} \max_{\pi' \in \Psi} \E^{\pi'}\left[ (\phi^\rep_h(\x_h,\a_h)^\top w - \E[f(\x_{h+1})\mid \x_h,\a_h])^2 \right] \leq c\cdot \frac{A d^5 \log(|\Phi|/\delta)}{n}, \label{eq:repp3}
\end{align}
where $c$ is a large enough absolute constant. 
\begin{comment}Moreover, the number of episodes required by $\replearn$ is less than $n H\leq c' \frac{A^2 H^2 d^9 \log(|\Phi|/\delta)}{\veps^2}$. Instantiating this result with $\Psi'=\Psi$, we get that there is an event $\cE$ of probability at least $1-\delta/2$ under which the output $\phi^\rep$ of $\texttt{replearn}(\cG, \Phi, \Psi,n)$ satisfies: for all $h\in[H-1]$:
\begin{align}
    \sup_{g\in \cG} \inf_{w_{h+1}\in \bbB_d(3d^{3/2})} \max_{\pi' \in \Psi'} \E^{\pi'}\left[ (\phi^\rep_h(\x_h,\a_h)^\top w_{h+1} - \E[g(\x_{h+1})\mid \x_h,\a_h])^2 \right] \leq c \frac{A H d^5 \log(|\Phi|/\delta)}{n}. \label{eq:repp}
\end{align}
\end{comment}
We use this to show \eqref{eq:replearn}. In what follows, we condition on $\cE$. Fix $\pi\in \Pi$ and $f\in \cG$ and let $\ww^f_{h+1}$ be the vector $\ww\in \bbB_d(3d^{3/2})$ achieving the infimum in \eqref{eq:repp3} for the given choice of $f$. Let $\cX_{h,\veps}$ be the set of $\veps$-reachable states at layer $h$ as defined in \eqref{eq:reachpre}. With this this, we have for all $h\in[H]$, 
\begin{align}
 &    \left|\E^{\pi}\left[ \phi^\rep_h(\x_h,\a_h)^\top \ww^f_{h+1} - \E[f(\x_{h+1})\mid \x_h,\a_h] \right]\right| \nn \\
 & \leq  \left|\E^{\pi}\left[\mathbb{I}\{ \x_h \not\in \cX_{h,\veps} \} \cdot (\phi^\rep_h(\x_h,\a_h)^\top \ww^f_{h+1} - \E[f(\x_{h+1})\mid \x_h,\a_h]) \right]\right|\nn\\
 & \quad + \left|\E^{\pi}\left[\mathbb{I}\{   \x_h \in \cX_{h,\veps}\} \cdot (\phi^\rep_h(\x_h,\a_h)^\top \ww^f_{h+1} - \E[f(\x_{h+1})\mid \x_h,\a_h])^2 \right]\right|, \nn
\\
& \leq  \left|\E^{\pi}\left[\mathbb{I}\{ \x_h \in \cX_{h,\veps}  \} \cdot (\phi^\rep_h(\x_h,\a_h)^\top \ww^f_{h+1} - \E[f(\x_{h+1})\mid \x_h,\a_h]) \right] \right| + 9 d^{7/2}  \veps, \label{eq:split}
\end{align}
where the last inequality follows by \cref{lem:rideoff}.

We now bound the first term on the right-hand side of \eqref{eq:split}. By Jensen's inequality, we have 
\begin{align}
& \left|\E^{\pi}\left[\mathbb{I}\{ \x_h \in \cX_{h,\veps}\} \cdot (\phi^\rep_h(\x_h,\a_h)^\top \ww^f_{h+1} - \E[f(\x_{h+1})\mid \x_h,\a_h]) \right] \right|\nn \\
&\leq \sqrt{\E^{\pi}\left[\mathbb{I}\{ \x_h \in \cX_{h,\veps} \} \cdot (\phi^\rep_h(\x_h,\a_h)^\top \ww^f_{h+1} - \E[f(\x_{h+1})\mid \x_h,\a_h])^2 \right]},\nn \\
\intertext{and so using that $\Psi$ is a $(\alpha,\veps')$-policy cover (see \cref{def:cover}), we have}
& \leq  \sqrt{\alpha^{-1}\cdot  \max_{\pi' \in \Psi} \E^{\pi'}\left[ (\phi^\rep_h(\x_h,\a_h)^\top \ww^f_{h+1} - \E[f(\x_{h+1})\mid \x_h,\a_h])^2 \right]}, \nn \\
& \leq  \sqrt{ c  \cdot \frac{A H d^5 \log(|\Phi|/\delta)}{\alpha n}},\nn 
\end{align} 
where the last step follows by \eqref{eq:repp3}. Combining this with \eqref{eq:split} yields \eqref{eq:replearn}.
\end{proof}

\clearpage
\section{Lower Bound for Bandit feedback with Unstructured Losses}
\label{sec:lower bound}
In full-information, one does not require any structure on the losses.
We show that this is not the case for the bandit case via a lower bound depending polynomially on the number of states.
This lower bound implies that the low-rank transition structure with unstructured losses does not give any significant improvements over the tabular setting.

\begin{theorem} \label{thm:lower bound}
There exists a low-rank MDP with $S$ states, $A$ actions and sufficiently large time step $T$ with unstructured losses such that any agent suffers at least regret of $\Omega(\sqrt{SAT})$.
\end{theorem}
\begin{proof}
We assume $4S<\sqrt{T}$.
The construction is an $H=1$ (i.e. contextual bandit) MDP with uniform initial distribution over states.
Each state is a copy of an $A$-armed bandit problem with Bernoulli losses with mean $\frac{1}{2}$, and one randomly chosen optimal arm with mean $\frac{1}{2}-\Delta$. 
Following the standard lower bound construction for bandits \citep{lattimore2020bandit}, there exists $\Delta=\Theta(1/\sqrt{TAS})$ such that the the regret of playing any individual bandit problem for $N\leq 2T/S$ rounds is lower bounded by $\Omega(N\Delta)$.
Let denote $N(s)$ the number of time the agent receives the initial state $s$, then any agent suffers a regret lower bound in our MDP of
\begin{align*}
\Omega\left(\sum_{s=1}^S\E[\min\{N(s),2T/S\}]\right)\,.
\end{align*}
$N(s)$ is the sum of $T$ Bernoulli random variables with mean $1/S$. 
We have by Hoeffding's inequality
\begin{align*}
\mathbb{P}[N(s)>T/S+x] \leq \exp(-2x^2/T)\,.
\end{align*}
This allows to upper bound the tail
\begin{align*}
    \E[N(s)\mathbb{I}(N(s)>2T/S)] &\leq \int_{T/S}^\infty x(2x/T\exp(-2x^2/T))\,dx\\
    &\leq \int_{T/S}^\infty 4\exp(x/\sqrt{T}-2x^2/T)\,dx \tag{$\frac{1}{2}(x/\sqrt{T})^2<\exp(x/\sqrt{T})$}\\
    &\leq \int_{T/S}^\infty 4\exp(-x/\sqrt{T})\,dx \tag{$x\geq   T/S>4\sqrt{T}$}\\
    &= 4\sqrt{T}\exp(-\sqrt{T}/S)\leq 4\sqrt{T}\exp(-4)\leq T/(2S)\,. 
\end{align*}
Hence $
    \E[\min\{N(s),2T/S\}]=T/S-\E[N(s)\mathbb{I}(N(s)>2T/S)]\geq T/(2S)$ and the regret in the MDP is lower bounded by $\Omega\left(\sqrt{TSA}\right)$.
\end{proof}

\section{Helper Results}

% The following is a restatement of Lemma 3.1 in \cite{dai2023refined}.
% \begin{lemma}[\cite{dai2023refined}]
% \label{lem:logbarrier}
% 	Let $x_1,\dots,x_T \in \Delta(\cA)$ be defined as
% 	\begin{align}
% 	\forall t \in[T],\quad 	x_t \in \argmin_{x\in \Delta(\cA)} \eta x^\top \sum_{\tau<t} c_\tau + \Psi(x), 
% 	\end{align}
% 	where $c_t\in \reals^{A}$ is an arbitrary loss vector corresponding to the $t$-th iteration and $\Psi(p)=\sum_{i=1}^A \log p_i^{-1}$ is the log-barrier regularizer. Then, the regret against any distribution $y \in \Delta(\cA)$ with respect to $(c_t)$ is bounded as:
% 	\begin{align}
% 		\sum_{t=1}^T (x_t - y)^\top c_t \leq \frac{\Psi(y)- \Psi(x_1)}{\eta} + \eta \sum_{t=1}^T \sum_{i\in \cA} x_{t,i}c_{t,i}^2.
% 	\end{align}
% \end{lemma}

\begin{lemma}
    \label{lem:rideoff}
    Let $\veps, B>0$ and $h\in[2 \ldotst H]$ be given. For any function $f:\cX \rightarrow [-B,B]$ and $\pi\in \Pi$, we have 
    \begin{align}
    \E^\pi[\mathbb{I}\left\{\x_{h}\not\in \cX_{h,\veps}  \right\} \cdot f(\x_h)]\leq B \sqrt{d} \veps,
    \shortintertext{where}
    \cX_{h,\veps} \coloneqq \left\{x\in \cX : \max_{\pi\in \Pi} d^{\pi}_h(x) \geq \veps \cdot \|\mu^\star_h(x)\| \right\},   \label{eq:reach}
    \end{align}
    denotes the set of states that are $\veps$-reachable at layer $h$.
    \end{lemma}
    \begin{proof}
    Fix $f:\cX \rightarrow [-B,B]$ and $\pi\in \Pi$. Using the definition of $\cX_{h,\veps}$ in \eqref{eq:reach}, we have that $x\not\in\cX_{h,\veps}$ only if $d^{\pi}_h(x)<\veps \|\mu^\star_h(x)\|$. Using this, we have 
    \begin{align}
        \E^\pi[\mathbb{I}\left\{\x_{h}\not\in \cX_{h,\veps}  \right\} \cdot f(\x_h)]&\leq \sum_{x\in \cX} \mathbb{I}\left\{x\not\in \cX_{h,\veps}  \right\}\cdot  d^\pi_h(x) \cdot   f(x), \nn \\
        & \leq B \veps \sum_{x\in \cX} \|\mu^\star_h(x)\|,\nn \\
        & \leq  B d^{3/2}\veps,
    \end{align}
    where the last step follows by the normalizing assumption on $\mu^\star$ (see \cref{assm:normalizing}) and \cite[Lemma I.3]{mhammedi2023efficient}.
    \end{proof}

\subsection{Martingale Concentration and Regression Results}
\begin{lemma}
	\label{lem:freedhelp}
	Let $R>0$ be given and let $\w^1,\dots \w^n$ be a sequence of
        real-valued random variables adapted to filtration
        $\mathfrak{F}^1,\cdots, \mathfrak{F}^n$. Assume that for all $i\in[n]$, $\w^i \leq R$ and $\E[\w^i\mid \mathfrak{F}^{i-1}]=0$. Define $\bm{S}_n\coloneqq \sum_{i=1}^n\w^i$ and $V_n\coloneqq \sum_{i=1}^n \E[(\w^i)^2\mid \mathfrak{F}^{i-1}].$ Then, for any $\delta \in(0,1)$ and $\lambda \in[0,1/R]$, with probability at least $1-\delta$, 
	\begin{align}
		\bm{S}_n \leq \lambda V_n + \ln (1/\delta)/\lambda.
	\end{align}
      \end{lemma}
We now state two helpful results from \cite{mhammedi2024power} without a proof.
\begin{lemma}	
	\label{lem:corbern}
	Let $B>0$ and $n\in \mathbb{N}$ be given.
        abstract set. Further, let $\cQ \subseteq \{g:\cX\times \cA \rightarrow
        [0,B]\}$ be a finite function class and $(\x^1,\a^1,\bm\veps^1),\dots, (\x^n,\a^n,\bm{\veps}^n)$ be
        a sequence of i.i.d.~random variables in $\cX\times \cA \times \reals$. Then, for any $\delta \in(0,1)$, with probability at least $1-\delta$, we have 
	\begin{align}
		\forall g\in \cQ, \quad \frac{1}{2} \|g\|^2 - 2 B^2\log(2|\cQ|/\delta) \leq \|g\|_n^2 \leq 2 \|g\|^2 + 2B^2 \log(2|\cQ|/\delta),
		\end{align}
		where $\|g\|^2 \coloneqq \sum_{i\in[n]}\E[g(\x^i,\a^i)^2\mid \mathfrak{F}^{i-1}]$ and $\|g\|_n^2 \coloneqq \sum_{i=1}^n g(\x^i,\a^i)^2$.
\end{lemma}

\begin{lemma}[Generic regression guarantee]
		\label{lem:reg}
		Let $B>0$, $n\in \mathbb{N}$, and $f_\star:\cX \times \cA \rightarrow [0,B]$ be given. Further, let $\cF \subseteq \{f:\cX \times \cA \rightarrow [0,B]\}$ be a finite function class and $(\x^1,\a^1,\bm\veps^1),\dots, (\x^n,\a^n,\bm{\veps}^n)$ be
        a sequence of i.i.d.~random variables in $\cX\times \cA \times \reals$. Suppose that 
		\begin{itemize}
            \item $f_\star \in \cF$;
			\item $\z^i = f_\star(\x^i,\a^i) + \bm{\veps}^i + \bm{b}^i$, for all $i\in[n]$; 
			\item $\bm{b}^1,\dots,\bm{b}^n\in \reals$ (not necessarily i.i.d.);
			\item $\bm{\veps}^i\in[-B,B]$, for all $i\in[n]$; and
			\item$\E[\bm{\veps}^i\mid \x^i,\a^i]=0$.
			\end{itemize}
Then, for $\fhat \in \argmin_{f\in \cF}\sum_{i=1}^n(f(\x^i,\a^i)- \z^i)^2$ and any $\delta \in(0,1)$, with probability at least $1-\delta/2$,
		\begin{align}
		\|\fhat - f_\star\|_n^2  \leq 8B^2 \log(2|\cF|/\delta)+8  \sum_{i=1}^n (\bm{b}^i)^2,
		\end{align}
	where $\|\fhat- f_\star\|^2_n  \coloneqq \sum_{i=1}^n (\fhat(\x^i,\a^i)- f^\star(\x^i,\a^i))^2$.
      \end{lemma}

      \begin{proof}
		Fix $\delta \in (0,1)$ and let $\Lhat_n(f) \coloneqq \sum_{i=1}^n (f(\x^i,\a^i)- \z^i)^2$, for $f\in \cF$, and note that since $\fhat \in \argmin_{f\in \cF}\Lhat_n(f)$, we have
		\begin{align}
		0 \geq \Lhat_n(\fhat)- \Lhat_n(f_\star) = \nabla \Lhat_n(f_\star)[\fhat - f_\star] + \|\fhat - f_\star\|^2_n,
		\end{align}
  where $\nabla$ denotes directional derivative. Rearranging, we get that 
	\begin{align}
	\|\fhat - f_\star\|^2_n 
	&\leq - 2 \nabla \Lhat_n(f_\star)[\fhat - f_\star] - \|\fhat - f_\star\|_n^2, \nn \\
	& = 4 \sum_{i=1}^n (\z^i - f_\star(\x^i,\a^i)) (\fhat(\x^i,\a^i)- f_\star(\x^i,\a^i))- \|\fhat - f_\star\|^2_n, \nn \\
	& = 4 \sum_{i=1}^n (\bm{\veps}^i + \bm{b}^i) (\fhat(\x^i,\a^i)- f_\star(\x^i,\a^i))- \|\fhat - f_\star\|^2_n, \nn \\
	& = 4\sum_{i=1}^n \bm{\veps}^i \cdot (\fhat(\x^i,\a^i)- f_\star(\x^i,\a^i))- \|\fhat - f_\star\|^2_n  +4 \sum_{i=1}^n  \bm{b}^i\cdot  (\fhat(\x^i,\a^i)- f_\star(\x^i,\a^i)), \label{eq:twoterms}\\
    & \leq 4\sum_{i=1}^n \bm{\veps}^i \cdot (\fhat(\x^i,\a^i)- f_\star(\x^i,\a^i))- \|\fhat - f_\star\|^2_n +4 \sum_{i=1}^n  (\bm{b}^i)^2 + \frac{1}{2}\sum_{i=1}^n  (\fhat(\x^i,\a^i)- f_\star(\x^i,\a^i))^2, \nn \\
    & =  4\sum_{i=1}^n \bm{\veps}^i \cdot (\fhat(\x^i,\a^i)- f_\star(\x^i,\a^i))- \|\fhat - f_\star\|^2_n +4 \sum_{i=1}^n  (\bm{b}^i)^2 + \frac{1}{2} \|\fhat - f_\star\|^2_n.
	\end{align}
    Thus, rearranging, we get
    \begin{align}
        \|\fhat - f_\star\|^2_n \leq   8\sum_{i=1}^n \bm{\veps}^i \cdot (\fhat(\x^i,\a^i)- f_\star(\x^i,\a^i))-2 \|\fhat - f_\star\|^2_n +8 \sum_{i=1}^n  (\bm{b}^i)^2.\label{eq:wrond}
    \end{align}
We now bound the first term on the right-hand side of \eqref{eq:wrond}. For this, we apply \cref{lem:freedhelp} with $\w^i = \bm{\veps}^i \cdot (\fhat(\x^i,\a^i)- f_\star(\x^i,\a^i))$, $R = B^2$, $\lambda = 1/(8B^2)$, and $\mathfrak{F}^i=\emptyset$, and use 
	\begin{enumerate}
		\item the union bound over $f\in \cF$; and
		\item the facts that $\E[\bm{\veps}^i\mid \x^i ,\a^i]= 0$,
		\end{enumerate}
		to get that with probability at least $1-\delta/2$,
	\begin{align}
		 \sum_{i=1}^n \bm{\veps}^i \cdot (\fhat(\x^i,\a^i)- f_\star(\x^i,\a^i)) \leq \frac{1}{4}\|\fhat -f_\star\|^2_n +  B^2 \log(2|\cF|/\delta).
		\end{align}
        Combining this with \eqref{eq:wrond}, we get that with probability at least $1-\delta/2$,  
			\begin{align}
                \|\fhat - f_\star\|_n^2  \leq 8 B^2 \log(2|\cF|/\delta)+8 \sum_{i=1}^n (\bm{b}^i)^2.
                \end{align} \label{eq:term1}
				This completes the proof. 
	\end{proof}

 \subsection{Online Learning}
The following is the standard guarantee of exponential weights (e.g. Lemma F.4 of \cite{sherman2023improved}).
\begin{lemma}[Exponential Weights] \label{lem:EXP bound}
Given a sequence of loss functions $\{g^t\}_{t=1}^T$ over a decision set $\Pi$,  $\{p^t\}_{t=1}^T$ is a distribution sequence with $p^t \in \Delta\left(\Pi\right), \,\forall t \in [T]$ such that 
\begin{align*}
    p^{t+1}(\pi) \propto \exp\left(-\eta \sum_{t=1}^T g^t(\pi)\right).
\end{align*}
If $p^1$ is a uniform distribution over $|\Pi|$ and $\eta g^{t}(\pi) \ge -1$ for all $t \in [T]$ and $\pi \in \Pi$. Then
\begin{align*}
    \max_{p \in \Delta(\Pi)}\left\{\sum_{t=1}^T \left\langle g^t, p^t - p\right\rangle \right\} \le \frac{\log(|\Pi|)}{\eta} + \eta\sum_{t=1}^T\sum_{\pi \in \Pi} p^t(\pi) g^t(\pi)^2
\end{align*}
\end{lemma}

\subsection{Reinforcement Learning}
The following is standard simulation lemma which is first proposed by \cite{abbeel2005exploration}.
% \begin{lemma}[Simulation Lemma] \label{lem:simulation}
% For two finite-horizon MDPs $\widehat{M} = \{\cX, \cA, \ell - b, \{\widehat{P}_h\}_{h=1}^H\}$ and $M = \{\cX, \cA, \ell, \{P_h\}_{h=1}^H \}$ with horizon $H$ and $\|\ell - b\|_{\infty} \le B$. For any policy $\pi: \cX \rightarrow \Delta(\cA)$, we have
% \begin{align*}
% \left|\widehat{V}^\pi(x_1; \ell - b) - V^\pi(x_1; \ell)\right| \le \sum_{h=1}^H\E_{x, a \sim d_h^\pi} \left[BH\left\|\hatp_h\left(\cdot~|~x,a\right) - P_h\left(\cdot~|~x,a\right)\right\|_1 + \left|b_h(x,a)\right|\right].
% \end{align*}
% \end{lemma}

\begin{lemma}[Simulation Lemma] \label{lem:simulation}
For two finite-horizon MDPs $\widehat{M} = \{\cX, \cA, \ell, \{\widehat{P}_h\}_{h=1}^H\}$ and $M = \{\cX, \cA, \ell, \{P_h\}_{h=1}^H \}$ with horizon $H$ and $\|\ell\|_{\infty} \le 1$. Let the corresponding value function be $\Vhat_h^\pi(x; \ell)$ and $V_h^\pi(x; \ell)$ for step $h \in [H]$. For any policy $\pi: \cX \rightarrow \Delta(\cA)$, we have
\begin{align*}
\left|\Vhat^\pi_1(x_1; \ell) - V^\pi_1(x_1; \ell)\right| \le H\sum_{h=1}^H\E_{x, a \sim d_h^\pi} \left[\left\|\hatp_h\left(\cdot~|~x,a\right) - P_h\left(\cdot~|~x,a\right)\right\|_1\right].
\end{align*}
% \begin{align*}
% \left|\Vhat_h^\pi(x; \ell) - V_h^\pi(x; \ell)\right| \le H\sum_{h' = h}^H\E_{x, a \sim d_{h'}^\pi} \left[\left\|\hatp_{h'}\left(\cdot~|~x,a\right) - P_{h'}\left(\cdot~|~x,a\right)\right\|_1\right].
% \end{align*}
\end{lemma}

The following is the standard performance difference lemma which is first proposed by \cite{kakade2002approximately}.
\begin{lemma}[Performance Difference Lemma] \label{lem:PDL}
For a  finite-horizon MDPs $M = \{\cX, \cA, \ell, \{P_h\}_{h=1}^H \}$ starting at $x_1$, and two policies $\pi, \pi': \cX \rightarrow \Delta(\cA)$, we have
\begin{align*}
    V^{\pi'}_1(x_1; \ell) - V^\pi_1(x_1; \ell) = \sum_{h=1}^H \E_{x \sim d_h^\pi} \left[\sum_{a \in \cA} \left(\pi_h'(a|x) - \pi_h(a|x)\right) Q^{\pi'}_h(x,a; \ell)\right].
\end{align*}    
\end{lemma}

\end{document}